\documentclass[11pt]{article}
\usepackage{bm}
\usepackage{amsmath,amsthm,verbatim,amssymb,amsfonts,amscd, graphicx, graphics,physics}
\usepackage[mathscr]{eucal}
\usepackage{indentfirst}
\usepackage{booktabs}
\usepackage{pict2e}
\usepackage{epic}
\numberwithin{equation}{section}
\usepackage[footnotesize,bf]{caption}
\usepackage[margin=1in]{geometry}
\usepackage{epstopdf}
\usepackage{mathtools}
\usepackage[dvipsnames]{xcolor}
\usepackage{algorithmic}
\usepackage[ruled,vlined]{algorithm2e}
\usepackage{authblk}
\usepackage{placeins}
\usepackage{enumitem}

\usepackage[normalem]{ulem}

\usepackage{xurl}

\theoremstyle{plain}
\newtheorem{theorem}{Theorem}[section]

\newtheorem{lemma}{Lemma}[section]
\newtheorem{proposition}{Proposition}[section]

\theoremstyle{definition}

\newtheorem{remark}[theorem]{Remark}
\newtheorem{myexample}{Example}

\newtheorem{assumption}{Assumption}

\usepackage{threeparttable}
\usepackage{subcaption} 
\usepackage{siunitx}
\usepackage{multirow}

\usepackage[backend=bibtex, firstinits=true,maxbibnames=9,natbib=true,url=false,sorting=nyt,doi=true,backref=false]{biblatex}
\renewbibmacro{in:}{\ifentrytype{article}{}{\printtext{\bibstring{in}\intitlepunct}}}

\usepackage[colorlinks,linkcolor = blue, citecolor=blue]{hyperref} 
\usepackage{hyperref}
\usepackage[capitalize]{cleveref}

\def\R{{\mathbb R}}

\def\P{{\mathbb P}}

\def\C{{\mathbb C}}

\def\M{{\mathcal M}}

\def\sbm{{\mathsf{SBM}}}

\def\KL{{\mathsf{KL}}}

\def\JJ{{\bar{\bm J}}}

\def\G{{\mathcal G}}

\def\e{{\varepsilon}}
\def\l{{\lambda}}

\def\u{{\bm u}}

\def\sym{{\mathrm{sym}}}

\def\l{{\mathsf L}}
\def\d{{\mathsf D}}
\def\u{{\mathsf U}}

\def\diag{{\mathrm{diag}}}

\def\rhosync{\rho_{_{\mathrm{sync}}}}
\def\rhoasync{\rho_{_{\mathrm{async}}}}
\def\brhosync{\bar{\rho}_{_{\mathrm{sync}}}}
\def\brhoasync{\bar{\rho}_{_{\mathrm{async}}}}

\newcommand{\<}{\left\langle}
\renewcommand{\>}{\right\rangle}

\renewcommand{\bar}{\overline}

\newcommand{\E}{\mathbb{E}}

\newcommand{\W}{\mathcal{W}}

\newcommand{\pih}{\widehat{\bm\pi}}

\newcommand{\vt}{\vartheta}

\DeclareMathOperator*{\argmax}{arg\,max}
\DeclareMathOperator*{\argmin}{arg\,min}
\usepackage{setspace}

\title{Convergence analysis of a family of Zermelo-type iterations for the Bradley--Terry model}
\date{\today}

\author[1]{Ruijian Han}
\author[2]{Ding Lu}
\author[2]{Yiming Xu}
\affil[1]{Department of Data Science and Artificial Intelligence, The Hong Kong Polytechnic University, Hung Hom, Kowloon, Hong Kong SAR}
\affil[2]{Department of Mathematics, University of Kentucky, Lexington, KY, 40506, USA}

\begin{document}

\maketitle

\begin{abstract}
Zermelo's algorithm is a classical method for computing the maximum likelihood estimator in the Bradley--Terry (BT) model, but its convergence can be slow in practice. To accelerate computation, Newman introduced a family of Zermelo-type fixed-point iterations parameterized by $\alpha$, with Zermelo's algorithm recovered at $\alpha=1$. Empirical evidence suggests that the choice $\alpha=0$ often converges substantially faster, making it a promising alternative, yet the mechanism underlying this acceleration remains elusive. This paper provides theoretical insight into this phenomenon through a systematic local convergence analysis. We derive closed-form expressions for local convergence factors under synchronous and asynchronous updates and analyze their dependence on $\alpha$ via spectral analysis of the associated Jacobian matrices. For synchronous updates, we show that the algorithm may fail to converge when $\alpha<1$, and its local convergence factor is quasi-convex in $\alpha$ under the population BT model. In contrast, asynchronous updates are always locally convergent, and their local convergence factor is provably monotonically increasing in $\alpha$ under the population BT model of consistently ordered bipartite comparison graphs, establishing the optimality of $\alpha=0$ in this setting. We further establish asymptotic approximation results for the population convergence factors under the BT model, justifying their practical relevance. Numerical experiments on synthetic and real-world datasets confirm the theory. Our analysis complements existing convergence results and shows that the acceleration of $\alpha=0$ arises not only from the parameter choice but, more importantly, from the use of asynchronous updates.
\end{abstract}



\section{Introduction} 

The problem of ranking a collection of objects from paired comparison data is a fundamental task in machine learning and data science.
Given a set of $n$ objects indexed by $[n]\coloneqq \{1, 2, \dots, n\}$ and $w_{ij}\geq 0$ denoting the number of times object $i$ defeats (or is preferred over) object $j$, the goal is to infer a global ordering of the objects that best explains the observed pairwise comparison outcomes $\W\coloneqq \{w_{ij}\}_{i,j=1}^n$. Paired comparison ranking problems arise in a wide range of applications, with notable examples including sports analytics, where players or teams are ranked based on match outcomes~\citep{baker2014dynamic, bozoki2016application};
psychometrics and social science, where individual preferences are
inferred through pairwise judgment or social choice~\citep{loewen2012testing,
mcfadden1973conditional, thurstone1927method, 10.1111/rssa.12124,
wapman2022quantifying}; and reinforcement learning, where AI systems are aligned with human
preferences through comparative feedbacks~\citep{Christiano2017, Rafailov2023, Sun2025}; among others.

The Bradley--Terry (BT) model, introduced by~\citet{zermelo1929berechnung} and later formalized by~\citet{MR0070925}, is a classical probabilistic model for paired comparisons. In this model, each object $i\in[n]$ is associated with an unknown positive parameter $\pi_i$, referred to as its {\em strength}. The probability that object $i$ defeats $j$ (denoted $i\succ j$) in a single comparison is modeled as 
\begin{equation}\label{btprob0}
\P(i\succ j) = \frac{\pi_i}{\pi_i + \pi_j}.
\end{equation}
According to~\eqref{btprob0}, the strength vector $\bm\pi = (\pi_1, \ldots, \pi_n)^\top \in \R_+^n$ quantifies the relative dominance of objects, with larger entries corresponding to stronger objects with higher winning probabilities, thereby inducing a natural global ranking. This conceptual simplicity has contributed to the widespread use of the BT model in both theoretical studies and practical applications; see, e.g.,~\cite{hamilton2023many}. Throughout this paper, we assume that the strength vector $\bm\pi$ is normalized to satisfy the unit-product constraint $\prod_{i}\pi_i = 1$, a convention frequently adopted to address the scaling-invariance of the BT model~\eqref{btprob0} (i.e., $\alpha\bm\pi$ yields the same probability for all $\alpha > 0$).

A common approach to infer the unknown strengths $\bm\pi$ from paired comparison data is to use maximum likelihood estimation. Given observed outcomes $\W$ generated according to the BT model~\eqref{btprob0}, where it is assumed $m_{ij}\coloneqq w_{ij}+w_{ji}$ comparisons (independent of $\bm\pi$) are conducted between each pair $(i, j)$ of objects, the log-likelihood function of the strength vector $\bm\pi$ is given by
\begin{align}
l(\bm \pi) = \sum_{1\leq i, j\leq n} w_{ij} \log \left( \frac{\pi_i}{\pi_i + \pi_j} \right). \label{btl}
\end{align}
The strength vector can then be estimated by the \textit{maximum likelihood estimator} (MLE)
\begin{align}
\pih = \argmax_{\bm\pi\in\R_+^n, \ \prod_{i}\pi_i = 1}l(\bm
\pi).\label{mle0} 
\end{align}
It is known that the MLE $\pih$ in \eqref{mle0} uniquely
exists and is finite if and only if the graph of $\W$ is {\em strongly
connected}; that is, for every $\mathcal A\subset [n]$, there exists
$i\in \mathcal A$ and $j\in \mathcal A^\complement$ such that
$w_{ij}>0$. Here, we refer to, e.g.,~\cite{MR0097876} for related details on the MLE.
Once obtained, the MLE $\pih$ provides a basis for ranking the objects and conducting statistical inference based on the magnitudes of their estimated strengths.

A classical method for computing the MLE~\eqref{mle0} is Zermelo's algorithm~\citep{zermelo1929berechnung}.
It is a simple iterative algorithm derived from the stationary condition 
$\nabla l(\bm\pi)=0$ for the optimization problem~\eqref{mle0}, 
which can be equivalently written as the fixed-point equations 
\begin{align}
	\pi_i = 
	\frac{ \sum_{j=1}^n w_{ij}}
	{\sum_{j=1}^n(w_{j i} +w_{ij})/({\pi}_i + {\pi}_j)},
		\quad\text{for $i=1,\dots,n$.}
\label{eq:zermelo}
\end{align}
These fixed-point equations admit a unique positive solution (up to scaling) provided $\W$ is strongly connected, and the fixed-point iteration, called Zermelo's algorithm, is globally convergent from any initial point $\bm\pi^{(0)}\in\mathbb R^{n}_{+}$,
with monotonically increasing log-likelihood $l(\bm\pi)$; see, e.g., \cite{ MR0097876,zermelo1929berechnung}. Historically, Zermelo's algorithm has served as the computational workhorse for the BT model and has been extensively studied; see, e.g.,~\cite{hunter2004mm} and references therein.

Despite its guaranteed global convergence, the convergence of Zermelo's algorithm can be slow in practice;
see, e.g.,~\cite{Dykstra:1956, hunter2004mm, newman2023efficient}. For acceleration, \citet{newman2023efficient} proposed a generalized family of iterative algorithms parameterized by $\alpha\geq 0$, which includes Zermelo's algorithm as the special case with $\alpha=1$. 
These algorithms are derived from an equivalent reformulation of~\eqref{eq:zermelo} as the fixed-point equations
\begin{equation}\label{eq:nmalpha0}
{\pi}_i =
\frac{\sum_{j=1}^n w_{ij}{(\alpha{\pi}_i+{\pi}_j)}/{({\pi}_i + {\pi}_j})}
{\sum_{j=1}^n {(\alpha w_{ij} + w_{j i})}/{({\pi}_i + {\pi}_j})},
\quad\text{for $i=1,\dots,n$}.
\end{equation}
When $\alpha = 1$, these equations reduce to~\eqref{eq:zermelo}, and the fixed-point iteration recovers Zermelo's algorithm.
A key observation in~\cite{newman2023efficient} is that, within this family of fixed-point iterations by~\eqref{eq:nmalpha0}, the choice of $\alpha=0$ typically achieves the fastest convergence and significantly outperforms Zermelo's algorithm in numerical experiments.
These results suggest that the $\alpha=0$ variant may serve as a computationally more efficient alternative to Zermelo's algorithm.
However, it remains elusive whether this choice is universally preferable and, more generally, what theoretical principles govern the role of $\alpha$ in determining convergence.

The major goal of this work is to fill those gaps through a systematic local convergence analysis of fixed-point iterations 
defined by~\eqref{eq:nmalpha0}, which we collectively refer to as {\em Newman's $\alpha$-scheme}.
Specifically, we derive closed-form expressions for the local convergence factors
$\rho(\alpha)$ of these algorithms and investigate their dependence on the parameter $\alpha$. 
One of our key findings is that the behavior of $\rho(\alpha)$ 
depends critically on how the fixed-point iterations are implemented, 
namely, whether the variables are updated 
{\em synchronously} (with all $\pi_i$ updated simultaneously) 
or 
{\em asynchronously} (with $\pi_i$ updated sequentially).
In particular, we show that an asynchronous implementation is essential
for the $\alpha=0$ variant to attain its superior performance observed in practice.
The choice $\alpha=0$ alone is not sufficient: Under certain conditions, the corresponding synchronous iteration is provably non-convergent and therefore offers no acceleration over Zermelo's algorithm.

\paragraph{Contributions.} Our main contributions are threefold.
\begin{enumerate}
\item {\em Synchronous iterations}: We establish the local convergence 
	factors $\rhosync(\alpha)$ of the fixed-point iterations for~\eqref{eq:nmalpha0}
	in terms of the extreme eigenvalues of an associated Jacobian, 
	whose eigenvalues are all real. 
	For expected comparison outcomes from the BT model, 
	we show that the corresponding $\rhosync(\alpha)$ 
	is a quasi-convex function in $\alpha$, but is not necessarily 
	monotonically increasing. 
	With numerical examples, we demonstrate that 
	$\rhosync(0)$ can exceed $1$, indicating the $\alpha=0$ variant 
	is non-convergent and provides no acceleration over Zermelo's
	algorithm.

\item {\em Asynchronous iterations}:
	We show that the local convergence factor satisfies $\rhoasync(\alpha)<1$ for
	all $\alpha>0$, implying the algorithms are always locally convergent. 
	For expected comparison outcomes from the BT model with a bipartite comparison graph in consistent order,
	we prove that the corresponding $\rhoasync(\alpha)$ can be characterized by the
	eigenvalues of an associated quadratic eigenvalue problem and is
	monotonically increasing in $\alpha$, thereby establishing the
	optimality of the choice $\alpha=0$ in this specialized setting.
	For general expected BT outcomes, we further provide quantitative bounds 
	on the convergence factors.

\item 
	{\em Approximation by population outcomes of BT models:}
	We illustrate that the convergence factors of Newman's $\alpha$-schemes can often 
	be well approximated using the population (expected) comparison outcomes 
	of an underlying BT model, thereby demonstrating the practical
	relevance of our theoretical analysis.
	We justify this approximation both numerically, 
	using synthetic data generated from BT models under various graph
	topologies as well as real-world datasets from applications,
	and theoretically, by
	establishing rigorous asymptotic approximation results for the
	convergence factors under the BT model.
\end{enumerate}

\paragraph{Organization.}
The rest of this paper is organized as follows. \Cref{sec:mc} introduces the preliminaries, including a detailed description of Newman's $\alpha$-scheme under both synchronous and asynchronous updates and their associated matrix representations. \Cref{sec:ana} and \Cref{sec:gs} present our main results on the local convergence analysis of Newman's
$\alpha$-scheme under synchronous and asynchronous updates, respectively. 
\Cref{sec:app5} establishes approximation results relating the
local convergence factors to their population counterparts under the BT model. 
Numerical experiments are reported in~\Cref{sec:nums}, followed by 
concluding remarks in~\Cref{sec:conc}.

\paragraph{Notation.}
We follow standard notation in linear algebra and statistics.
$\R^{m\times n}$ and $\C^{m\times n}$ are the spaces
of real and complex $m$-by-$n$ matrices, 
with $\R^n=\R^{m\times 1}$ and $\C^m=\C^{m\times 1}$ for column vectors.
We write $\bm I_n$ for the identity matrix of size $n$,
$\bm e_i$ for the $i$-th column of $\bm I_n$,
and $\bm 1_n$ for the all-ones vector of length-$n$
(the subscript $n$ is often dropped when the size is clear from context).
For a matrix or vector, 
$\|\cdot\|_2$, $\|\cdot\|_\infty$, and $\|\cdot\|_F$ 
denote the $2$-norm, infinity norm, and the Frobenius norm,
respectively.
The operator $\Re(\cdot)$ extracts the real part,
and $\cdot^\top$ and $\cdot^{\rm H}$ denote the transpose and
conjugate transpose.
For a square matrix $\bm A\in\R^{n\times n}$, 
$\Lambda(\bm A)$ denotes the set of eigenvalues,
$\rho(A)$ denotes the spectral radius 
(i.e., largest absolute value of eigenvalue),
and $\lambda_i(A)$ denotes the $i$-th largest eigenvalue 
when all eigenvalues are real. 
We also use $\bm A_\l$, $\bm A_\d$, and $\bm A_\u$ for the 
strictly lower triangular, diagonal, and strictly upper triangular parts
of $\bm A$, respectively. 
For a vector $\bm v\in\R^n$,
$\bm v^ {\circ -1}$ denotes its entrywise reciprocal, 
and $\mbox{diag}(\bm v)$ denotes the diagonal matrix with $\bm v$ on the diagonal.
We use $\partial_{\bm x}$ for the partial derivative of a function with
respect to the variable $\bm x$, and  
$f\circ g(\bm x)\coloneqq f(g(\bm x))$ for the composition of functions.
$\mathbb I_{\{i\neq j\}}$  denotes the indicator function
of the event $i\neq j$; that is, 
$\mathbb I_{\{i\neq j\}}=1$ if $i\neq j$ and $0$ otherwise.
$\G([n], \W)$ denotes a weighted graph with 
$n$ nodes, where $[n]\coloneqq \{1,\dots,n\}$ is the vertex set 
and $\W\coloneqq\{w_{ij}\}$ denotes the edge weights; the associated matrix representation of $\W$ is denoted by $\bm W=(w_{ij})\in\R^{n\times n}$.
For a random variable $r(\alpha)$ parameterized by $\alpha$, 
we write $r(\alpha) \overset{\smash{\scriptstyle{\P}}}{\to} 0$ as $\alpha\to \alpha_*$ 
if $r(\alpha) $ converges in probability to $0$.
Other notations will be explained as first used.

\section{Preliminaries}\label{sec:mc}

\subsection{Newman's $\alpha$-scheme for the MLE}\label{sec:s+as}

Recall the fixed point equations~\eqref{eq:nmalpha0}, which give rise to the fixed-point iterations 
\begin{equation}\label{eq:alpha}
\begin{aligned}
{\pi}_i \leftarrow f_{\alpha,i}(\bm \pi)\coloneqq\frac{\sum_{j=1}^nw_{ij}(\alpha{\pi}_i+{\pi}_j)/({\pi}_i + {\pi}_j)}{\sum_{j=1}^n(\alpha w_{ij} + w_{j i})/({\pi}_i + {\pi}_j)},
\qquad\text{for $i=1,2,\dots,n$},
\end{aligned}
\end{equation}
where $\alpha\geq 0$ is a prescribed parameter. Varying the parameter $\alpha$ yields a family of iterative schemes for computing the MLE $\pih$, which we collectively refer to as {\em Newman's $\alpha$-scheme}~\citep{newman2023efficient}.
In practice, each iteration can be implemented in two distinct ways: synchronously or asynchronously.

Using {\em synchronous} iteration, 
all components $\pi_i$ in~\eqref{eq:alpha} 
are updated simultaneously, i.e.,
\begin{equation}\label{eq:sync1}
(\pi_1,\dots,\pi_n)^\top \leftarrow (f_{\alpha,1}(\bm \pi),\dots, f_{\alpha,n}(\bm \pi))^\top.
\end{equation}
To ensure feasibility and maintain numerical stability, the strength vector $\bm\pi$ is normalized to have a unit product after each pass: 
\begin{equation}\label{eq:normalize}
\bm\pi \leftarrow S(\bm\pi) \quad\text{with}\quad 
S(\bm\pi) \coloneqq \frac{\bm \pi}{\sqrt[n]{\prod_{i=1}^n\pi_i}}. 
\end{equation}
This yields the fixed-point iteration: given $\bm\pi^{(0)}\in\mathbb R_+^n$,  iteratively generate for $k=1,2,\dots$,
\begin{equation}\label{eq:fpsf}
    \bm\pi^{(k)}=S\circ F_\alpha (\bm\pi^{(k-1)}),
\end{equation} 
where  
\begin{equation}\label{n-s}
\begin{aligned}
F_\alpha (\bm\pi)\coloneqq (f_{\alpha,1}(\bm\pi), \ldots, f_{\alpha,n}(\bm\pi))^\top,
\end{aligned}
\end{equation}
and $\circ$ denotes the composition of functions.

Using {\em asynchronous} iteration, the components $\pi_i$ in~\eqref{eq:alpha} are updated sequentially as soon as new values become available.
Under a fixed ordering, this gives the (cyclic) asynchronous update 
\begin{equation}\label{eq:async1}
\pi_i \leftarrow  f_{\alpha,i}(\bm \pi),
\quad\text{sequentially for $i=1,\dots,n$}.
\end{equation}
Combined with normalization~\eqref{eq:normalize} after each full pass, this leads to the fixed-point iteration: 
given $\bm\pi^{(0)}\in\mathbb R_+^n$,  iteratively generate for $k=1,2,\dots$,
\begin{equation}\label{eq:fpasyn}
\bm\pi^{(k)} = S\circ A_{\alpha}(\bm\pi^{(k-1)}),
\end{equation}
where 
\begin{align}\label{fas}
A_{\alpha}(\bm\pi) \coloneqq F_{\alpha}^{(n)}\circ\dots\circ F_{\alpha}^{(1)}(\bm\pi),
\end{align}
and 
\begin{equation}
F_{\alpha}^{(i)}(\bm \pi) \coloneqq (\pi_1,\dots,\pi_{i-1}, f_{\alpha, i}(\bm\pi), \pi_{i+1},\dots, \pi_n)^\top,
\quad\text{for $i=1,\dots,n$}.
\end{equation}
Observe that each mapping $F_{\alpha}^{(i)} $ accounts for an individual asynchronous update from~\eqref{eq:async1}.

In practice, the asynchronous update~\eqref{eq:fpasyn} is empirically
{\em preferred} over its synchronous counterpart~\eqref{eq:fpsf}, as it
incorporates newly updated components at each pass and may lead to
faster convergence~\citep{newman2023efficient, yeung2025efficient}.
However, this empirical advantage is not yet fully supported by
theoretical analysis, and a systematic comparison between the two
implementations is largely absent from the literature. 
Existing theoretical analyses primarily focus on the global convergence of the algorithms.
In~\cite{newman2023efficient}, it is proved that the
asynchronous iteration~\eqref{eq:fpasyn} is globally convergent, since
each component update $\pi_i\leftarrow f_{\alpha, i}(\bm \pi)$ yields a
monotonic increase in the likelihood function $l(\bm\pi)$. 
When $\alpha=1$, the scheme~\eqref{eq:sync1} reduces to Zermelo's algorithm~\citep{zermelo1929berechnung}, whose
global convergence is well established and can be interpreted from 
different perspectives, including
the Minorization--Maximization (MM) framework~\citep{hunter2004mm},
gradient descent optimization~\citep{vojnovic2023accelerated}, and its
connection to the Sinkhorn algorithm for matrix
balancing~\citep{qu2025sinkhorn}. 
However, these global convergence analyses do not readily
extend to the synchronous scheme \eqref{eq:fpsf} when $\alpha < 1$,
and therefore do not fully explain the observed advantage of the asynchronous scheme. 
In fact, as we show, the two implementations can exhibit fundamentally different
convergence behavior, particularly as $\alpha\to0$ in certain cases. 
For example, the synchronous scheme may fail to converge when the observed outcomes deviate
substantially from the BT model assumptions (see \Cref{sec:num1} for a
numerical illustration). 
This contrasts with the common belief that the two implementations converge to the same solution
(see, e.g.,~\cite{yeung2025efficient}),
which generally holds for $\alpha=1$.

To quantify the convergence rate and gain further insight into the effects of the parameter $\alpha$, as well as the synchronous vs. asynchronous update, a local convergence analysis is needed. In particular, based on the fixed-point equations in \eqref{eq:fpsf} and \eqref{eq:fpasyn}, one can express the local convergence rate using the Jacobian $\bm J_{F_\alpha}$ (which is computed in \Cref{sec:mcmc}), and then a spectral analysis of this Jacobian would reveal detailed information on local convergence behavior. This is the focus of Sections~\ref{sec:ana}-\ref{sec:gs}.

\subsection{Matrix representations and linearizations}\label{sec:mcmc}
We introduce several matrix representations of Newman's $\alpha$-scheme
that will be used throughout the paper to facilitate the convergence analysis and provide further insight into the underlying algebraic structure of the algorithm.
To begin with, the governing fixed-point equation~\eqref{eq:nmalpha0} can be written in a matrix product form as 
\begin{equation}\label{eq:governfp}
\bm\pi =  \bm K_\alpha(\bm \pi) \bm\pi,
\end{equation}
where the coefficient matrix  $\bm K_\alpha(\bm \pi)\in\R^{n\times n}$ is given by 
\begin{equation}\label{eq:nalpha}
\bm K_\alpha(\bm \pi) 
\coloneqq \left(\bm C(\bm\pi) + \alpha \bm R(\bm\pi)\right)^{-1}
\cdot
\left(\bm W(\bm\pi) + \alpha \bm R(\bm\pi)\right),
\end{equation}
with $\bm W(\bm\pi)\in\R^{n\times n}$ denoting a strength-weighted adjacency matrix given by 
\begin{align}\label{myQ}
[\bm W(\bm\pi)]_{ij} \coloneqq \frac{w_{ij}}{\pi_i + \pi_j}\cdot\mathbb I_{\{i\neq j\}},
\end{align}
and $\bm R(\bm\pi)$ and $\bm C(\bm\pi)\in\R^{n\times n}$  being diagonal and defined by 
\begin{align}
\bm R(\bm\pi) \coloneqq \mbox{diag}(\bm W(\bm\pi) \bm 1)
\quad\text{and}\quad
\bm C(\bm\pi) \coloneqq \mbox{diag}(\bm W(\bm\pi)^\top \bm 1),
\end{align}
where $\bm 1\in\R^{n}$ denotes  a vector of all ones; that is, $\bm R(\bm\pi)$ and $\bm C(\bm\pi)$ correspond to the row and column sums of $\bm W(\bm\pi)$, respectively.

For a useful property,  the matrix $\bm K_\alpha(\bm\pi)$~\eqref{eq:nalpha}
is {\em similar} to a column-stochastic matrix $\bm P_\alpha(\bm\pi)$ (i.e., $\bm P_\alpha(\bm\pi) \geq 0$ and $\bm 1^\top \bm P_\alpha(\bm\pi) = \bm 1^\top$), obtained by interchanging the two matrix factors in~\eqref{eq:nalpha}:
\begin{align}\label{eq:palpha}
\bm P_\alpha(\bm\pi) 
&= \left(\bm W(\bm\pi) + \alpha \bm R(\bm\pi)\right) \cdot \left(\bm C(\bm\pi) + \alpha \bm R(\bm\pi)\right)^{-1}.
\end{align}
This implies $\bm K_\alpha(\bm\pi)$ shares the same eigenvalues as $\bm P_\alpha(\bm\pi)$,
and in particular, has a dominant eigenvalue equal to $1$ due to the property of stochastic matrices.

For local convergence analysis, we need to linearize and consider the Jacobian of the 
fixed-point mapping $F_\alpha(\bm\pi) \equiv \bm K_\alpha(\bm \pi) \bm\pi$ from~\eqref{n-s} 
at its fixed-point solution, i.e., the MLE $\pih$. 
A direct calculation yields the Jacobian $\bm J_{F_\alpha}(\bm\pi)\in\R^{n\times n}$ with entries
\begin{align}\label{jac:J}
[\bm J_{F_\alpha}(\bm\pi)]_{ij} = \begin{cases}
\frac{w_{ij} \frac{\pi_i(1 - \alpha)}{(\pi_i + \pi_j)^2} }{\sum_{\ell} \frac{\alpha w_{i\ell} + w_{\ell i}}{\pi_i + \pi_\ell} } + \frac{\left(\sum_{\ell} w_{i\ell} \frac{\alpha \pi_i + \pi_\ell}{\pi_i + \pi_\ell}\right)\frac{\alpha w_{ij} + w_{ji}}{(\pi_i + \pi_j)^2}}{\left(\sum_{\ell} \frac{\alpha w_{i\ell} + w_{\ell i}}{\pi_i + \pi_\ell}\right)^2} & j\neq i,\\
\frac{\sum_{\ell} \frac{w_{i\ell} \pi_\ell(\alpha - 1)}{(\pi_i + \pi_\ell)^2} }{\sum_{\ell} \frac{\alpha w_{i\ell} + w_{\ell i}}{\pi_i + \pi_\ell} }+\frac{\left( \sum_{\ell} w_{i\ell} \frac{\alpha \pi_i + \pi_\ell}{\pi_i + \pi_\ell} \right) \left( \sum_{\ell} \frac{\alpha w_{i\ell} + w_{\ell i}}{(\pi_i + \pi_\ell)^2} \right)}{\left( \sum_{\ell} \frac{\alpha w_{i\ell} + w_{\ell i}}{\pi_i + \pi_\ell} \right)^2} & j = i.
\end{cases}
\end{align}
By definition, $\bm J_{F_\alpha}(\bm\pi)$ is non-symmetric in general and may have complex eigenvalues.
At the fixed-point solution $\pih$, we can further simplify the Jacobian formulation to
\begin{align}\label{jac:Jfp}
[\bm J_{F_\alpha}(\pih)]_{ij} = \begin{cases}
\left.
\left( \widehat\pi_i\frac{w_{ij} +w_{ji}}{(\widehat\pi_i + \widehat\pi_j)^2} \right)\middle/
\left(\sum_{\ell} \frac{\alpha w_{i\ell} + w_{\ell i}}{\widehat\pi_i + \widehat\pi_\ell} \right)
\right.
& j\neq i,\\
1 - 
\left.
\left( \sum_{\ell}\widehat \pi_\ell \frac{w_{i\ell } +w_{\ell i}}{(\widehat \pi_i + \widehat \pi_\ell )^2} \right)\middle/
\left(\sum_{\ell} \frac{\alpha w_{i\ell} + w_{\ell i}}{\widehat \pi_i + \widehat \pi_\ell} \right)
\right.
& j = i,
\end{cases}
\end{align}
which follows directly from~\eqref{jac:J}
and the governing equations in~\eqref{eq:nmalpha0} for the solution $\pih$.

We make two useful observations for the Jacobian $\bm J_{F_\alpha}(\pih)$. First,
$\bm J_{F_\alpha}(\pih)$ is similar  
to a symmetric matrix
\begin{equation}\label{eq:jsym}
\bm J^{\sym}_{F_\alpha}(\pih):=
\bm D_\alpha(\pih)^{-1} \bm J_{F_\alpha}(\pih) \bm D_\alpha(\pih),
\end{equation}
where 
$\bm D_\alpha(\pih)$ is a diagonal matrix with entries
$[\bm D_\alpha(\pih)]_{ii} = \sqrt{
\left. 
\widehat{\pi}_i 
\middle/
\left(\textstyle \sum_{\ell} \frac{\alpha w_{i\ell} + w_{\ell i}}{\widehat{\pi}_i + \widehat{\pi}_\ell}\right)
\right.
}
$, i.e.,
\begin{align}\label{myD}
\bm D_\alpha(\pih) =
\diag(\pih)^{\frac{1}{2}}\cdot
\left(\bm C(\pih)+\alpha \bm R(\pih)\right)^{-\frac{1}{2}}.
\end{align}
Consequently, $\bm J_{F_\alpha}(\pih)$ has $n$ real eigenvalues. 
Second, the diagonal entries of $\bm J_{F_\alpha}(\pih)$ 
are minimized when $\alpha=0$, 
which yields a uniform lower bound
\begin{align}\label{eq:boundjii}
[\bm J_{F_\alpha}(\pih)]_{ii} 
\geq  
1 -
\left.
\left(\sum_{\ell}
\frac{\widehat \pi_\ell  \widehat \pi_i(w_{i\ell } +w_{\ell i})}
{(\widehat \pi_i + \widehat \pi_\ell )^2} \right)\middle/\left(\sum_{\ell} \frac{\widehat \pi_i w_{\ell i}}{\widehat \pi_i + \widehat \pi_\ell} \right)
\right.
> -1,
\end{align}
where the second inequality follows from 
\[
\sum_{\ell}\frac{\widehat \pi_\ell  \widehat \pi_i  (w_{i\ell }+w_{\ell i})}{(\widehat \pi_i + \widehat \pi_\ell )^2}
< \sum_{\ell}\frac{\widehat \pi_\ell\, w_{i\ell }}{\widehat \pi_i + \widehat \pi_\ell }
+\sum_\ell \frac{\widehat \pi_i w_{\ell i}}{\widehat \pi_i + \widehat \pi_\ell }
= 2 \sum_\ell \frac{\widehat \pi_i w_{\ell i}}{\widehat \pi_i + \widehat \pi_\ell},
\]
recalling the identity $\sum_{\ell}\frac{\widehat \pi_\ell\, w_{i\ell}}{\widehat \pi_i + \widehat \pi_\ell }
= \sum_{\ell}\frac{\widehat \pi_i\, w_{\ell i}}{\widehat \pi_i + \widehat \pi_\ell}$ at the fixed-point solution $\pih$ of~\eqref{eq:alpha}.

\subsection{Population version of BT model}
For theoretical purposes, the population version of the BT model
will play an important role in our subsequent analysis.
We recall the BT model~\eqref{btprob0}, where, for notational clarity,
the true strength vector is denoted by
$\bm\pi^*=(\pi_1^*,\dots,\pi_n^*)^\top\in\R_+^n$ 
and normalized such that $\prod_i \pi_i^*=1$:
\begin{equation}\label{eq:pbt}
\P(i\succ j) = \frac{\pi_i^*}{\pi_i^* + \pi_j^*}.
\end{equation}
Suppose that each pair of objects $i$ and $j$ is compared $m_{ij}$ times independently, 
and let $\M \coloneqq\{m_{ij}\}_{i,j=1}^n$.
Then, the observed number of wins of $i$ over $j$ follows a binomial distribution:
\begin{align}
w_{ij}\sim\text{Bin}\left(m_{ij}, \frac{\pi_i^*}{\pi_i^*+\pi_j^*}\right). \label{bt-model}
\end{align}
The {\em population} version of the comparison outcome is then given by the expectation as 
\begin{equation}\label{eq:barw0}
\bar w_{ij} \coloneqq \E[w_{ij}] = \frac{m_{ij}\pi_i^*}{\pi_i^*+\pi_j^*}.
\end{equation}
In practice, rather than analyzing a particular realization of the outcomes
$\W=\{w_{ij}\}_{i,j=1}^n$, which is subject to sampling randomness, 
it is often more informative from a theoretical perspective to consider the 
corresponding population version  $\bar\W=\{\bar w_{ij}\}_{i,j=1}^n$.
For example, as a desirable property of this population model, 
the MLE $\pih$~\eqref{mle0} based on the population log-likelihood function, 
i.e., \eqref{btl} with $w_{ij}$ replaced by $\bar w_{ij}$, 
exactly recovers the true strength vector $\bm\pi^*$,
which follows directly from Jensen's inequality.
As will be further explored in~\Cref{sec:app5},
when the sampled outcomes $\W$ concentrate around the population $\bar\W$,
the estimator $\pih$ is also expected to be close to $\bm\pi^*$.

To distinguish the population setting from the sampled setting, 
we denote the corresponding coefficient matrices in~\eqref{eq:nalpha} 
under the population comparison outcome $\bar\W$ by
\begin{equation}\label{eq:nalphabar}
\bar{\bm K}_\alpha(\bm \pi) 
\coloneqq \left(\bar{\bm C}(\bm\pi) + \alpha \bar{\bm R}(\bm\pi)\right)^{-1}
\cdot
\left(\bar{\bm W}(\bm\pi) + \alpha \bar{\bm R}(\bm\pi)\right),
\end{equation}
where the expected $\bar{\bm W}(\bm\pi)\coloneqq\E[\bm W(\bm\pi)]\in\R^{n\times n}$ is given by 
\begin{align}\label{myQbar}
[\bar{\bm W}(\bm\pi)]_{ij} = 
\frac{m_{ij}\pi_i^*}{(\pi_i + \pi_j)(\pi^*_i + \pi^*_j)}\cdot\mathbb I_{\{i\neq j\}},
\end{align}
and the expected diagonal matrices 
\begin{align}\label{eq:barrc}
\bar{\bm R}(\bm\pi) = \mbox{diag}(\bar{\bm W}(\bm\pi) \bm 1)
\quad\text{and}\quad
\bar{\bm C}(\bm\pi) = \mbox{diag}(\bar{\bm W}(\bm\pi)^\top \bm 1).
\end{align}
In analogy, the Jacobian for the population fixed-point mapping $\bar F_\alpha(\bm\pi)\coloneqq  \bar{\bm K}_\alpha(\bm \pi) \bm \pi$ is denoted by $\bar{\bm J}_{F_\alpha}$. It follows from~\eqref{jac:Jfp} that this Jacobian at the fixed point solution $\pih=\bm\pi^*$ 
(recall the MLE~\eqref{mle0} has the solution $\pih= \bm\pi^*$ under the expected outcome~\eqref{eq:barw0}) has entries 
\begin{align}\label{Jbar}
[\JJ_{F_\alpha}(\bm\pi^*)]_{ij}\coloneqq\begin{cases}
\left.{\left(\frac{m_{ij} \pi_i^*}{(\pi_i^* + \pi_j^*)^2}\right)}\middle/{\left(\sum_{\ell} \frac{m_{i\ell}(\alpha \pi_i^* + \pi_\ell^*)}{(\pi_i^* + \pi_\ell^*)^2}\right)}\right. & j\neq i,\\
\alpha \left.{\left(\sum_{\ell} \frac{m_{i\ell}\pi_i^*}{(\pi_i^*+\pi_\ell^*)^2}\right)}\middle/{\left(\sum_{\ell} \frac{m_{i\ell}(\alpha \pi_i^* + \pi_\ell^*)}{(\pi_i^* + \pi_\ell^*)^2}\right)}\right. & j = i.
\end{cases} 
\end{align}
A useful observation here is that 
\begin{equation}\label{eq:jjfstar}
\JJ_{F_\alpha}(\bm\pi^*) \equiv 
\left(\bar{\bm C}(\bm\pi^*) + \alpha \bar{\bm R}(\bm\pi^*)\right)^{-1}
\cdot
\left(\bar{\bm W}(\bm\pi^*) + \alpha \bar{\bm R}(\bm\pi^*)\right)
\equiv 
\bar{\bm K}_\alpha(\bm \pi^*). 
\end{equation}
Note that $\bm J_{F_\alpha}(\pih) \neq \bm K_\alpha(\pih)$ for matrices from~\eqref{eq:nalpha} and~\eqref{jac:Jfp} under a general outcome $\mathcal W$.

\section{Convergence analysis of synchronous updates}\label{sec:ana}

In this section, we focus on the local convergence rate of Newman's $\alpha$-scheme with synchronous updates. The study of synchronous updates is of interest in its own right and also provides the basis for understanding the local convergence of the asynchronous version discussed in~\Cref{sec:gs}. Our main finding is that the convergence rate is not necessarily monotonically increasing in $\alpha$, indicating that $\alpha=0$ may not always be the optimal choice. This contrasts with the observation in \cite{newman2023efficient}, where $\alpha=0$ is empirically observed to be optimal for the {\em asynchronous} version of the algorithm in a number of scenarios. These results highlight that synchronous vs. asynchronous updates can exhibit fundamentally different behavior with respect to the parameter $\alpha$, and the optimality of $\alpha=0$ is intrinsic to the asynchronous update of the algorithm in certain scenarios.

\subsection{Local convergence factors}\label{sec:linear1}
Recall the Newman's $\alpha$-scheme with synchronous updates in \eqref{eq:fpsf}:
\begin{equation*}
    \bm\pi^{(k)}=S\circ F_\alpha (\bm\pi^{(k-1)})
    \quad\text{for $k=1,2,\dots$}
\end{equation*} 
By standard local convergence theory for fixed-point iterations \citep{ortega2000iterative}, the sequence $\{\bm\pi^{(k)}\}$ converges locally to the fixed-point $\pih$, if the spectral radius of the Jacobian matrix $\bm J_{S\circ F_\alpha}(\pih)$ satisfies 
\begin{equation}\label{eq:sprd}
\rho\left( \bm J_{S\circ F_\alpha}(\pih) \right) < 1.
\end{equation}
In this case, the asymptotic convergence factor is governed by $\rho\left( \bm J_{S\circ F_\alpha}(\pih) \right)$.

Direct analysis of the Jacobian $\bm J_{S\circ F_\alpha}$ is cumbersome. Instead, we can work with the Jacobian of $F_\alpha$, denoted by $\bm J_{F_\alpha}$, by leveraging the spectral relationship between the two Jacobians as established below.
\begin{lemma}
The Jacobian matrix of the function $S\circ F_\alpha$ satisfies 
\begin{align} \label{normF}
\bm J_{S\circ F_\alpha }(\pih) =   \bm P_{\pih^\perp}\bm J_{F_\alpha}(\pih)
\quad\text{with\quad $\bm P_{\pih^\perp}\coloneqq \bm I - \frac{1}{n}\pih (\pih^{\circ -1})^\top$},
\end{align}
where $\bm x^{\circ -1}$ denotes entrywise reciprocal of a vector $\bm x$.
\end{lemma}
\begin{proof}
By the chain rule of differentiation, $\bm J_{S\circ F_\alpha }(\bm\pi) = \bm J_{S}(F_\alpha(\bm\pi)) \bm J_{F_\alpha}(\bm\pi)$, 
\[
\bm J_{S\circ F_\alpha }(\bm\pi) =  
\frac{1}{\left(\prod_if_i(\bm\pi)\right)^{1/n}}
\left( \bm I - \frac{1}{n} F_\alpha (\bm\pi) (F_\alpha (\bm\pi)^{\circ -1})^\top \right)
\bm J_{F_\alpha}(\bm\pi).
\]
We then obtain~\eqref{normF}  by recalling $\pih = F_\alpha(\pih)$ and the normalization condition $\prod_i\widehat{\pi}_i = 1$ at the solution $\bm\pi=\pih$.
\end{proof}

Recall that the Jacobian matrix $\bm J_{F_\alpha}$ at the fixed point $\pih=F_\alpha (\pih)$, given by~\eqref{jac:Jfp},
is similar to a symmetric matrix and has  $n$ real eigenvalues. Moreover, these eigenvalues lie to the left of $1$, as established below.
\begin{lemma}\label{lm:normspec}
Assume that $\G([n], \W)$ is strongly connected. 
\begin{enumerate}[label=(\alph*)]
\item\label{i:lm:normspec} The eigenvalues of $\bm J_{F_\alpha}(\pih)$ are real and can be ordered as
$\lambda_n\leq \cdots\leq\lambda_2<\lambda_1\equiv 1$. 
The leading eigenvalue $\lambda_1 =1$ is simple, with corresponding right-eigenvector $\pih$.
\item\label{ii:lm:normspec} $\bm J_{S\circ F_\alpha }(\pih)$ 
has the same eigenvalues  as  $\bm J_{F_\alpha}(\pih)$, 
except that $\lambda_1=1$ is replaced by $0$; that is, its eigenvalues are 
$\lambda_n, \ldots, \lambda_2$, and  $0$.
\end{enumerate}
\end{lemma} 

\begin{proof}
For \Cref{i:lm:normspec}, since 
$F_\alpha(\bm\pi)$ is homogeneous (i.e., $F_\alpha(k\bm\pi) \equiv k F_\alpha(\bm\pi)$ for all $k\neq 0$),
it follows from  Euler's theorem for homogeneous functions that  $F_\alpha(\bm\pi)\equiv \bm J_{F_\alpha}(\bm\pi)\bm\pi$.
Hence, the fixed point equation $\pih = F_\alpha(\pih)$ implies $\pih = \bm J_{F_\alpha}(\pih)\bm\pih$, i.e., 
$\pih$ is an eigenvector of $\bm J_{F_\alpha}$ with eigenvalue $\lambda=1$. 

It remains to show that $\lambda = 1$ is simple and all other eigenvalues of $\bm J_{F_\alpha}(\pih)$ are less than 1. 
Observe that $\bm J_{F_\alpha}(\pih)$ is quasi-positive, i.e., its off-diagonal entries are nonnegative,
and under the strong connectivity assumption it is also irreducible.
Therefore, by the Perron--Frobenius theorem for irreducible quasi-positive matrices~\citep{berman1994nonnegative},
there exists a unique (up to scaling) positive eigenvector 
corresponding to a simple and rightmost eigenvalue  of $\bm J_{F_\alpha}(\pih)$,
known as the Perron eigenvalue.
Since $\pih>0$ is an eigenvector associated with $\lambda=1$, 
it follows that $\lambda=1$ is the Perron eigenvalue,
and consequently,  $\lambda = 1$ is simple  and the remaining eigenvalues are less than $1$.

For \Cref{ii:lm:normspec}, recalling that  $\bm J_{F_\alpha}(\pih)$ is similar to a symmetric matrix, it admits real eigenvalues $\lambda_1\geq\cdots\geq\lambda_n$ with corresponding linearly independent (right) eigenvectors $\bm x_1, \dots, \bm x_n$ satisfying $\bm J_{F_\alpha}(\pih)\bm x_i = \lambda_i \bm x_i$. By \Cref{i:lm:normspec}, we have $\lambda_1 = 1$ and $\bm x_1=\bm\pih$. From the definition in~\eqref{normF}, it follows immediately that $\bm J_{S\circ F_\alpha}(\pih)\bm x_1=0$ and $\bm J_{S\circ F_\alpha}(\pih) \cdot\bm P_{\pih^\perp} \bm x_i =\lambda_i\cdot \bm P_{\pih^\perp} \bm x_i$, for $i=2,\dots, n$. Consequently,  $\bm J_{S\circ F_\alpha}(\pih)$ has 
eigenvalues $0, \lambda_2, \ldots, \lambda_n$, with corresponding linearly independent eigenvectors $\bm y_1=\pih$ and $\bm y_i= \bm P_{\pih^\perp} \bm x_i$, for $i=2, \dots, n$. This completes the proof.
\end{proof}

According to~\Cref{lm:normspec}, the local convergence factor 
of Newman's $\alpha$-scheme with synchronous updates can be formulated as follows.
\begin{theorem}\label{thm:rhosync}
Assume that $\G([n], \W)$ is strongly connected. 
The local convergence factor of Newman's $\alpha$-scheme with synchronous updates in~\eqref{eq:fpsf},
for convergence to the MLE $\pih$,
is given by 
\begin{equation}\label{eq:rhosync}
\rhosync (\alpha)\coloneqq\max \left\{|\lambda|\colon \lambda \in \Lambda(\bm J_{F_\alpha}(\pih)),\ \lambda < 1 \right\}
\equiv  \max \{\lambda_2(\bm J_{F_\alpha}(\pih) ), -\lambda_n(\bm J_{F_\alpha}(\pih))\},
\end{equation}
where $\lambda_i(\bm J_{F_\alpha}(\pih))$ denotes the $i$th largest eigenvalue
(in algebraic order) of the Jacobian~\eqref{jac:Jfp}.
\end{theorem}

We caution that the local convergence factor $\rhosync (\alpha)$ is not always bounded by $1$. When $\alpha$ is large (e.g., $\alpha\geq 1$), $\bm J_{F_\alpha}(\pih)$ is nonnegative, and all its nonleading eigenvalues have modulus less than $1$ by the Perron--Frobenius theorem, hence $\rhosync (\alpha)<1$. However, when $\alpha$ is small, the convergence factor $\rhosync(\alpha)$ may exceed $1$ for a general comparison outcome $\W$, leading to local divergence of the algorithm. For instance, this occurs if the outcome $\W$ severely deviates from the expected outcome of the BT model; see \Cref{sec:num1} for an example. Under such circumstances, finding the MLE~\eqref{mle0} itself becomes questionable in practice. However, this divergence will not occur when the outcome follows the assumed BT model asymptotically; see \Cref{sec:gap1} and \Cref{sec:app5}.


\subsection{Variation of convergence factors}\label{sec:gap1}
In this section, we investigate the variational behavior of the local convergence factor $\rhosync(\alpha)$ in the parameter $\alpha$.
We will focus on the expected comparison outcome from a BT model~\eqref{bt-model}:
\begin{equation}\label{eq:barw}
\bar w_{ij} = \frac{m_{ij}\pi_i^*}{\pi_i^*+\pi_j^*},
\end{equation}
and consider the corresponding convergence factor from~\eqref{eq:rhosync}, denoted as 
\begin{equation}
\brhosync(\alpha)
\coloneqq\max \left\{|\lambda|\colon \lambda \in \Lambda(\JJ_{F_\alpha}(\bm\pi^*)),\ \lambda < 1 \right\}
\equiv  \max \{\lambda_2(\JJ_{F_\alpha}(\bm\pi^*)), -\lambda_n(\JJ_{F_\alpha}(\bm\pi^*))\},
\label{eq:brhosync}
\end{equation}
where $\JJ_{F_\alpha}(\bm\pi^*)$, as given by~\eqref{Jbar},
denotes the Jacobian $\bm J_{F_\alpha}(\pih)$~\eqref{jac:Jfp} corresponding to the expected outcome $\W=\bar \W$ and the MLE $\pih=\bm \pi^*$. 
This population formulation removes sampling variability and yields a deterministic characterization of the convergence factor $\brhosync(\alpha)$, which serves as an asymptotically equivalent counterpart to $\rhosync(\alpha)$; see \Cref{sec:app5} for further discussions.

For the population formulation~\eqref{Jbar}, we show that the corresponding convergence factor $\brhosync(\alpha)$ is a quasi-convex function in $\alpha\in[0,\infty)$. The following monotonicity property of the eigenvalues of  $\JJ_{F_\alpha}(\bm\pi^*)$ is key to our development. 

\begin{lemma}\label{thm:comp}
Assume that $\G([n], \M)$ is connected and let $\kappa_L\coloneqq\max_{m_{\ell j}\neq 0}\pi_\ell^*/\pi_j^*$ denote the local dynamic range of $\bm\pi^*\in\R_+^n$. There exist real analytical functions $\psi_1(\alpha),\dots,\psi_n(\alpha)$ of $\alpha$ on $\R_+$ such that for every $\alpha\geq 0$, the set $\{\psi_i(\alpha)\}_{i}$ consists of the $n$ (unordered) eigenvalues of $\JJ_{F_\alpha}(\bm\pi^*)$ in~\eqref{Jbar}. Each $\psi_i(\alpha)$ is monotonically increasing, and its derivative is bounded by
\begin{equation}\label{eq:derbnd}
\frac{1-\psi_i(\alpha)}{\alpha+\kappa_L} \leq \psi_i'(\alpha) \leq \frac{{1-\psi_i(\alpha)}}{\alpha + \kappa_L^{-1}},
\quad\mathrm{for}\ \alpha \geq 0.
\end{equation}
Consequently, the function values are bounded by linear rational functions of $\alpha$ as 
\begin{equation}\label{eq:valbnd}
\frac{\psi_i(0)+\kappa_L^{-1}\alpha}{1+\kappa_L^{-1}\alpha}
\leq  \psi_i(\alpha) \leq
\frac{\psi_i(0)+ \kappa_L\alpha}{1+\kappa_L\alpha},
\quad\mathrm{for}\ \alpha \geq 0.
\end{equation}
In particular, $\lim_{\alpha\to\infty} \psi_i(\alpha)=1$.
\end{lemma}
\begin{proof}
We first recall that the matrix $\JJ_{F_\alpha}(\bm\pi^*)$ is {\em similar} to the symmetric matrix 
\[ \bm \JJ^\sym_{F_\alpha}(\bm\pi^*) \coloneqq \bar{\bm D}_\alpha(\bm\pi^*)^{-1}\JJ_{F_\alpha}(\bm\pi^*)\bar{\bm D}_\alpha(\bm\pi^*), \]
where 
$
\bar{\bm D}_\alpha(\bm\pi^*) =
\diag(\bm\pi^*)^{1/2}\cdot 
\left(\bar{\bm C}(\bm \pi^*)+\alpha \bar{\bm R}(\bm\pi^*)\right)^{-1/2}
$
is the population version of the diagonal matrix $\bm D_\alpha(\pih)$ from~\eqref{myD}, with $\mathcal W=\bar {\mathcal W}$ and $\pih=\bm\pi^*$.
By classical spectral analysis \citep{Kato:2013, Rellich:1969}, 
the eigenvalues and eigenvectors of the symmetric $\bm \JJ^\sym_{F_\alpha}(\bm\pi^*)$ consist of analytical functions in $\alpha\geq 0$.
Consequently,  due to the similarity $\JJ_{F_\alpha}(\bm\pi^*) = \bar{\bm D}_\alpha(\bm\pi^*)\bm \JJ^\sym_{F_\alpha}(\bm\pi^*)\bar{\bm D}_\alpha(\bm\pi^*)^{-1}$,
the eigenvalues of $\JJ_{F_\alpha}(\bm\pi^*)$ consist of $n$ analytical functions $\psi_1(\alpha),\dots, \psi_n(\alpha)$.
Moreover,  the left and right eigenvectors associated with $\psi_i(\alpha)$
can be chosen as analytical functions $\bm u_i(\alpha)$ and $\bm v_i(\alpha)$, related by 
\begin{align}
\bm v_i(\alpha) =  \bar{\bm D}_\alpha(\bm\pi^*)^2 \bm u_i(\alpha)
\equiv 
\diag(\bm\pi^*) \cdot 
\left(\bar{\bm C}(\bm \pi^*)+\alpha \bar{\bm R}(\bm\pi^*)\right)^{-1}
\cdot 
\bm u_i(\alpha).
\label{uvrelate}
\end{align} 
The derivative of the eigenvalue functions $\psi_i$ admits the expression 
\begin{equation}\label{eq:vpi0}
\psi_i'(\alpha) = \frac{\bm u^\top_i(\alpha) \left (\partial_\alpha\JJ_{F_\alpha}(\bm\pi^*)\right) \bm v_i(\alpha)}{\bm u^\top_i(\alpha)\bm v_i(\alpha)},
\end{equation}
where $\partial_\alpha\JJ_{F_\alpha}(\bm\pi^*)$ denotes the derivative of $\JJ_{F_\alpha}(\bm\pi^*)$ in $\alpha$. 
This expression is classical and follows from differentiating the eigenvalue equation
$\JJ_{F_\alpha}(\bm\pi^*)\bm v_i(\alpha) = \psi_i(\alpha) \bm v_i(\alpha)$ and then  left-multiplying by $\bm u^\top_i$.

Now, recall~\eqref{eq:jjfstar} that 
$\JJ_{F_\alpha}(\bm\pi^*)  \equiv (\bar{\bm C}(\bm\pi^*) + \alpha \bar {\bm R}(\bm\pi^*)^{-1} (\bar{\bm W}(\bm\pi^*) + \alpha \bar{\bm R}(\bm\pi^*))$.
By the chain rule, 
\[
\partial_\alpha \JJ_{F_\alpha}(\bm\pi^*) =  (\bar{\bm C}(\bm\pi^*) + \alpha \bar {\bm R}(\bm\pi^*))^{-1} \bar{\bm R}(\bm\pi^*)(\bm I-\JJ_{F_\alpha}(\bm\pi^*)). 
\]
Plugging into~\eqref{eq:vpi0} and using the eigenvalue equation $ \JJ_{F_\alpha}(\bm\pi^*) \bm v_i(\alpha) = \psi_i(\alpha) \bm v_i(\alpha)$, we obtain 
\begin{equation}\label{eq:vpi}
\psi_i' = \frac{\bm u(\alpha)^\top_i
(\bar{\bm C}(\bm\pi^*) + \alpha \bar {\bm R}(\bm\pi^*))^{-1} 
\bar {\bm R}(\bm\pi^*) \bm v_i(\alpha)}{\bm u_i(\alpha)^\top\bm v_i(\alpha)} \cdot(1-\psi_i)
=  \frac{1-\psi_i}{\alpha+c_i}, 
\end{equation}
where for the last equation we substitute $\bm u_i$ by~\eqref{uvrelate} and define
\[
c_i\coloneqq\frac{\bm v_i(\alpha)^\top \left(\diag(\bm \pi^*)^{-1}\bar{\bm C}(\bm \pi^*)\right) \bm v_i(\alpha)}{\bm v_i(\alpha)^\top 
\left(\diag(\bm \pi^*)^{-1}\bar{\bm R}(\bm \pi^*)\right)  \bm v_i(\alpha)}.  
\]
Since $c_i$ is the Rayleigh quotient of a pair of diagonal matrices $\bm A_1\coloneqq\diag(\bm \pi^*)^{-1}\bar{\bm C}(\bm \pi^*)$ and $\bm A_2\coloneqq\diag(\bm \pi^*)^{-1}\bar{\bm R}(\bm \pi^*)$, it follows from the classical eigenvalue minimization principle that $c_i$ is bounded from below and above by the smallest and largest eigenvalue of the diagonal matrix
$\bm A_2^{-1}\bm A_1=  \bar{\bm R}(\bm \pi^*)^{-1} \bar{\bm C}(\bm \pi^*)$;
namely,
\[
\min_{j}[\bar{\bm R}(\bm\pi^*)^{-1}\bar{\bm C}(\bm\pi^*)]_{jj}\leq c_i \leq \max_{j}[\bar{\bm R}(\bm\pi^*)^{-1}\bar{\bm C}(\bm\pi^*)]_{jj}.
\]
A quick verification also shows 
\[
\kappa_L^{-1}\equiv\min_{m_{\ell j}\neq 0}\frac{\pi_\ell^*}{\pi_j^*}\leq [\bar{\bm R}(\bm\pi^*)^{-1}\bar{\bm C}(\bm\pi^*)]_{jj}\leq\max_{m_{\ell j}\neq 0}\frac{\pi_\ell^*}{\pi_j^*}\equiv\kappa_L. 
\]
We hence prove~\eqref{eq:derbnd} upon noticing the eigenvalues $\psi_i\leq 1$. The eigenvalue bounds~\eqref{eq:valbnd} follow directly from~\eqref{eq:derbnd} by integrating $\int_{0}^\alpha \psi_i'(t)/(1-\psi_i(t))~\mathrm{d} t$.
\end{proof}

By the monotonicity of eigenvalues in~\Cref{thm:comp}, the convergence factor $\brhosync(\alpha)$ from~\eqref{eq:brhosync} is the maximum of an increasing function $\lambda_2(\cdot)$ and a decreasing function $-\lambda_n(\cdot)$. Hence, it is a quasi-convex function on  $\R_+$, and there exists a unique optimal parameter 
\begin{equation}\label{eq:optalpha}
\alpha_{\rm opt} = \argmin_{\alpha\in[0,\infty)} (\brhosync(\alpha)).
\end{equation}
Two possible scenarios may occur: 
(a) $\alpha_{\rm opt} = 0$, where $\brhosync(\alpha)$ is monotonically increasing on $\mathbb R_+$
and  (b)  $\alpha_{\rm opt} > 0$,  where $\brhosync(\alpha)$ is monotonically decreasing on $[0,\alpha_{\rm opt})$ and monotonically increasing on $[\alpha_{\rm opt},\infty)$. The following result characterizes both scenarios. 

\begin{theorem}\label{cor:optalpha}
Assume that $\G([n], \M)$ is connected and let $\alpha_{\rm opt}$ be defined as in~\eqref{eq:optalpha}. 
Let $s_0=\frac{1}{2}(\lambda_n(\JJ_{F_0}(\bm\pi^*)) +\lambda_2(\JJ_{F_0}(\bm\pi^*)))$. 
\begin{enumerate}[label=(\alph*)]
    \item\label{aesaa}
    If  $s_0\geq 0$, then  $\alpha_{\rm opt}=0$, and for all $\alpha \geq 0$, we can bound the convergence factor $\brhosync(\alpha)$ as 
    \begin{align}\label{hsjfhsa}
    \frac{ \brhosync(0)+\kappa_L^{-1}\alpha }{1 + \kappa_L^{-1}\alpha }
    \leq 
    \brhosync(\alpha) 
    \leq 
    \frac{ \brhosync(0)+\kappa_L\alpha}{1 + \kappa_L\alpha}.
    \end{align}
    \item \label{aesab}
    If  $s_0< 0$,  then $0<\alpha_{\rm opt}<\infty$.
    Moreover,  we can bound
    $|s_0| {\kappa_L}^{-1} \leq  \alpha_{\rm opt} \leq  |s_0| {\kappa_L}$,
    with 
   \begin{align}
    \frac{ \brhosync(0) - \kappa_L^2|s_0|}{1+\kappa_L^2|s_0|}
    \leq 
    \brhosync(\alpha_{\rm opt})
    \leq 
    \frac{ \brhosync(0) - \kappa_L^{-2}|s_0|}{1+\kappa_L^{-2}|s_0|}.
    \end{align}
 
 \end{enumerate}
\end{theorem}

\begin{proof}
For convenience, we assume $\psi_i(\alpha)$ corresponds to the eigenvalue function started from $\lambda_i(\JJ_{F_0})$. 
By \eqref{eq:valbnd} and~\Cref{lm:normspec}, $\psi_1(\alpha)\equiv \lambda_1(\JJ_{F_\alpha})\equiv 1$ remains the simple leading eigenvalue 
for all $\alpha\geq 0$,  and  $\psi_i(\alpha)<1$ for $i=2,\dots,n$ are all monotonically increasing in $\alpha$.
Clearly,
\[
\lambda_2( \JJ_{F_\alpha}(\bm\pi^*))\equiv  \max_{i=2:n} \psi_i(\alpha) 
\quad \text{and}\quad
\lambda_n( \JJ_{F_\alpha}(\bm\pi^*)) \equiv \min_{i=2:n} \psi_i(\alpha),
\]
and both are monotonically increasing in $\alpha$.
We also introduce the indicating function
\begin{equation}\label{eq:salpha}
s(\alpha) := 
\frac{1}{2}\left(\lambda_2( \JJ_{F_\alpha}(\bm\pi^*)) + \lambda_n( \JJ_{F_\alpha}(\bm\pi^*))\right)
\equiv 
\frac{1}{2}\left( \max_{i=2:n} \psi_i(\alpha)  +  \min_{i=2:n} \psi_i(\alpha) \right),
\end{equation}
which is monotonically increasing in $\alpha\geq 0$  with an initial $s(0)=s_0$. Observe that $s(\alpha)\geq 0$ indicates $\brhosync(\alpha) \equiv  \lambda_2( \JJ_{F_\alpha}(\bm\pi^*))$, and $s(\alpha)\leq 0$ indicates  $\brhosync(\alpha) \equiv  -\lambda_n( \JJ_{F_\alpha}(\bm\pi^*))$.

For~\Cref{aesaa},  the initial condition $s(0)=s_0\geq 0$ implies the monotonic function $s(\alpha) \geq 0$
for all $\alpha\geq 0$.
It indicates that  $\brhosync(\alpha) = \lambda_2(\JJ_{F_\alpha}(\bm\pi^*))$,
where we can bound the second eigenvalue by 
\begin{align*}
\frac{\psi_2(0)+ \kappa^{-1}_L\alpha}{1+\kappa^{-1}_L\alpha}\stackrel{\eqref{eq:valbnd}}{\leq}\psi_2(\alpha)\leq
\lambda_2(\JJ_{F_\alpha}(\bm\pi^*)) \equiv
\max_{i\geq 2}\psi_i(\alpha)\stackrel{\eqref{eq:valbnd}}{\leq} \max_{i\geq 2}\frac{\psi_i(0)+ \kappa_L\alpha}{1+\kappa_L\alpha}\leq \frac{\psi_2(0)+ \kappa_L\alpha}{1+\kappa_L\alpha}.
\end{align*} 
Observing that  $ \psi_2(0)=\lambda_2( \JJ_{F_0}(\bm\pi^*))\equiv \brhosync(0)$
yields the result in \Cref{aesaa}.

For~\Cref{aesab}, 
by the  initial condition $s(0)=s_0<0$ and the fact $s(\alpha)\to 1$ as $\alpha\to \infty$,
the monotonic indicating function  $s(\alpha)$ must change sign over $[0,\infty)$.
Therefore, by the intermediate value theorem, there exists a unique $\alpha_{\rm opt}\in(0,\infty)$ satisfying $s(\alpha_{\rm opt}) = 0$:
For $\alpha<\alpha_{\rm opt}$, we have $s(\alpha)\leq 0$, indicating 
$\brhosync(\alpha_{\rm opt}) = -\lambda_n( \JJ_{F_\alpha}(\bm\pi^*))$ is monotonically decreasing in $\alpha$;
For $\alpha\geq \alpha_{\rm opt}$, we have $s(\alpha)\geq 0$, indicating 
$\brhosync(\alpha_{\rm opt}) = \lambda_2( \JJ_{F_\alpha}(\bm\pi^*))$ is monotonically increasing in $\alpha$.
Consequently,  $\brhosync(\alpha)$ achieves its minimum at $\alpha_{\rm opt}$.

To bound $\alpha_{\rm opt}$,  applying \eqref{eq:valbnd} to each $\psi_i$ in~\eqref{eq:salpha} yields immediately
\begin{align*}
\frac{s_0 + \kappa_L^{-1}\alpha}{1+\kappa_L^{-1}\alpha} 
\leq 
s(\alpha) 
\leq 
\frac{s_0 + \kappa_L\alpha}{1+\kappa_L\alpha}.
\end{align*}
Then  $s(\alpha_{\rm opt}) = 0$ leads to the bounds
$|s_0| {\kappa_L}^{-1} \leq  \alpha_{\rm opt} \leq  |s_0| {\kappa_L}$.
Combining these bounds with the monotonicity of $\psi_i$, and
$\brhosync(\alpha) = -\lambda_n( \JJ_{F_\alpha}(\bm\pi^*)) = -\min_{i\geq 2} \psi_i(\alpha)$ for 
$\alpha=0, \alpha_{\rm opt}$,
we obtain 
\begin{align*}
\brhosync(\alpha_{\rm opt}) 
& =  -\min_{i\geq 2} \psi_i(\alpha_{\rm opt})
\geq -\min_{i\geq 2} \psi_i(|s_0|\kappa_L) 
\stackrel{\eqref{eq:valbnd}}{\geq} -\min_{i\geq 2}\frac{\psi_i(0) + \kappa_L^2|s_0|}{1 + \kappa_L^2|s_0|} 
= \frac{\brhosync(0) - \kappa_L^2|s_0|}{1+\kappa_L^2|s_0|};\\
\brhosync(\alpha_{\rm opt}) 
& =  -\min_{i\geq 2} \psi_i(\alpha_{\rm opt})
\leq -\min_{i\geq 2} \psi_i(|s_0|\kappa_L^{-1}) 
\stackrel{\eqref{eq:valbnd}}{\leq} -\min_{i\geq 2}\frac{\psi_i(0) + \kappa_L^{-2}|s_0|}{1 + \kappa_L^{-2}|s_0|} 
=
\frac{\brhosync(0) - \kappa_L^{-2}|s_0|}{1+\kappa_L^{-2}|s_0|}.
\end{align*}
\end{proof}

\begin{remark}\label{rmk:ssls}
The optimal $\alpha_{\rm opt}$ is not necessarily attained at $\alpha=0$ if $s_0< 0$. This occurs, for instance, when the comparison graph is bipartite. In this case, $\lambda_n(\JJ_{F_0})=-1$ and thus $s_0<0$ and $\brhosync(0)=1$. Thus, Newman's $0$-scheme with synchronous updates may not converge. For $\alpha>0$, \Cref{cor:optalpha} suggests that $\brhosync(\alpha)<1$ and decreases monotonically until reaching its minimum, after which it begins to increase. As a result, the convergence rate of Newman's $\alpha$-scheme with synchronous updates varies non-monotonically as $\alpha$ increases. These interpretations rely on the assumption that $\brhosync(\alpha)$ well approximates $\rhosync(\alpha)$, which is established in \Cref{sec:app5}. 
\end{remark}

\begin{remark}\label{rem:dynamicrange}
The results in~\Cref{thm:comp} and~\Cref{cor:optalpha} show that the local dynamic range $\kappa_L$ plays a meaningful role in bounding the convergence rate. In those bounds, $\kappa_L$ may also be replaced by the global dynamic range 
$\kappa\coloneqq\max_{\ell j}\pi_\ell^*/\pi_j^*$, which is independent of the structure of the comparison 
graph $\M$ (i.e., nonzero patterns of $\{m_{ij}\}$). In practice, however, $\kappa_L$ can be substantially smaller than $\kappa$,
especially when objects are compared only with those of similar strengths, leading to significantly tighter bounds.
\end{remark}

\begin{remark}\label{rem:inter}
When $s_0\geq 0$, we can rewrite the estimates in \eqref{hsjfhsa} in \Cref{cor:optalpha} using the spectral gap $1-\brhosync(\alpha)$ to obtain a multiplicative bound: 
\begin{align*}
(1+\kappa^{-1}_L\alpha)(1-\brhosync(\alpha))\leq 1-\brhosync(0)\leq (1+\kappa_L\alpha)(1-\brhosync(\alpha)). 
\end{align*}
This bound is tight when $\kappa_L$ is close to 1. In this regime, Newman's 0-scheme with synchronous updates is approximately twice as efficient as Zermelo's algorithm ($\alpha = 1$). More generally, the potential acceleration increases with $\kappa_L$.        
\end{remark}


\section{Convergence analysis of asynchronous updates}\label{sec:gs}

In this section, we study Newman's $\alpha$-scheme using asynchronous updates, where a single component $\pi_i$ is updated at a time while all other components $\pi_j$ with $j \neq i$ remain fixed. We assume a sequential update order for $i = 1, \dots, n$, but the analysis extends directly to any fixed, but otherwise arbitrary, ordering of the component updates. As before, we focus on the local convergence analysis via fixed-point analysis. We begin with a closed-form expression of the local convergence factor and show that the asynchronous scheme is always locally convergent for $\alpha \geq 0$, provided that $\G([n], \W)$ is strongly connected. This stands in contrast to the results in the synchronous setting from~\Cref{sec:ana} and is consistent with the global convergence result from \cite{newman2023efficient}, which was obtained via a different optimization-based approach. Moreover, this explicit characterization also allows us to look more closely at the role of the parameter $\alpha$ in convergence. In particular, for bipartite and consistently ordered comparison graphs, we show that the local convergence rate increases monotonically with $\alpha$, and for general graphs, it leads to a quantitative characterization of the possible range of convergence factors. 

\subsection{Local convergence factors}

Recall Newman's $\alpha$-scheme with asynchronous updates in \eqref{eq:fpasyn}:
\begin{equation*}
\bm\pi^{(k)} = S\circ A_{\alpha}(\bm\pi^{(k-1)}).
\end{equation*}
Since $\pih$ is a fixed point of $S\circ A_\alpha$, the above scheme is locally convergent to $\pih$ if the spectral radius  of the Jacobian matrix satisfies 
\begin{equation}
\rho(\bm J_{S\circ A_\alpha}(\pih)) < 1,
\end{equation}
in which case the local convergence rate is given by $\rho(\bm J_{S\circ A_\alpha}(\pih))$. Similar to the discussion in \Cref{sec:linear1}, the Jacobian matrix $\bm J_{S\circ A_\alpha}$ can be studied via the unnormalized Jacobian $\bm J_{A_\alpha}$, which in turn is related to the original Jacobian matrix $\bm J_{F_\alpha}$ from the synchronized setting~\eqref{eq:fpsf}.
\begin{lemma}\label{lm:4.1}
The Jacobian matrix of the function $S\circ A_\alpha$ satisfies 
\begin{align} \label{normA}
\bm J_{S\circ A_\alpha}(\pih) = \bm P_{\pih^\perp}\bm J_{A_\alpha}(\pih)\
\quad\text{with\quad $\bm P_{\pih^\perp}\coloneqq \bm I - \frac{1}{n}\pih (\pih^{\circ -1})^\top$},
\end{align}
where $\bm x^{\circ -1}$ denotes entrywise reciprocal of a vector $\bm x$.
Moreover, 
\begin{align}
\bm J_{A_\alpha}(\pih) = (\bm I - \bm J_{F_\alpha, \l}(\pih))^{-1}(\bm J_{F_\alpha, \d}(\pih) + \bm J_{F_\alpha, \u}(\pih)), \label{crucial}
\end{align}
where $\bm J_{F_\alpha, \l}(\pih)$, $\bm J_{F_\alpha, \d}(\pih)$, and $\bm J_{F_\alpha, \u}(\pih)$ are the lower triangular, diagonal, and upper triangular parts of the Jacobian $\bm J_{F_\alpha}(\pih)$ 
defined in \eqref{jac:Jfp}. 
\end{lemma}

\begin{proof}
\cref{normA} follows directly from the same argument used to derive~\eqref{normF}.  
\cref{crucial} follows from a standard analysis for the nonlinear Gauss--Seidel method~\cite[E 10.3-7, pp333]{ortega2000iterative}. For completeness, we sketch the proof here. Recall that $\pih = F_{\alpha, i}(\pih)$ for all $i\in [n]$. 
We can therefore apply the chain rule of differentiation
to the composition $A_\alpha=F_{\alpha, n}\circ\dots\circ   F_{\alpha, 1}$ to obtain
\begin{align}
\bm J_{A_\alpha}(\pih) = \bm J_{F_{\alpha, n}}(\pih)\cdots \bm J_{F_{\alpha, 1}}(\pih),\label{J_A_new}
\end{align}
where, by definition, 
\[\bm J_{F_{\alpha, i}}(\pih) = \bm I + \bm e_i (\text{row}_i(\bm J_{F_\alpha}(\pih)) - \bm e_i^\top),\]
with $\bm e_i$ denoting the $i$th column of the identity matrix $\bm I_n$ 
and $\text{row}_i(\cdot)$ extracting $i$-th row of a matrix.
Substituting this back to \eqref{J_A_new} and rearranging terms yields~\eqref{crucial}. 
\end{proof}

\begin{remark}
\cref{crucial} indicates that $\bm J_{A_\alpha}(\pih) \coloneqq \bm M^{-1}\bm N$ arises from a splitting of the matrix 
$\bm I - \bm J_{F_\alpha}(\pih) = \bm M - \bm N$ with $\bm M=\bm I - \bm J_{F_\alpha, \l}(\pih)$ and $\bm N=\bm J_{F_\alpha, \d}(\pih) + \bm J_{F_\alpha, \u}(\pih)$. This resembles the classical Gauss--Seidel splitting, where $\bm M= \bm I - \bm J_{F_\alpha, \l}(\pih) -  \bm J_{F_\alpha, \d}(\pih)$ and $\bm N=\bm J_{F_\alpha, \u}(\pih)$. The two splittings are generally different but may coincide when $\bm J_{F_\alpha, \d}(\pih) \equiv \bm 0$ (e.g., when $\alpha = 0$). 
\end{remark}

In analogy with~\Cref{lm:normspec}, the eigenvalues of $\bm J_{S\circ A_\alpha}(\pih)$ can then be directly related to those of $\bm J_{A_\alpha}(\pih)$, as shown below.

\begin{lemma}\label{lm:normspec2}
Assume that $\G([n], \W)$ is strongly connected. Then
\begin{enumerate}[label=(\alph*)]
\item\label{i:lm:normspec2} $\bm J_{A_\alpha}(\pih)$ has a simple eigenvalue $\lambda =1$ with 
corresponding right-eigenvector $\pih$; all other eigenvalues of $\bm J_{A_\alpha}(\pih)$ have modulus strictly less than $1$. 

\item\label{ii:lm:normspec2} $\bm J_{S\circ A_\alpha}(\pih)$ has the same eigenvalues  as $\bm J_{A_\alpha}(\pih)$, 
except that $\lambda=1$ is replaced by $0$.
\end{enumerate}
\end{lemma} 
\begin{proof}
To begin with, we derive from the definition~\eqref{crucial}
that 
\begin{equation}
\bm J_{A_\alpha}(\pih) = 
\bm I - \bm M^{-1}\bm L 
\quad \text{with\quad 
$\bm M\coloneqq  \bm I - \bm J_{F_\alpha, \l}(\pih)$,
\quad 
$\bm L\coloneqq  \bm I - \bm J_{F_\alpha}(\pih) $.
}
\label{eq:normspec2:fact1}
\end{equation}
As a core observation for our subsequent analysis, both factors $\bm M$ and $\bm L$ in~\eqref{eq:normspec2:fact1}
are definite matrices after similarity transformations using the diagonal $\bm D_\alpha(\pih)$ from \eqref{myD}. To see this, we first recall from~\eqref{eq:jsym} that $\bm J_{F_\alpha}(\pih)$ is similar to the symmetric matrix $\bm J^{\sym}_{F_\alpha}(\pih)= \bm D_\alpha(\pih)^{-1} \bm J_{F_\alpha}(\pih) \bm D_\alpha(\pih)$, whose eigenvalues are real and bounded by $1$ (\Cref{lm:normspec}).
Consequently, $\bm L$ becomes a symmetric positive semi-definite matrix under the similarity transformation:
\begin{equation}
\bm L^\sym\coloneqq  
\bm D_\alpha(\pih)^{-1} \bm L \bm D_\alpha(\pih) 
\equiv
\bm I-  \bm J^{\sym}_{F_\alpha}(\pih) 
\succeq \bm 0.
\label{eq:normspec2:fact2}
\end{equation}
On the other hand, the transformed $\bm M$, denoted by $\bm M^\star$, is a (non-symmetric) positive definite matrix
\begin{equation}
\bm M^\star\coloneqq  
\bm D_\alpha(\pih)^{-1}\bm M \bm D_\alpha(\pih) 
\equiv
\bm I-  \bm J^{\sym}_{F_\alpha,\l}(\pih)
\succ \bm 0.
\label{eq:normspec2:fact3}
\end{equation}
Indeed, for all nonzero $\bm v\in\C^{n}$,
\begin{equation}\label{eq:pdm}
2\mbox{Re}(\bm v^{\mathrm{H}} \bm M^\star
\bm v)
= \bm v^{\mathrm{H}} (\bm M^\star
+ (\bm M^\star)^\top) \bm v
= \bm v^{\mathrm{H}} (2\bm I+ \bm J^{\sym}_{F_\alpha,\d}(\pih)
- \bm J^{\sym}_{F_\alpha}(\pih)) \bm v
> \bm v^{\mathrm{H}} \bm L^\sym  \bm v \geq 0,
\end{equation}
where the strict inequality $>$ follows from 
the bound $[\bm J_{F_\alpha}(\pih)]_{ii}>-1$
in~\eqref{eq:boundjii},
together with the definition and definiteness of $\bm L^\sym$ in \eqref{eq:normspec2:fact2}.

For~\Cref{i:lm:normspec2}, recall from~\Cref{lm:normspec} that 
$\bm J_{F_\alpha}(\pih)$ has a simple eigenvalue $\lambda=1$, with 
a corresponding right eigenvector $\pih$ and left eigenvector 
$ \bm D_\alpha(\pih)^{-2}\pih$ (due to the symmetry in~\eqref{eq:jsym}).
It follows from~\eqref{eq:normspec2:fact1}
that 
$\bm I- \bm J_{A_\alpha}(\pih) \equiv \bm M^{-1}( \bm I- \bm J_{F_\alpha}(\pih))$, and hence, the matrix $\bm J_{A_\alpha}(\pih)$ shares the same eigenvalue $\lambda=1$ as $\bm J_{F_\alpha}(\pih)$. Moreover, it has a unique (up to scaling) right eigenvector $\bm u\coloneqq \pih$ and left eigenvector $\bm v\coloneqq \bm M^{\mathrm H}\bm D_\alpha(\pih)^{-2}\pih$ corresponding to $\lambda = 1$, which satisfy 
\[ 
\bm v^{\mathrm{H}}\bm u\equiv \pih^{\mathrm H} \bm D_\alpha(\pih)^{-2}\bm M \pih
=(\bm D_\alpha(\pih)^{-1}\pih)^{\mathrm H}\cdot \bm M^\star\cdot (\bm D_\alpha(\pih)^{-1}\pih) \neq 0,
\]
where the last inequality follows from the fact $\bm M^\star\succ 0$
by~\eqref{eq:normspec2:fact3}, which implies $\mbox{Re}(\bm v^{\mathrm{H}}\bm u)>0$.
This in turn implies the eigenvalue $\lambda=1$ must be algebraically simple,
as it has a finite eigenvalue condition number $(\|\bm u\|_2\|\bm v\|_2)/|\bm v^{\mathrm{H}}\bm u|$; see, e.g.,~\cite{Stewart:1990}.

Next, we show that if an eigenvalue $\lambda\in\C$ of $\bm J_{A_\alpha}(\pih)$ satisfies $\lambda\neq 1$, then it has modulus $|\lambda|<1$.
Let $\bm x\in\C^n$ be an eigenvector corresponding to $\lambda\neq 1$.
It follows from  $\bm J_{A_\alpha}(\pih)\bm x = \lambda \bm x$
and  the representation of $\bm J_{A_\alpha}(\pih)$ in~\eqref{eq:normspec2:fact1} that 
\[
\bm L\bm x = (1-\lambda)\bm M\bm x
\quad
\iff
\quad 
\bm L^\sym\bm z = (1-\lambda)\bm M^\star\bm z,
\]
where $\bm z \coloneqq   \bm D_\alpha(\pih)^{-1}\bm x$.
This implies 
\begin{align}
\lambda  = 1-\frac{1}{z}
\quad\text{with $z\coloneqq 
\frac{\bm z^{\mathrm H}\bm M^\star\bm z}
{\bm z^{\mathrm H}\bm L^\sym\bm z}$.}\label{lambdas}
\end{align}
Here, $\Re(z)>1/2$ due to \eqref{eq:pdm}. By elementary complex analysis, the M\"obius transform $z\mapsto 1-1/z$ maps the open half plane $\Re(z)> 1/2$ into the open unit disc. We hence proved $|\lambda|< 1$. 

\Cref{ii:lm:normspec2} holds in analogy with~\Cref{lm:normspec}~\ref{ii:lm:normspec}.
In particular, let $\lambda_1,\dots,\lambda_n$ be eigenvalues of $\bm J_{A_\alpha}(\pih)$, with corresponding 
linearly independent eigenvectors $\bm x_1,\dots,\bm x_n$ (including generalized eigenvectors for degenerate eigenvalues).
Assume $\lambda_1=1$ and $\bm x_1=\pih$. Then, by the fact that
$\bm P_{\pih^\perp}  \bm J_{A_\alpha}(\pih)  \bm P_{\pih^\perp} \bm x \equiv 
\bm P_{\pih^\perp}  \bm J_{A_\alpha}(\pih)  \bm x$ for all $\bm x\in\mathbb C^n$, one quickly verifies 
$\bm P_{\pih^\perp}  \bm J_{A_\alpha}(\pih)$ has eigenvalues $0,\lambda_2,\dots,\lambda_n$
with corresponding linearly independent eigenvectors (and generalized eigenvectors) 
$\bm y_1=\pih$ and $\bm y_i= \bm P_{\pih^\perp} \bm x_i$, for $i=2, \dots, n$. 
\end{proof}

According to~\Cref{lm:normspec2}, the local convergence factor of Newman's $\alpha$-scheme with asynchronous updates can be
expressed using the eigenvalues  $\bm J_{A_\alpha}(\pih)$  as follows.

\begin{theorem}\label{thm:rhoasync}
Assume that $\G([n], \W)$ is strongly connected. The local convergence factor of Newman's $\alpha$-scheme with asynchronous updates in~\eqref{eq:fpasyn}, for convergence to the MLE $\pih$, is given by 
\begin{equation}\label{eq:rhoasync}
\rhoasync (\alpha)\coloneqq\max \left\{|\lambda|\colon \lambda \in \Lambda(\bm J_{A_\alpha}(\pih)),\ |\lambda| < 1 \right\},
\end{equation}
where $\Lambda(\bm J_{A_\alpha}(\pih))$ denotes the set of eigenvalues
of the Jacobian~\eqref{crucial}.
\end{theorem}

In contrast to the synchronous case in~\Cref{thm:rhosync}, the matrix $\bm J_{A_\alpha}(\pih)$ may contain complex eigenvalues.
Nevertheless, $\rhoasync (\alpha)<1$ always holds, indicating the asynchronous update is always locally convergent with at least 
a linear rate. This is consistent with, and complements, the results in~\cite{newman2023efficient} on the global convergence of the algorithm.

\subsection{Variation of convergence factors}

In \cite{newman2023efficient}, it was empirically observed that the choice of $\alpha=0$ is often optimal for the asynchronous update.
The major goal of this section is to provide a theoretical justification for this observation. Following the approach in~\Cref{sec:gap1}, we focus on the population model with the expected outcome $\bar \W$ from~\eqref{eq:barw0} and defer the corresponding approximation results to \Cref{sec:app5}. To distinguish it, we denote the convergence factor $\rhoasync(\alpha)$ in~\eqref{eq:rhoasync} corresponding to  $\W=\bar \W$ as 
\begin{equation}\label{eq:brhoasync}
\brhoasync (\alpha)\coloneqq\max \left\{|\lambda|\colon \lambda \in \Lambda(\JJ_{A_\alpha}(\bm\pi^*)),\ |\lambda| < 1 \right\},
\end{equation}
where 
\begin{equation}\label{eq:jjabar}
\JJ_{A_\alpha}(\bm\pi^*) = (\bm I - \JJ_{F_\alpha, \l}(\bm\pi^*))^{-1}(\JJ_{F_\alpha, \d}(\bm\pi^*) + \JJ_{F_\alpha, \u}(\bm\pi^*))
\end{equation}
is the population version of the matrix $\bm J_{A_\alpha}(\pih)$ in~\eqref{crucial}, corresponding to $\W=\bar \W$
and $\pih=\bm \pi^*$. 
In the following discussion, we begin with a specialized case of bipartite graphs with a consistent ordering, where the monotonicity of the convergence rate can be established. For general graphs, where the analysis becomes more intricate, we characterize the region $\Omega(\alpha)$ containing the leading eigenvalues of $\JJ_{A_\alpha}(\bm\pi^*)$, and show that this region is monotonically converging to the single point $\Omega(\alpha)\to \{1\}$ as $\alpha$ increases, indicating that the convergence tends slower as $\alpha \to \infty$.

We first consider the case of bipartite graphs. Analyzing the dynamics of $\brhoasync(\alpha)$ in the general setting is challenging due to the presence of complex eigenvalues. A case where such analysis is feasible is when the comparison graph $\G([n], \M)$ is {\em consistently-ordered bipartite}. Bipartite or nearly bipartite comparison graphs frequently arise in practice, particularly in frameworks like item-response theory \citep{chen2025item}, where interactions occur exclusively between disjoint sets of subjects and items. This type of structure is also a classical assumption in the spectral analysis of matrix iterative methods, particularly those of the Gauss–Seidel type (see, e.g.,~\cite{saad2003iterative,young2014iterative}). In this context, a consistent ordering is a standard tool
---  it requires updates to propagate consistently along a particular ordering, which ensures that the iteration matrix has real eigenvalues in certain cases and facilitates convergence proofs. 

\begin{assumption}[Bipartite in consistent order~\cite{saad2003iterative,young2014iterative}]\label{ass:bipartite}
The comparison graph $\G([n], \M)$ is connected and bipartite, i.e., there exists a partition of $[n] = \mathcal A_1\cup \mathcal A_2$ such that $m_{ij} =0$, for all $i, j\in \mathcal A_1$ or $i, j\in \mathcal A_2$. Moreover, the vertices are ordered consistently with this partition, 
i.e., the first $|\mathcal A_1|$ indices correspond to  $\mathcal A_1$ and the remaining $|\mathcal A_2|$ indices correspond to  $\mathcal A_2$.
\end{assumption}

By Assumption~\ref{ass:bipartite}, the matrix representation of $\M$ can be partitioned into 
a block $2$-by-$2$ matrix
\begin{equation}\label{eq:blockm}
\bm M=\begin{bmatrix} \bm 0 & \bm M_{12}\\ \bm M_{12}^\top & \bm 0 \end{bmatrix}.
\end{equation}
Under this block structure, the eigenvalues $\lambda$ of the iteration matrix $\JJ_{A_\alpha}(\bm\pi^*)$
can be characterized by the eigenvalues $\mu$ of the associated quadratic eigenvalue problem (QEP):
\begin{align}\label{eq:qep}
\det(\bm Q_\alpha (\mu))=0
\quad\text{with}\quad 
\bm Q_\alpha (\mu) \coloneqq
\mu^2 \left( \bar{\bm C}_* +\alpha\bar{\bm R}_*\right)
-\mu\bar{\bm W}_* - \alpha\bar{\bm R}_*,
\end{align}
where the coefficient matrices 
$\bar{\bm C}_*\equiv  \bar{\bm C}(\bm\pi^*)$,
$\bar{\bm R}_*\equiv \bar{\bm R}(\bm\pi^*)$,
and 
$\bar{\bm W}_*=\bar{\bm W}(\bm \pi^*)$ are as defined in \eqref{myQbar} and \eqref{eq:barrc}.  
Moreover, the eigenvalues of the QEP consist of monotonic functions in $\alpha$.

\begin{lemma}\label{Ding's thm}
Let  Assumption~\ref{ass:bipartite} hold and $\alpha \geq 0$. 
\begin{enumerate}[label=(\alph*)]
\item\label{dd1} 
$\mu$ is an eigenvalue of  the QEP~\eqref{eq:qep}
if and only if 
$\lambda = \mu^2$ is an eigenvalue of $\JJ_{A_\alpha}(\bm\pi^*)$.
\item\label{dd2} 
The QEP~\eqref{eq:qep} has $2n$ real eigenvalues
that can be represented by $\{\pm 1\}\bigcup\{\pm \phi_i(\alpha)\}_{i=1}^{n-1}$,
where each  $0\leq \phi_i(\alpha)< 1$ is continuous and monotonically increasing in $\alpha\geq 0$.
\end{enumerate}
\end{lemma}

\begin{proof}
Let $\lambda$ be an eigenvalue of $\JJ_{A_\alpha}(\bm\pi^*)$,
i.e.,  $ \det(  \JJ_{A_\alpha}(\bm\pi^*) - \lambda \bm I) =0$.
By~\eqref{eq:jjabar}, the eigenvalue equation can be written equivalently as 
\begin{align}\label{eq:det0}
\det(   (\JJ_{F_\alpha, \d}(\bm\pi^*) + \JJ_{F_\alpha, \u}(\bm\pi^*)) - \lambda  (\bm I-\JJ_{F_\alpha, \l}(\bm\pi^*)))
=0.
\end{align}
Recall~\eqref{eq:jjfstar}  that 
$\JJ_{F_\alpha}(\bm\pi^*)  \equiv (\bar{\bm C}_* + \alpha \bar {\bm R}_*)^{-1} (\bar{\bm W}_* + \alpha \bar{\bm R}_*)$,
where  $\bar{\bm C}_*\equiv \bar{\bm C}(\bm\pi^*)$ and $\bar {\bm R}_*\equiv\bar {\bm R}(\bm \pi^*)$ are diagonal,
while $\bar{\bm W}_*\equiv \bar{\bm W}(\bm\pi^*)$ has zero diagonal. Multiplying  $\bar{\bm C}_*+\alpha \bar {\bm R}_*$ to the matrix in~\eqref{eq:det0}, we further rewrite the equation as 
\begin{align}\label{eq:det1}
\det\left(\left(\alpha \bar{\bm R}_* + \bar{\bm W}_{*\u}\right) - \lambda\left(\bar{\bm C}_*+\alpha \bar {\bm R}_* -  \bar{\bm W}_{*\l}\right)\right) = 0.
\end{align}
Observe that $\bar{\bm W}_* = \bar{\bm W}_{*\l}+\bar{\bm W}_{*\u}$, defined in~\eqref{myQ},
has the same block structure as in~\eqref{eq:blockm}.

For~\cref{dd1}, we first consider the case of a nonzero eigenvalue $\lambda\neq0$,
for which the block structure of $\bar{\bm W}_*$ allows us to balance the scalar factors associated with
$\bar{\bm W}_{*\l}$ and  $\bar{\bm W}_{*\u}$ in~\eqref{eq:det1}.
To this end, introduce a nonsingular balancing matrix
$\bm D_{\lambda}  = \mbox{blockdiag}(\bm I, \sqrt{\lambda} \bm I)$, partitioned consistent with the block structure of $\bar{\bm W}_*$ in~\eqref{eq:blockm}.
Multiplying the matrix in~\eqref{eq:det1},
from the left and right by  
$\bm D_\lambda^{-1}$ and $\bm D_\lambda$, 
respectively, leads to 
\begin{align}\label{eq:det2}
\det\left(\alpha \bar{\bm R}_* + \sqrt{\lambda}\left(\bar{\bm W}_{*\u}+\bar{\bm W}_{*\l}\right) 
- \lambda\left(\bar{\bm C}_*+\alpha \bar {\bm R}_* \right)\right) = 0.
\end{align}
Hence,  by letting $\mu=\sqrt{\lambda}$, we obtained the QEP~\eqref{eq:qep}. 

For the special case of a zero eigenvalue $\lambda=0$,
it follows from~\eqref{eq:det1} that this occurs if and only if $\alpha=0$,
which in turn holds if and only if  the QEP~\eqref{eq:qep} admits the zero eigenvalue $\mu=0$.

For \Cref{dd2},  we first observe that the eigenvalues of the QEP occur in pairs $\pm \mu$, since  
\[
\det({\bm Q}_\alpha(\mu)) = \det(\bm D  {\bm Q}_\alpha(\mu) \bm D) = \det({\bm Q}_\alpha(-\mu)),
\]
where we used $\bm D  {\bm W}_* \bm D \equiv -\bm W_*$ with $\bm D\coloneqq \mbox{blockdiag}(\bm I,-\bm I)$,
partitioned consistently with the block structure of $\bar{\bm W}_*$ in~\eqref{eq:blockm}.

Consider the case $\alpha \neq 0$. 
Using a standard linearization technique for QEPs 
(see, e.g.,~\cite{Higham:2002a,Tisseur:2001}),
we reformulate the QEP~\eqref{eq:qep} as 
the equivalent linear eigenvalue problem
\begin{equation}\label{eq:lep}
(\bm A_\alpha-\mu \bm B_\alpha)\bm z \coloneqq  
\left(
\begin{bmatrix}
\widetilde {\bm W}_* & \alpha \widetilde {\bm R}_* \\
\alpha \widetilde {\bm R}_*  & 0
\end{bmatrix} 
-\mu
\begin{bmatrix}
\widetilde{\bm C}_* +\alpha\widetilde {\bm R}_* & 0\\
0&  \alpha \widetilde {\bm R}_*  
\end{bmatrix} 
\right)
\begin{bmatrix} \mu \bm v \\ \bm v\end{bmatrix}
=0,
\end{equation}
where $\bm z\in\R^{2n}$ denotes the eigenvector,
and the coefficient matrices have been replaced by the equivalent diagonally scaled matrices
$[\widetilde{\bm C}_*, \widetilde {\bm R}_*, \widetilde {\bm W}_*] \coloneqq  \diag(\bm\pi^*) [\bar{\bm C}_*, \bar {\bm R}_*, \bar {\bm W}_*]$.
The diagonal scaling ensures $\bm A_\alpha$ is symmetric, recalling that $\diag(\bm \pi_*)\bm W_*$ is symmetric by the definition of $\bar{\bm W}(\bm\pi^*)$ in~\eqref{myQ}. Consequently, \eqref{eq:lep} is a symmetric definite eigenvalue problem.

By classical spectral theory~\citep{Kato:2013,Rellich:1969}, the eigenvalues of the definite eigenvalue problem~\eqref{eq:lep}
are all real and can be represented as $2n$ analytical functions of the parameter $\alpha >0$.
Due to the spectral symmetry $\pm\mu$, together with the bound $\lambda\equiv \mu^2\leq 1$ 
(with strict $<$ except for one $\lambda=1$) by~\Cref{dd1} and~\Cref{lm:normspec2}~\ref{i:lm:normspec2},
we may write those eigenvalues as 
$\{\pm 1\}\bigcup \{\pm\phi_i(\alpha)\}_{i=1}^{n-1}$, where each $0\leq \phi_i(\alpha)< 1$ is analytical in $\alpha>0$.
Moreover, the derivative of each eigenvalue function admits a standard expression
\[
\phi_i'(\alpha) = 
\frac{\bm z_i^\top  \left(\bm A'_\alpha-\phi_i(\alpha) \bm  B'_\alpha\right)\bm z_i}
{\bm z_i^\top \bm B_\alpha \bm z_i},
\]
where $\bm A'_\alpha$ and $\bm  B'_\alpha$ denote the derivatives with respect to $\alpha$,
and  $\bm z_i$ is the eigenvector corresponding to the eigenvalue $\phi_i(\alpha)$.
Using the block structure of $\bm A'_\alpha$, $\bm B'_\alpha$, and  $\bm z_i = [\phi_i(\alpha) \bm v_i^\top, \bm v_i^\top]^\top$ 
from~\eqref{eq:lep}, we obtain  immediately
\[
\phi_i'(\alpha) = 
\phi_i(\alpha) (1-\phi_i(\alpha)^2) 
\frac{\bm v_i^\top \widetilde {\bm R}_* \bm v_i}
{\bm z_i^\top \bm B_\alpha \bm z_i}
\geq 0,
\]
where the second step follows from $0\leq \phi_i(\alpha) < 1$, together with the positive definite $\widetilde {\bm R}_*$ and $\bm B_\alpha$. Therefore, each $\phi_i(\alpha)$ is increasing for $\alpha>0$.

Finally, this monotonicity of $\phi_i(\alpha)$ extends naturally to 
$\alpha\geq 0$, including the boundary case $\alpha=0$,
since the eigenvalues of the QEP~\eqref{eq:qep}
(i.e.,  the roots of the polynomial $\det(\bm Q_\alpha(\mu))$)
depend continuously on the coefficient $\alpha$.
\end{proof}

It follows immediately from~\Cref{Ding's thm} that the local convergence factor $\brhoasync(\alpha)$~\eqref{eq:brhoasync}
admits the following characterization under a consistently ordered bipartite structure.

\begin{theorem}\label{thm:coo}
Let Assumption~\ref{ass:bipartite} hold.
\begin{enumerate}[label=(\alph*)]
\item\label{i:thm:coo} 
The convergence factor $\brhoasync(\alpha)<1$  is monotonically increasing in $\alpha\geq 0$, 
\item\label{ii:thm:coo} 
At $\alpha=0$,
the convergence factor $\brhoasync(\alpha)$ reduces to 
\begin{align}
\brhoasync (0)&\equiv \max \left\{|\lambda|^2\colon \lambda \in \Lambda(\bar{\bm J}_{F_0}(\bm\pi^*)),\ |\lambda| < 1 \right\}
&<1,\label{eq:rasync0}\\
\intertext{in contrast to the convergence factor of the synchronized update,}
\brhosync (0)&\equiv \max \left\{|\lambda|~\colon \lambda \in \Lambda(\bar{\bm J}_{F_0}(\bm\pi^*)),\ \lambda < 1 \right\}
&\equiv 1,\label{eq:sync0}
\end{align}
where $\bar{\bm J}_{F_0}(\bm\pi^*)$ denotes the Jacobian~\eqref{eq:jjfstar} with $\alpha=0$.
Consequently,  for all  sufficiently small $\alpha\geq 0$,
we have  $\brhoasync (\alpha) <  \brhosync (\alpha)$ 
by the continuity with respect to $\alpha$.
\end{enumerate}
\end{theorem}
\begin{proof}
\Cref{i:thm:coo} follows directly from  \eqref{eq:brhoasync} and~\Cref{Ding's thm},
which implies 
\[
\brhoasync(\alpha)\equiv \max_{1\leq i\leq n-1}\phi_i(\alpha)^2,
\]
being the maximum of increasing functions $\phi_i(\alpha)^2$,
is itself monotonically increasing in $\alpha$.

For~\Cref{ii:thm:coo}, we derive from~\eqref{eq:qep} with $\alpha=0$ that 
\[
\det(\bm Q_0(\mu)) = \det(\mu^2   \bar{\bm C}_*  -\mu\bar{\bm W}_*)
= |\mu|^m\det (\mu \bar{\bm C}_*  -\bar{\bm W}_*).
\]
Consequently, the nonzero eigenvalues of $\bm Q_0(\mu)$  are given by 
$\det(\mu \bar{\bm C}_*  -\bar{\bm W}_*)=0$, which correspond to the eigenvalues 
of the matrix $ \bar{\bm C}_*^{-1} \bar{\bm W}_* \equiv  \JJ_{F_0}(\bm\pi^*)$ 
from~\eqref{eq:jjfstar} with $\alpha=0$.
Hence, result~\eqref{eq:rasync0} follows from~\eqref{eq:brhoasync} and~\Cref{Ding's thm}. 
On the other hand, result~\eqref{eq:sync0}
follows from~\eqref{eq:brhosync} upon noticing that 
$\pm1$ are eigenvalues of $\JJ_{F_0}(\bm\pi^*)$, 
since they are eigenvalues of $\bm Q_0(\mu)$ by~\Cref{Ding's thm}.
\end{proof}
\begin{remark}
The assumption of consistent ordering is essential for establishing the monotonicity of the convergence factor. For bipartite graphs with a general non-consistent ordering, the convergence factor need not be monotonic in $\alpha$; see \Cref{sec:num1}.
\end{remark}

We next consider the case of a general comparison graph structure. In this case, the eigenvalues of the iteration matrices $\JJ_{A_\alpha}(\bm\pi^*)$ may be complex, making the analysis of the convergence factor~\eqref{eq:brhoasync} more intriguing. 
Nevertheless, the structure of the expected outcome $\bar\W$ allows us to further refine \Cref{lm:normspec2}
to characterize the possible location of the eigenvalues of $\JJ_{A_\alpha}(\bm\pi^*)$. This characterization provides insight into how the eigenvalues vary with the parameter $\alpha$ and partially explains why increasing $\alpha$ tends to slow the convergence.

\begin{proposition}\label{lm:normspec3}
Assume that $\G([n], \M)$ is connected. Then $\JJ_{A_\alpha}(\bm\pi^*)$ has a simple eigenvalue $\lambda =1$ with the remaining eigenvalues lying in the following disk punctured at $1$:
\begin{align}
\Omega(\alpha)\coloneqq \left\{z\in\C: |z-c_\alpha|\leq r_\alpha \right\} \setminus \{1\}, \label{Omega}
\end{align}
with  center $c_\alpha \coloneqq  (1-r_\alpha)$ and radius $r_\alpha\coloneqq (1- \lambda_\alpha )/(2-\lambda_\alpha)\in (0,2/3]$,
where  $\lambda_\alpha\coloneqq \lambda_n(\JJ_{F_\alpha}(\bm\pi^*))\in[-1,1)$ denotes the smallest eigenvalue of $\JJ_{F_\alpha}(\bm\pi^*)$. 
\end{proposition}

\begin{proof}
The proof refines the arguments of~\Cref{lm:normspec2} for the expected outcome $\bar{\W}$. 
Since $\M$ is connected, $\bar{\W}$ is strongly connected,
and~\Cref{lm:normspec2} implies that $\JJ_{A_\alpha}(\bm\pi^*)$ has a
simple eigenvalue $\lambda=1$.
Moreover,  we can improve the bound 
for the diagonal entries 
$[\bm J_{F_\alpha}(\pih)]_{ii}>-1$ to 
$[\JJ_{F_\alpha}(\bm\pi^*)]_{ii}\geq 0$
according to the definition~\eqref{Jbar}, and refine~\eqref{eq:pdm}
to 
\begin{equation}
2\mbox{Re}(\bm v^{\mathrm{H}} \bar{\bm M}^\star \bm v) = \bm v^{\mathrm{H}} (2\bm I+ \bar{\bm J}^{\sym}_{F_\alpha,\d}(\bm\pi^*)
- \bar{\bm J}^{\sym}_{F_\alpha}(\bm\pi^*)) \bm v\geq \|\bm v\|_2^2 + \bm
v^{\mathrm H}\bar{\bm L}^\sym\bm v,
\end{equation}
where to differentiate, we use the `bar' to denote the coefficient matrices 
corresponding to $\bar{\W}$.
Therefore, 
for the ratio $z:= \frac{{\bm z}^{\mathrm H}\bar{\bm M}^\star{\bm z}}{{\bm
z}^{\mathrm H}{\bar{\bm L}}^\sym{\bm z}}$ in~\eqref{lambdas},
we can improve its bound $\Re(z)\geq 2$ to 
\begin{align*}
\Re(z)= \Re\left( \frac{\bar{\bm z}^{\mathrm H}\bar{\bm M}^\star\bar{\bm
z}}{\bar{\bm z}^{\mathrm H}\bar{\bm L}^\sym\bar{\bm z}}\right)
\geq\frac{1}{2}\left(1+\frac{\|\bar{\bm z}\|_2^2}
{\bar{\bm z}^{\mathrm H}\bar{\bm L}^\sym\bar{\bm z}}\right)\geq\frac{1}{2}\left(1+\frac{1}{\lambda_1(\bar{\bm L}^\sym)}\right)
= \frac{1}{2r_\alpha},
\end{align*}
where in the last equation we used $\lambda_1(\bar{\bm L}^\sym)=1- \lambda_n(\JJ_{F_\alpha}(\bm\pi^*))$ by~\eqref{eq:normspec2:fact2}
and the definition of $r_\alpha$ in~\eqref{Omega}. Finally, by the M\"obius transform $z\mapsto \lambda\equiv 1-1/z$
from~\eqref{lambdas}, this half plane $\Re(z)\geq 1/(2r_\alpha)$ for $z\in\mathbb C$ is mapped to the disc in $\Omega(\alpha)$~\eqref{Omega} for all eigenvalues satisfying $\lambda\neq 1$.
\end{proof}

Observe that the condition $r_\alpha\leq 2/3$ implies $\Re(z)\geq -1/3$ for all $z\in\Omega(\alpha)$ in~\eqref{Omega}.
Therefore, the real parts of all eigenvalues of $\JJ_{A_\alpha}(\bm\pi^*)$ are bounded from below by $-1/3$. 
Moreover, as $\alpha$ increases,~\Cref{thm:comp} indicates $\lambda_n(\JJ_{F_\alpha}(\bm\pi^*))\to 1$ monotonically.
Consequently,  $r_\alpha \to 0$ and $c_\alpha \to 1$, and hence the region $\Omega(\alpha)$ shrinks monotonically towards the single point $\{1\}$. 
See~\Cref{fig:omegaa} for an illustration of $\Omega(\alpha)$ for different values of $\alpha$.

\begin{figure}[htbp]
\centering
\includegraphics[width=0.4\textwidth, trim={0.2cm 1cm 3cm 2cm}, clip]{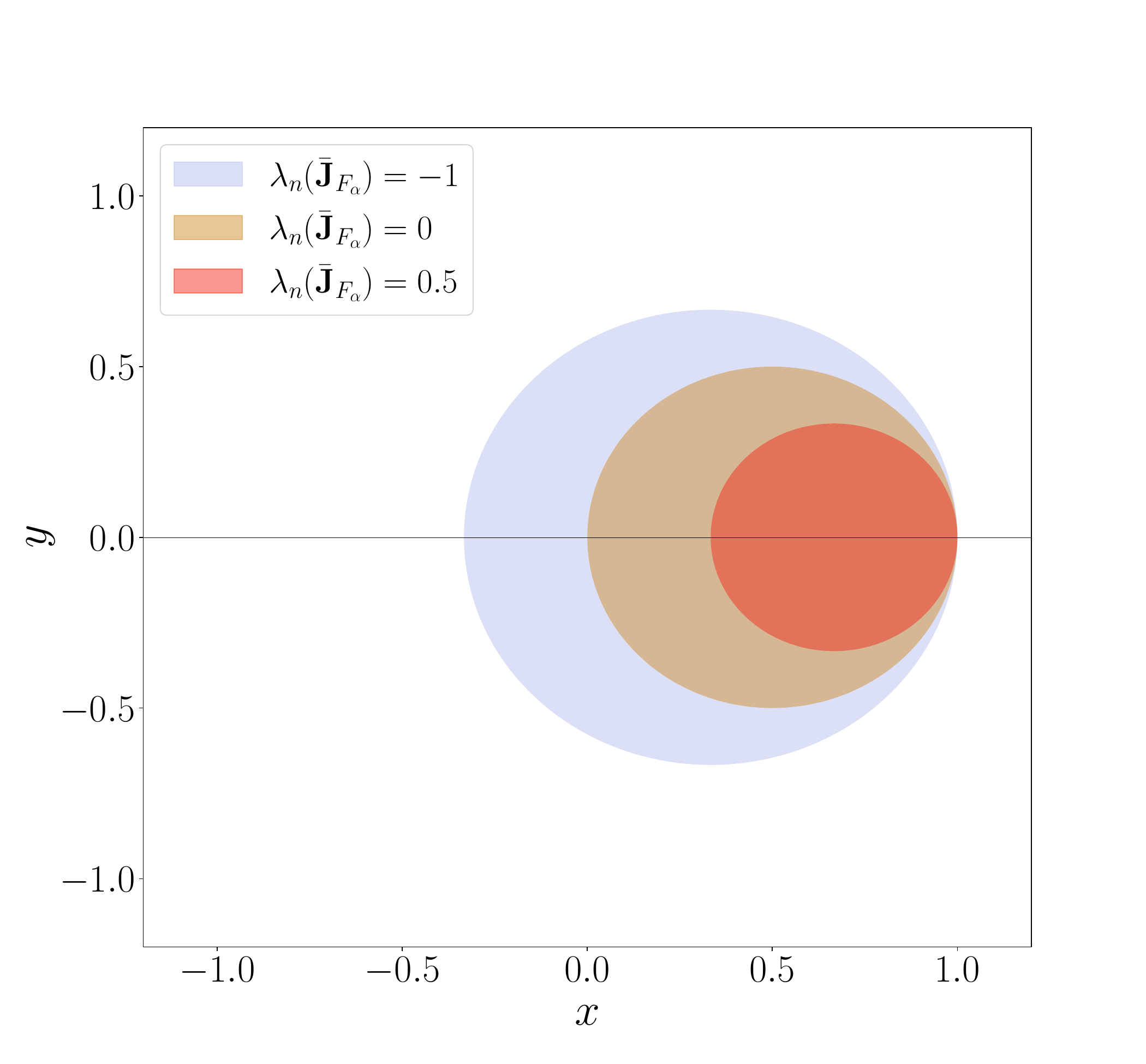}
\caption{Visualization of $\Omega(\alpha)$ for three increasing values
of $\alpha$ with the corresponding $\lambda_n(\JJ_{F_\alpha}(\bm\pi^*))=-1, 0, 0.5$.}
\label{fig:omegaa}
\end{figure}

\section{Asymptotic approximation}\label{sec:app5}
Several preceding convergence results are established for expected outcomes under the BT model. This idealization simplifies the theoretical analysis and enables a rigorous characterization of the local convergence behavior of Newman's $\alpha$-scheme.  
The strength of these conclusions relies on the premise that the local convergence factors $\rho_{\mathrm{sync}}$ and $\rho_{\mathrm{async}}$ are close to their counterparts $\brhosync$ and $\brhoasync$ in the expected model, respectively. 
In this section, we establish such approximation results in both the conventional fixed-$n$ regime and the large-$n$ regime. 

\subsection{Fixed-$n$ regime}

Recall the BT model with outcomes $\{w_{ij}\}$ subject to the binary distribution in~\eqref{bt-model}.
One may expect the observed outcomes $w_{ij}$ to approximate their expectations $\bar w_{ij}$ in \eqref{eq:barw0}
as the numbers of comparisons $m_{ij}$  become sufficiently large,
and consequently,  the convergence analysis established for $\bar w_{ij}$ in~\Cref{sec:ana} 
approximates that for the observed $w_{ij}$. 

To formalize such approximations, we consider comparison graphs $\G([n], \M)$ in the {\em fixed-$n$ regime},
where the number $n$ of the objects remains fixed and the minimum number of pairwise comparisons, measured by
\begin{align}\label{myL}
L \coloneqq \min_{1\leq i,j\leq n, \atop m_{ij}\neq 0} m_{ij},
\end{align}
is sufficiently large. The definition of $L$ also accommodates prescribed structures of the comparison graph, where certain pairs $i$, $j$ may satisfy  $m_{ij}\equiv 0$ and hence are not compared.

In this fixed-$n$ regime, the following consistency result holds for the MLE~\eqref{mle0}. The proof is deferred to~\Cref{app:001}, which is based on a chaining argument and extends an existing result for general logistic regression \citep{fahrmeir1985consistency} that requires additional balancing conditions.

\begin{theorem}[Consistency of the MLE in the fixed-$n$ regime]\label{thm:asymp1}
Consider connected comparison graphs  $\G([n], \M)$ with a fixed $n$
and associated  $L$ defined in \eqref{myL}.
As $L\to\infty$, the MLE $\pih$ in~\eqref{mle0}, based on  the observed $\W$ in~\eqref{bt-model},
exists  uniquely  with probability tending to one, and 
$\pih - \bm\pi^*  \xrightarrow{\P} \bm 0$.  
\end{theorem}
\begin{proof}
See~\Cref{app:001}.
\end{proof}

Regarding the synchronous update, we can establish the asymptotic approximation of the Jacobian $\bm J_{F_\alpha}(\pih)$ in~\eqref{jac:Jfp} by its population counterpart $\JJ_{F_\alpha}(\bm\pi^*)$ in~\eqref{eq:jjfstar}, corresponding respectively to the observed $\W$ and expected $\bar\W$ outcomes. The following result can be established by observing that 
$\pih - \bm\pi^* \xrightarrow{\P}\bm 0$ from~\Cref{thm:asymp1}, 
$\bm W -\bar{\bm W} \xrightarrow{\P} 0$, and the continuity of the matrix $\bm J_{F_\alpha}(\pih)$ with respect to $\pih$ and $\W$; a detailed proof is deferred to~\Cref{app:002}.

\begin{theorem}\label{thm:asymp}
Following~\Cref{thm:asymp1},  as $L\to\infty$, we have
\begin{enumerate}[label=(\alph*)]
\item\label{0021} $\bm{J}_{F_\alpha}(\pih) - \JJ_{F_\alpha}(\bm\pi^*) \xrightarrow{\P} \bm 0$, and
\item\label{0022} 
$\rhosync(\alpha)-\brhosync(\alpha)\xrightarrow{\P} 0$
for convergence factors of the synchronous updates in~\eqref{eq:rhosync}~and~\eqref{eq:brhosync}.
\end{enumerate}
\end{theorem}

\begin{proof}
See \Cref{app:002}. 
\end{proof}

Similarly, for the asynchronous update, we can establish the asymptotic approximation of 
the corresponding Jacobian $\bm J_{A_\alpha}(\pih)$ in~\eqref{crucial} by its population counterpart 
$\JJ_{A_\alpha}(\bm\pi^*)$ in~\eqref{eq:jjabar}
as follows. 

\begin{theorem}\label{lm:gs}
Following~\Cref{thm:asymp1},  as $L\to\infty$, we have
\begin{enumerate}[label=(\alph*)]
\item $\bm{J}_{A_\alpha}(\pih) - \JJ_{A_\alpha}(\bm\pi^*) \xrightarrow{\P} \bm 0$, and
\item  $\rhoasync(\alpha)-\brhoasync(\alpha)\xrightarrow{\P} 0$ 
for convergence factors of the asynchronous updates  in~\eqref{eq:rhoasync}~and~\eqref{eq:brhoasync}.
\end{enumerate}
\end{theorem}

\begin{proof}
See \Cref{app:0066}. 
\end{proof}

\subsection{Large-$n$ regime}\label{sec:div}

In modern data science applications, comparison graphs are typically vast and sparse~\citep{fang2026recent} and exhibit particular structures, such as expander, clustered, or bipartite graphs~\citep{goldenberg2010survey}. It is therefore necessary to consider the regime $n\to\infty$ while also accounting for these structures. With this in mind, in this section we consider a large-$n$ regime in which the comparison graph $\G([n], \M)$ is generated from the \textit{stochastic block model} (SBM)~\citep{holland1983stochastic}, which is flexible enough to capture certain graph structures while remaining tractable for rigorous analysis; see, e.g.,~\cite{abbe2018community, han2025unified}.

In our study, we focus on a two-community SBM with equal-sized groups, denoted by $\sbm(n, p, q, L)$, in which within-community and cross-community edges of fixed weight $L$ are generated independently with 
probabilities $p$ and $q$, respectively. That is, for $\G([n], \M) \sim  \sbm(n, p, q, L)$, the number of comparisons between objects $i,j$ is 
\begin{equation}\label{eq:sbm}
\text{$m_{ij} = L$ with probability $r$, and $0$ otherwise},
\end{equation}
where $r= p$, if $i$ and $j$ belong to the same community, and $r=q$ otherwise. 
Special cases of $\sbm(n, p, q, L)$  include the Erd\H{o}s--R\'enyi model obtained when $p=q$, which is commonly used in the analysis of the BT model in the large-$n$ regime \citep{Negahban2012RankCR, agarwal2018accelerated, chen2019spectral, han2020asymptotic, chen2022optimal, gao2023uncertainty, chen2022partial}. 
More generally, $p\gg q$ gives rise to cluster structures, and $q \gg p$ to near-bipartite structures.

For the BT model~\eqref{bt-model} under $\G([n], \mathcal{M}) \sim \sbm(n, p, q, L)$, the following uniform consistency for the MLE can be established. This result essentially follows from Theorem 4 and Proposition 2 of~\cite{han2023general},
where the special case $L=1$ was considered, and the extension to general $L$ is straightforward.
We omit the proof and refer to~\cite{han2023general} for details.

\begin{theorem}[Consistency of the MLE in the large-$n$ regime]\label{thm:asymp2}
Consider the BT model~\eqref{bt-model} with $\G([n], \M) \sim  \sbm(n, p, q, L)$ and true strength vectors $\bm\pi^*\in\R^n_{+}$. As $n$ varies, we assume the corresponding $p, q$ satisfy $\min\{p, q\}\cdot n/(\log n)^3\to\infty$ as $n\to\infty$, and each $\bm\pi^*$ has a dynamic range $\kappa_n:=\max_{i,j} \pi^*_i/\pi_j^*\leq \kappa$, for some constant $\kappa$ independent of $n$. Then, for all sufficiently large $n$, the MLE $\pih$ in~\eqref{mle0} uniquely exists with probability at least $1-n^{-2}$, 
and it satisfies 
\begin{align*}
\|\pih - \bm\pi^*\|_\infty\leq c(\kappa)\sqrt{\frac{(\log n)^3}{np_{\min}L}},  
\end{align*}
where $p_{\min}=\min\{p, q\}$ and $c(\kappa)>0$ is a constant that depends only on $\kappa$ 
(and is independent of $n$, $p$, $q$, and $L$).
\end{theorem}

We can further establish the following approximation of the convergence factor for the synchronous update.
This result serves as a large-$n$ analogue of~\Cref{thm:asymp}, and its proof follows from a refinement of the perturbation argument used in~\Cref{thm:asymp} together with~\Cref{thm:asymp2}. 

\begin{theorem}\label{thm:long}
Following~\Cref{thm:asymp2}, there exists a constant $C(\kappa)>0$ depending only on $\kappa$ (and independent of $n,p, q, \alpha$, and $L$),  such that the following hold with probability at least $1-4n^{-2}$  for all sufficiently large $n$: 
\begin{enumerate}[label=(\alph*)]
\item\label{0023} $\|\bm J_{F_\alpha}(\pih)-\JJ_{F_\alpha}(\bm\pi^*)\|_2 \le C(\kappa)\sqrt{\frac{(\log n)^3}{np_{\min}L}}$
for the Jacobian in~\eqref{jac:Jfp} and~\eqref{Jbar}; and
\item\label{0024} $|\rhosync(\alpha)-\brhosync(\alpha)| \le C(\kappa)\sqrt{\frac{(\log n)^3}{np_{\min}L}}$ for the convergence factors in~\eqref{eq:rhosync}~and~\eqref{eq:brhosync}. 
\end{enumerate}
\end{theorem}

\begin{proof}
See~\Cref{sec:proof}. 
\end{proof}

\begin{remark}
\Cref{thm:long} shows that for any $\delta>0$ and a large $n$ with $p_{\min}\geq (\log n)^{3+\delta}/n$ 
(i.e., an object is expected to be compared with at least $\mathrm{poly}(\log n)$ distinct objects on average in a network involving $n$ objects), the convergence factor $\rhosync(\alpha)$ converges to $\brhosync(\alpha)$, for all $\alpha\geq 0$ and $L\geq 1$. This stands in contrast to the convergence result in \Cref{thm:asymp}, where it requires the number of pairwise comparisons to satisfy $L\to\infty$.
\end{remark}

For asynchronous updates, we can also establish the approximation of the Jacobian $\bm J_{A_\alpha}(\pih)$ by $\JJ_{A_\alpha}(\bm\pi^*)$
in the large-$n$ regime, in analogy to~\Cref{lm:gs} for the fixed-$n$ regime.

\begin{theorem}\label{thm:long1}
Following~\Cref{thm:asymp2}, there exists a constant $C(\kappa)>0$ depending only on $\kappa$ (and independent of $n,p,q, \alpha$, and $L$), such that  for sufficiently large $n$, it holds with probability at least $1-4n^{-2}$that
\begin{equation*}
\|\bm J_{A_\alpha}(\pih)-\JJ_{A_\alpha}(\bm\pi^*)\|_2 \le C(\kappa)\sqrt{\frac{(\log n)^3}{np_{\min}L}}. 
\end{equation*}
\end{theorem}

We should mention whether~\Cref{thm:long1} yields a rigorous bound for the corresponding convergence factor difference
$|\rhoasync(\alpha)-\brhoasync(\alpha)|$ in the general setting remains open. One may still apply the classical eigenvalue perturbation bounds (e.g.,~\cite[Section IV.1]{Stewart:1990}) as in the proof of \Cref{lm:gs} to obtain a worst-case bound on $|\rhoasync(\alpha)-\brhoasync(\alpha)|$. However, due to the presence of $n$ in the bound, this would require $L$ to be extremely large to guarantee convergence to zero, which essentially reduces the problem to the fixed-$n$ setting.

In practice, however, we observe that even a relatively small $L$ leads to a good approximation of the convergence factors (see \Cref{sec:nums}). 
This can be partially explained by the empirical observation that the eigenvalue condition number  
for the eigenvalue $\lambda$ corresponding to $\rhoasync(\alpha)$ of the matrix $\JJ_{A_\alpha}(\bm\pi^*)$ 
(i.e., $\|\bm v\|_2\|\bm u\|_2/|\bm v^\top \bm u|$ with $\bm u, \bm v\in\C^n$ the left and right eigenvectors associated with $\lambda$),
remains asymptotically constant in the matrix size $n$ under our SBM setting. Thus, assuming constant eigenvalue condition numbers, a small perturbation in the matrix $\JJ_{A_\alpha}(\bm\pi^*)$ from  $\bm J_{A_\alpha}(\pih)$, as $n\to\infty$ according to~\Cref{thm:long1}, leads to a correspondingly small perturbation in $\brhoasync(\alpha)$ from $\rhoasync(\alpha)$. A rigorous analysis of this phenomenon is, however, beyond the scope of the present work.

\section{Numerical experiments}\label{sec:nums}
In this section, we present numerical examples to verify our theoretical analysis and illustrate its practical implications.
We consider both synthetic test data generated from BT models with representative graph structures and real-world datasets from applications.
We will show that the convergence factor $\bar{\rho}$ for the expected outcomes often effectively captures the actual convergence factor $\rho$ 
of Newman's $\alpha$-scheme under the BT model. Moreover, the synchronous and asynchronous updates may exhibit
fundamentally different behavior as $\alpha$ varies, and the relative advantage of asynchronous updates generally becomes more
significant as $\alpha\to 0$. 

Unless otherwise specified, the {\em exact} MLE $\pih$ used in the experiments is computed using Newman's algorithm with asynchronous updates, terminated when the iterates $\{\bm\pi^{(k)}\}$ satisfy $\|\bm\pi^{(k)}-\bm\pi^{(k-1)}\|_2/\|\bm\pi^{(k)}\|_2\leq 10^{-15}$. 
All experiments are carried out in Python on a MacBook Air (M4, 2025) with 10 CPU cores. The source code to replicate the simulation results is available at \url{https://github.com/Yiminithere/newman-algorithm-for-BT}.

\subsection{Synthetic data}\label{sec:num1}

\begin{myexample}\label{ex:synth}\rm 

This example illustrates the local convergence factors of the algorithms and their dependence on the parameter
$\alpha$ using test data generated from the BT model~\eqref{bt-model}.
The test problems are set up as follows.
First, to account for different graph structures, we generate test comparison graphs 
using the two-community SBM described in \Cref{sec:div}:
\[
\G([n],\M) \sim \sbm(n, p, q, L).
\]
Note that different choices of $p$ and $q$ lead to varying connectivity patterns in the resulting graphs,
and we consider three representative cases corresponding to different topological structures:
\begin{enumerate}[label=(\alph*),itemsep=-2pt]
\item  $ p = q = 0.05$ ~~~~~~~~(homogeneous graphs); 
\item $ p = 0.1$, $q = 0.01$ ~~(clustered graphs);
\item $ p = 0.01$, $q = 0.1$ ~~(near-bipartite graphs). 
\end{enumerate}
For each comparison graph  $\G([n],\M)$, we sample the outcome matrix $\W$ according to 
the BT model~\eqref{bt-model}, using a prescribed strength vector $\bm\pi^*\in\R_+^n$.
The test strengths are generated from a lognormal distribution,
\[
\bm\pi^* \sim  \text{Lognormal}(0, \sigma^2),
\]
followed by normalization to a unit product. 
The lognormal distribution mimics practical scenarios in which most objects have strengths 
concentrated within a moderate range. 
Note that $\sigma$ controls the spread of each strength $\pi^*_{i}$, 
and hence the size of the expected dynamic range $\kappa=\max_{i,j} \pi^*_{i}/\pi^*_{j}$.
We test two values of $\sigma$, $\sigma = 0.1$ and $\sigma = 0.5$, 
corresponding to small and large dynamic ranges, with 
expected dynamic ranges $\kappa\approx 1.90$ and $24.69$, respectively.
All test problems are generated with a fixed size $n = 500$,
and two different numbers of pairwise comparisons $L = 10$ and $L = 100$.

\Cref{fig:1} reports the convergence factors of the algorithms against the parameter $\alpha\in[0,1]$ for the 
$12$ sampled test problems. 
For each comparison graph $\G([n],\M)$ and strength vector $\bm\pi^*$, 
we consider both the local convergence factors $\rho$ for a particular outcome $\W$ generated from the BT model~\eqref{bt-model},
and their population counterparts $\bar\rho$ corresponding to the averaged outcome $\bar\W$,
i.e., $\brhosync$ in~\eqref{eq:brhosync} and  $\brhoasync$ in \eqref{eq:brhoasync}.
For more direct comparison, \Cref{fig:112} reports the convergence histories for the boundary cases 
$\alpha=1$ (Zermelo's algorithm) and $\alpha=0$ (Newman's algorithm),
with the corresponding convergence factors listed in~\Cref{tab:results}.
We make the following observations.

\smallskip

\noindent{\em (a) Consistency of $\brhosync$}:
As shown in \Cref{fig:1}, the expected convergence factors $\bar{\rho}$
closely track the observed convergence factors
computed from comparison outcomes sampled from the BT model.
This phenomenon is consistent with the asymptotic analysis in~\Cref{sec:app5},
but it is nevertheless remarkable given the relatively modest problem size
and number of pairwise comparisons considered in our experiments
($n=500$ and $L=10,100$).
In particular, \Cref{tab:results} shows that, for the test cases with $\alpha=0,1$,
the expected and observed convergence factors agree to three decimal
places in most instances.
These results justify the use of $\bar{\rho}$ to predict and characterize 
the convergence behavior of Newman's $\alpha$-scheme 
in practice when the comparison data fit the BT model reasonably well.

\smallskip
\noindent{\em (b) Synchronous updates}:  We observe in~\Cref{fig:1} that the convergence factors $\rhosync(\alpha)$ 
are monotonically increasing in $\alpha$ for the homogeneous and clustered graphs, 
whereas they are quasi-convex (first decreasing and then increasing) for the near-bipartite graphs.
These observations are consistent with~\Cref{cor:optalpha}. In particular, all tested homogeneous and cluster graphs satisfy $s_0\geq 0$, indicating the monotonicity of $\rhosync(\alpha)$, and in contrast, the near bipartite graphs have $s_0<0$, 
indicating a `v' shaped curve for the convergence factors. Moreover, for the monotonic cases (homogeneous and cluster graphs, where $s_0\geq 0$), we also observe that a larger dynamic range (i.e., $\sigma=0.5$ from $0.1$) tends to yield a more significant speedup 
as $\alpha$ decreases from $\alpha=1$ to $0$.
This trend is more clearly illustrated in~\Cref{tab:results} and~\Cref{fig:112}
through comparisons between the boundary cases $\alpha=1$ (Zermelo's algorithm)
and  $\alpha=0$ (Newman's algorithm).
In particular, for $\sigma=0.1$ (smaller dynamic range), the convergence gaps $1-\brhosync$ for $\alpha=0$ are approximately twice as large as
those for $\alpha=1$ 
(e.g., $0.625/0.309\approx 2.02$ for the homogeneous graph with $\sigma=0.1$).
While for $\sigma=0.5$ (larger dynamic range), this ratio can become even larger
(e.g., $0.624/0.146\approx 4.27$ for the homogeneous graph with $\sigma=0.5$).
These observations are consistent with the predictions from the upper bounds on $1-\brhosync$ 
in Remark~\ref{rem:inter}. 

\smallskip

\noindent{\em (c) Asynchronous updates}:
\Cref{fig:1} shows that the convergence factor $\rhoasync(\alpha)$ increases monotonically with $\alpha$ across all testing cases,
including the near-bipartite graphs for which the synchronous updates lose monotonicity. Moreover, $\rhoasync(\alpha) \leq \rhosync(\alpha)$ in all test cases, indicating the asynchronous update always improves over its synchronous counterpart. Interestingly, the magnitude of this improvement depends strongly on the choice of $\alpha$. While the gain is seemingly negligible as $\alpha\to 1$ (Zermelo's algorithm), it becomes increasingly significant as $\alpha\to 0$ (Newman's algorithm); see also~\Cref{tab:results}. In particular, for the near-bipartite graphs with $\alpha=0$,
we observe $\rhosync(\alpha)\approx 1$, whereas $\rhoasync(\alpha)$ is substantially smaller (by roughly a factor of $7$), a behavior partially explained by our analysis for the bipartite case in~\Cref{thm:coo}. Taken together, these observations suggest that asynchronous updates are crucial for the practical effectiveness of Newman's algorithm ($\alpha=0$) in many applications. In contrast, when implemented synchronously, the algorithm may offer only limited acceleration over the classical Zermelo's algorithm ($\alpha=1$).

\begin{figure}[htbp]
\begin{subfigure}{0.3\textwidth}{\includegraphics[width=\linewidth, trim={0.1cm 0cm 0.2cm 0.2cm},clip]{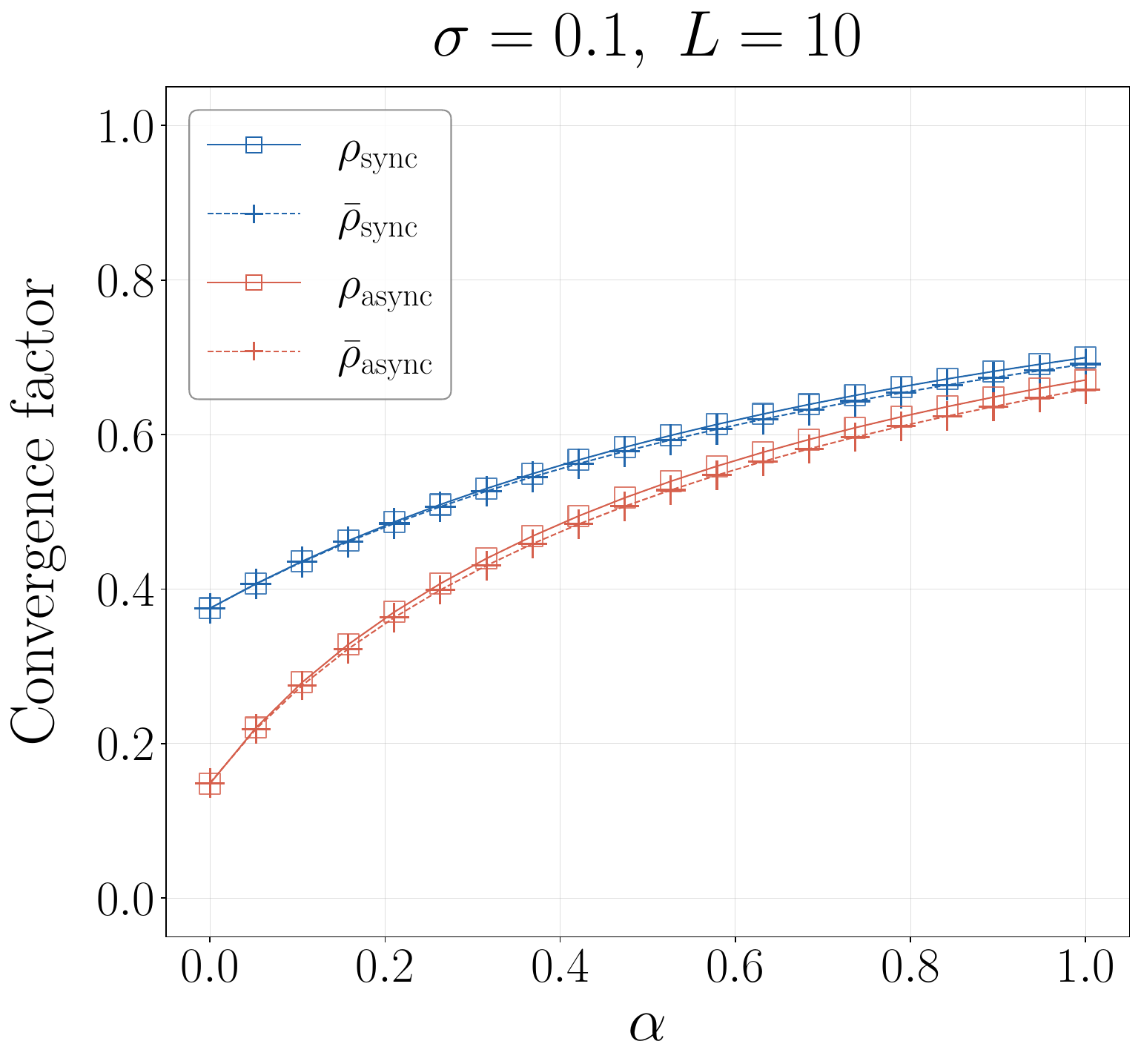}}
\end{subfigure}\hfill
\begin{subfigure}{0.3\textwidth}{\includegraphics[width=\linewidth, trim={0.1cm 0cm 0.2cm 0.2cm},clip]{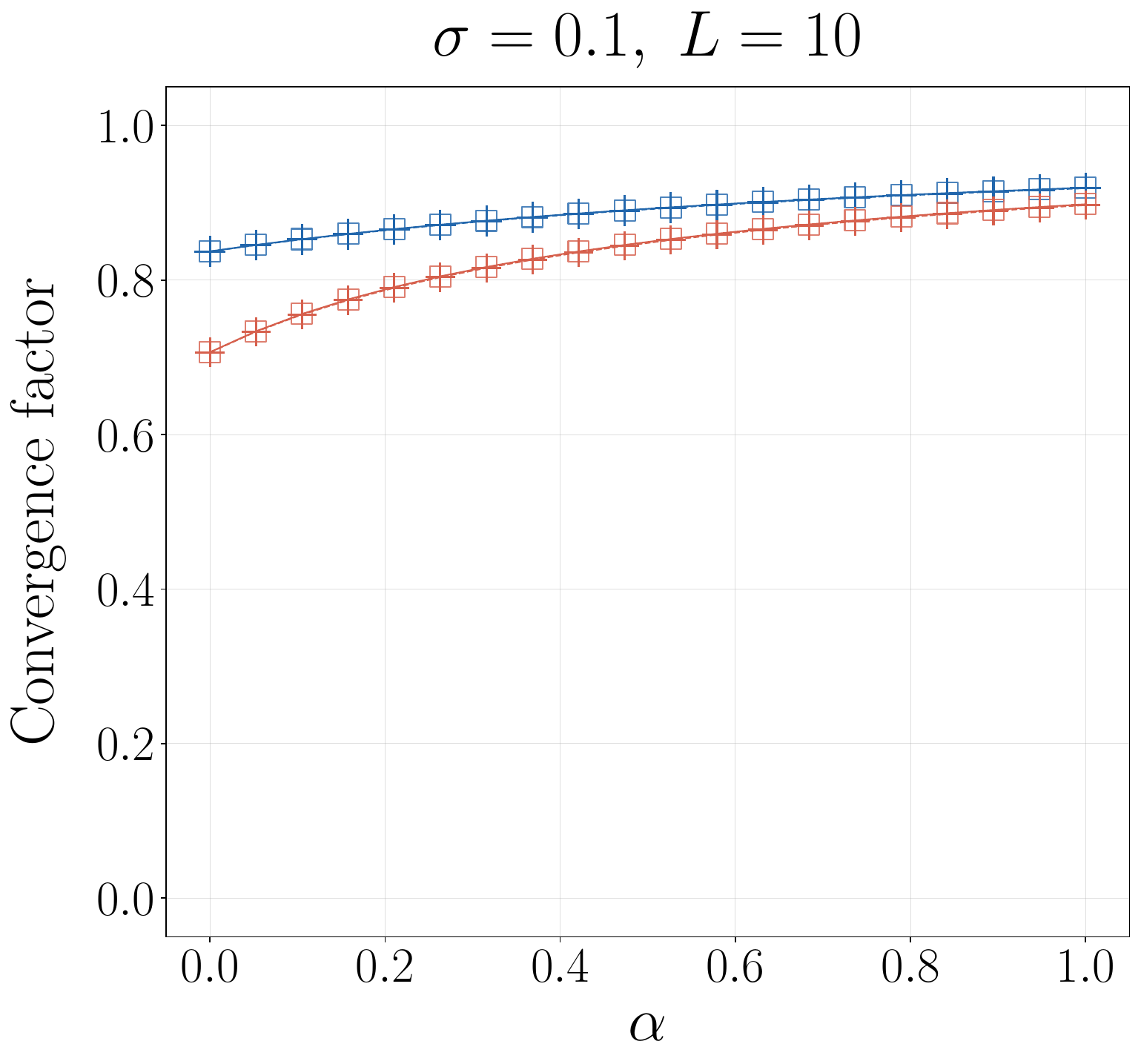}}
\end{subfigure}\hfill
\begin{subfigure}{0.3\textwidth}{\includegraphics[width=\linewidth, trim={0.1cm 0cm 0.2cm 0.2cm},clip]{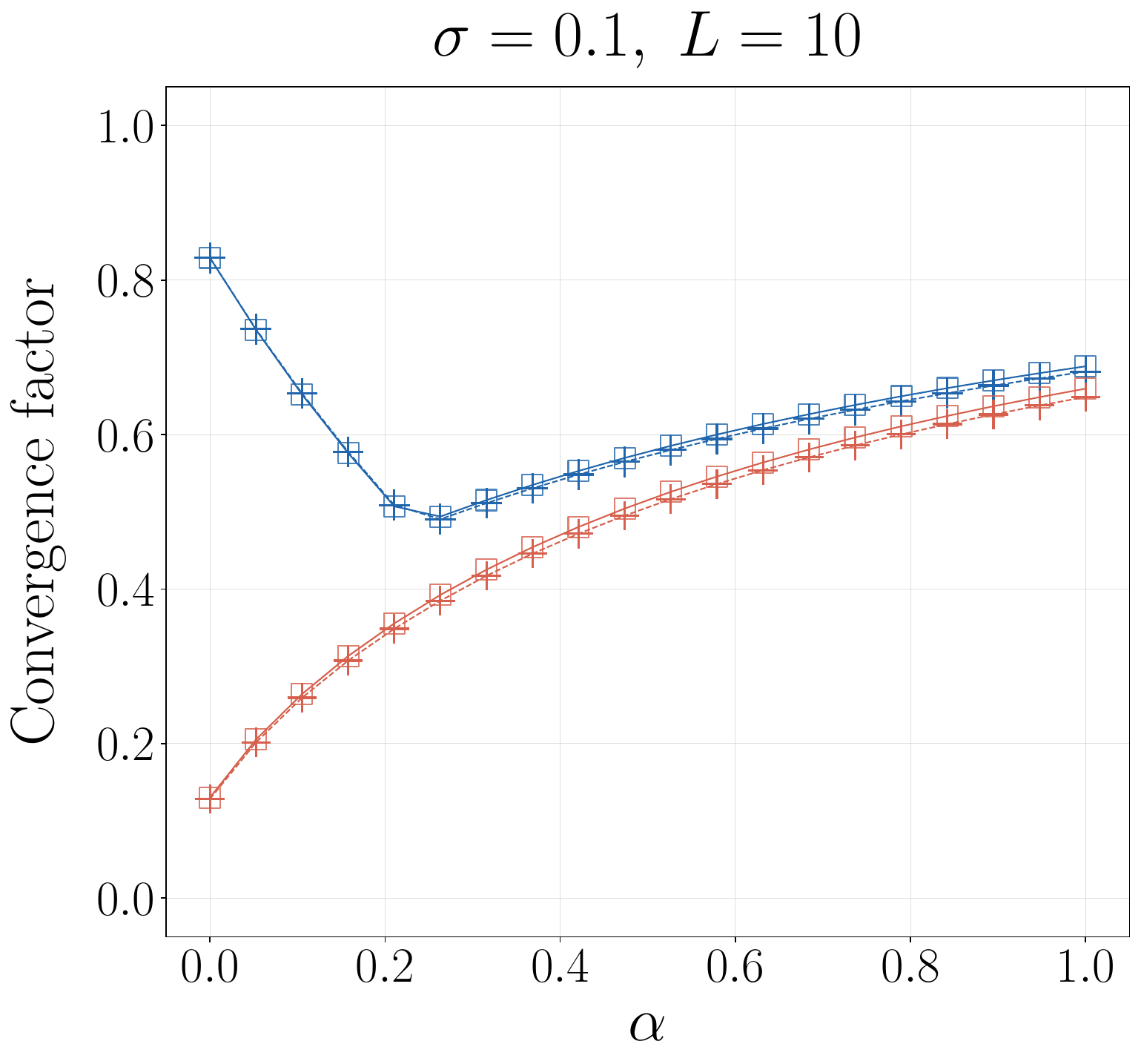}}
\end{subfigure}
\\[6pt]
\begin{subfigure}{0.3\textwidth}{\includegraphics[width=\linewidth, trim={0.1cm 0cm 0.2cm 0.2cm},clip]{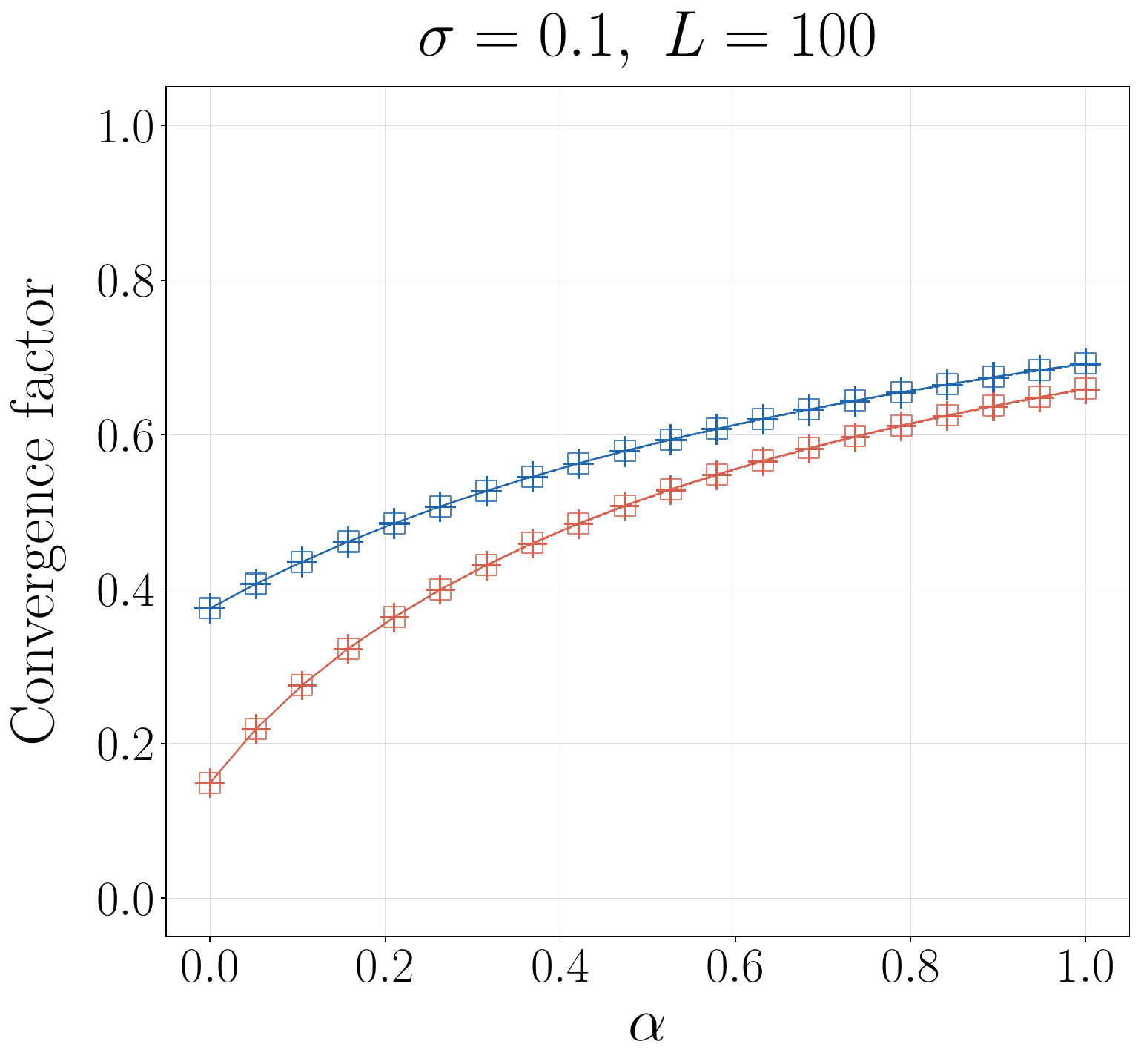}}
\end{subfigure}\hfill
\begin{subfigure}{0.3\textwidth}{\includegraphics[width=\linewidth, trim={0.1cm 0cm 0.2cm 0.2cm},clip]{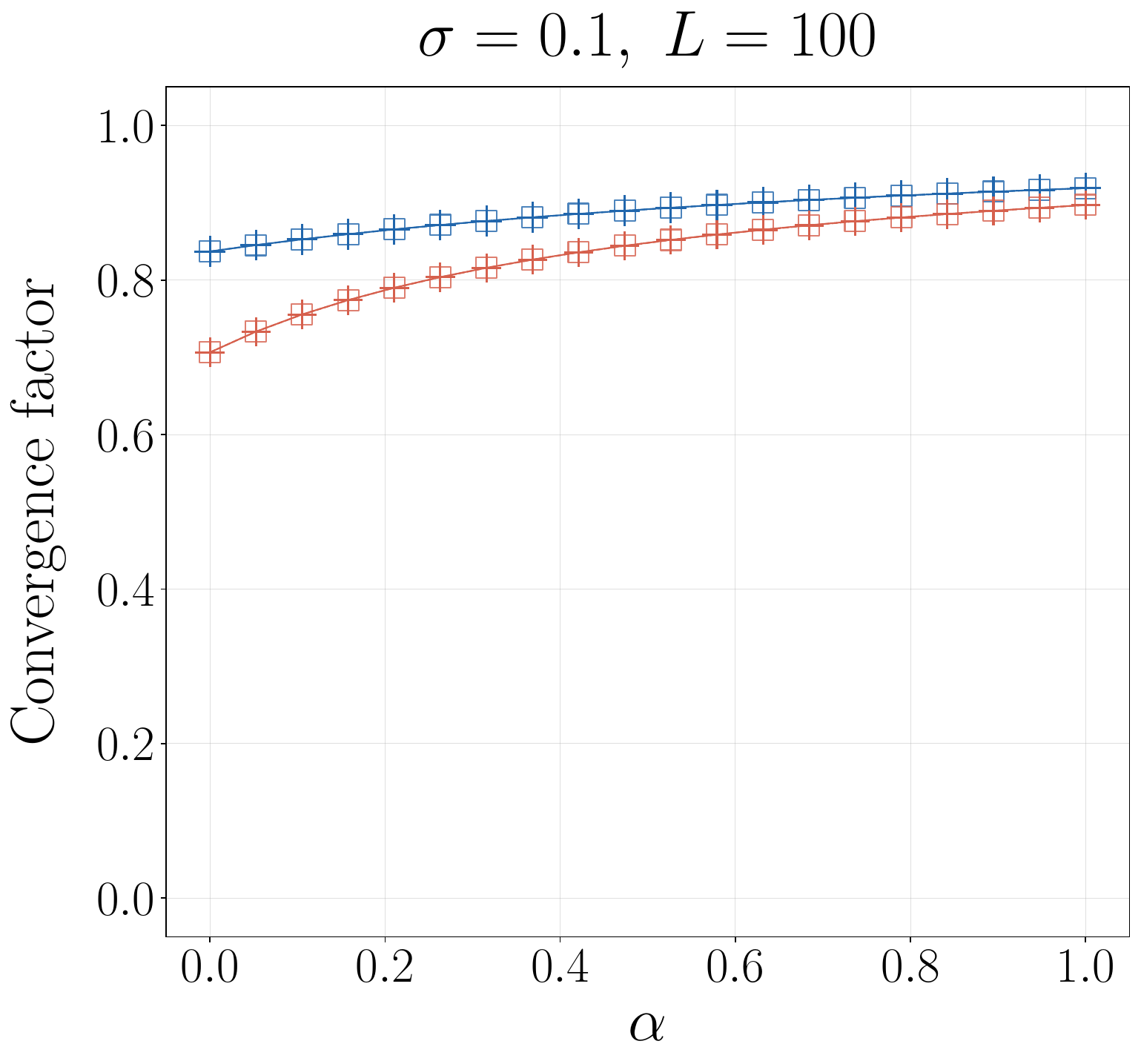}}
\end{subfigure}\hfill
\begin{subfigure}{0.3\textwidth}{\includegraphics[width=\linewidth, trim={0.1cm 0cm 0.2cm 0.2cm},clip]{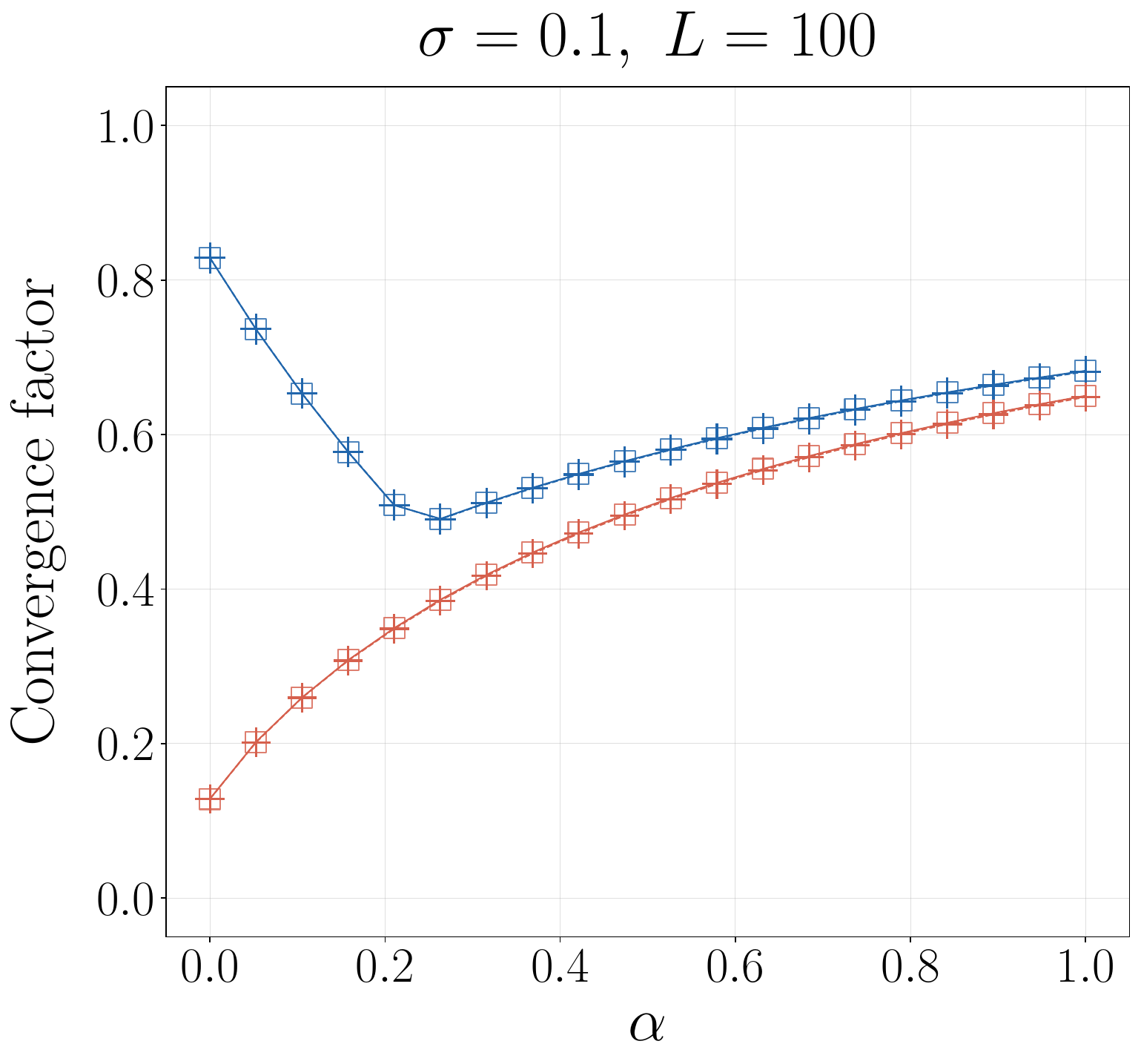}}
\end{subfigure}
\\[6pt]
\begin{subfigure}{0.3\textwidth}{\includegraphics[width=\linewidth, trim={0.1cm 0cm 0.2cm 0.2cm},clip]{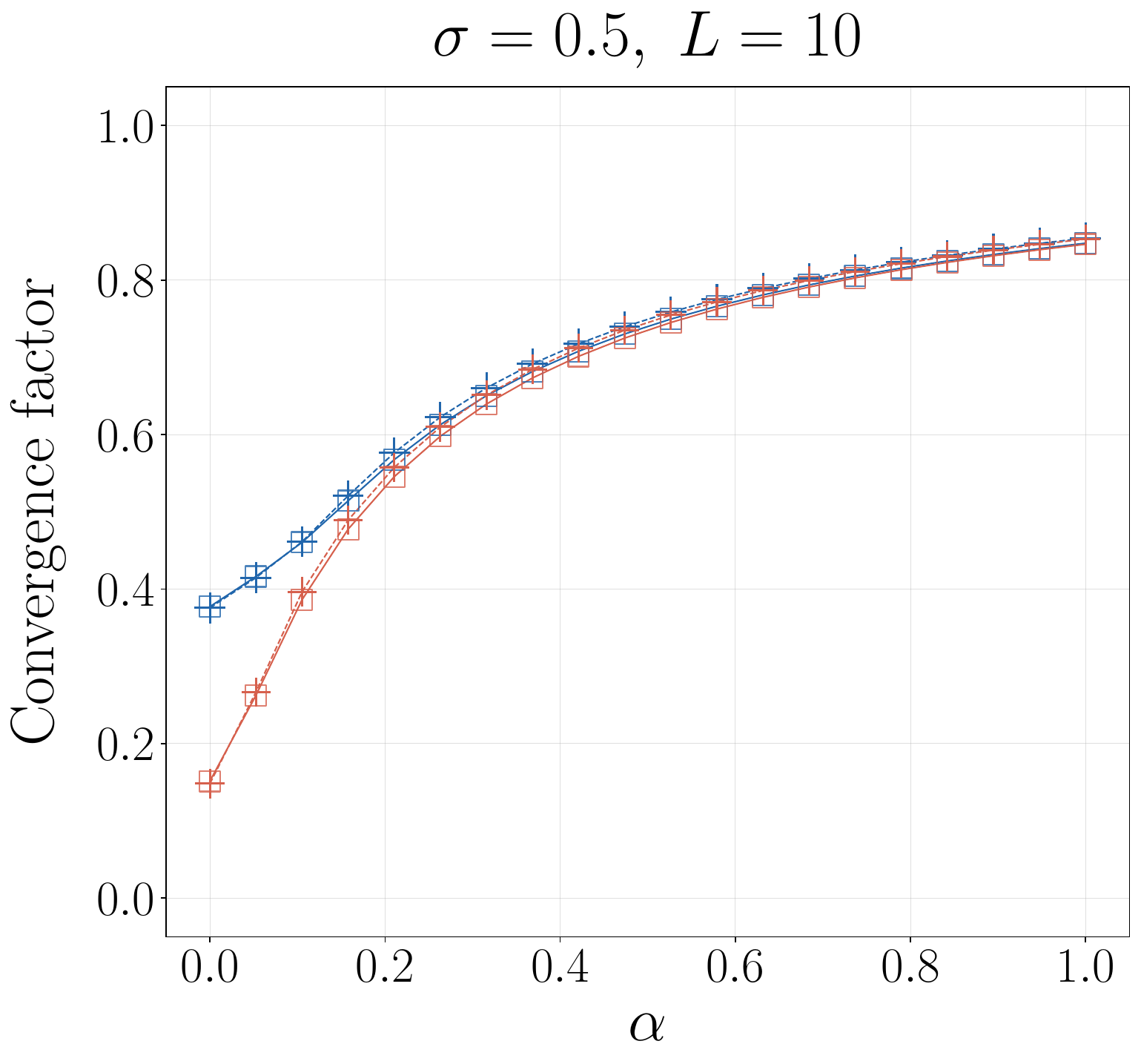}}
\end{subfigure}\hfill
\begin{subfigure}{0.3\textwidth}{\includegraphics[width=\linewidth, trim={0.1cm 0cm 0.2cm 0.2cm},clip]{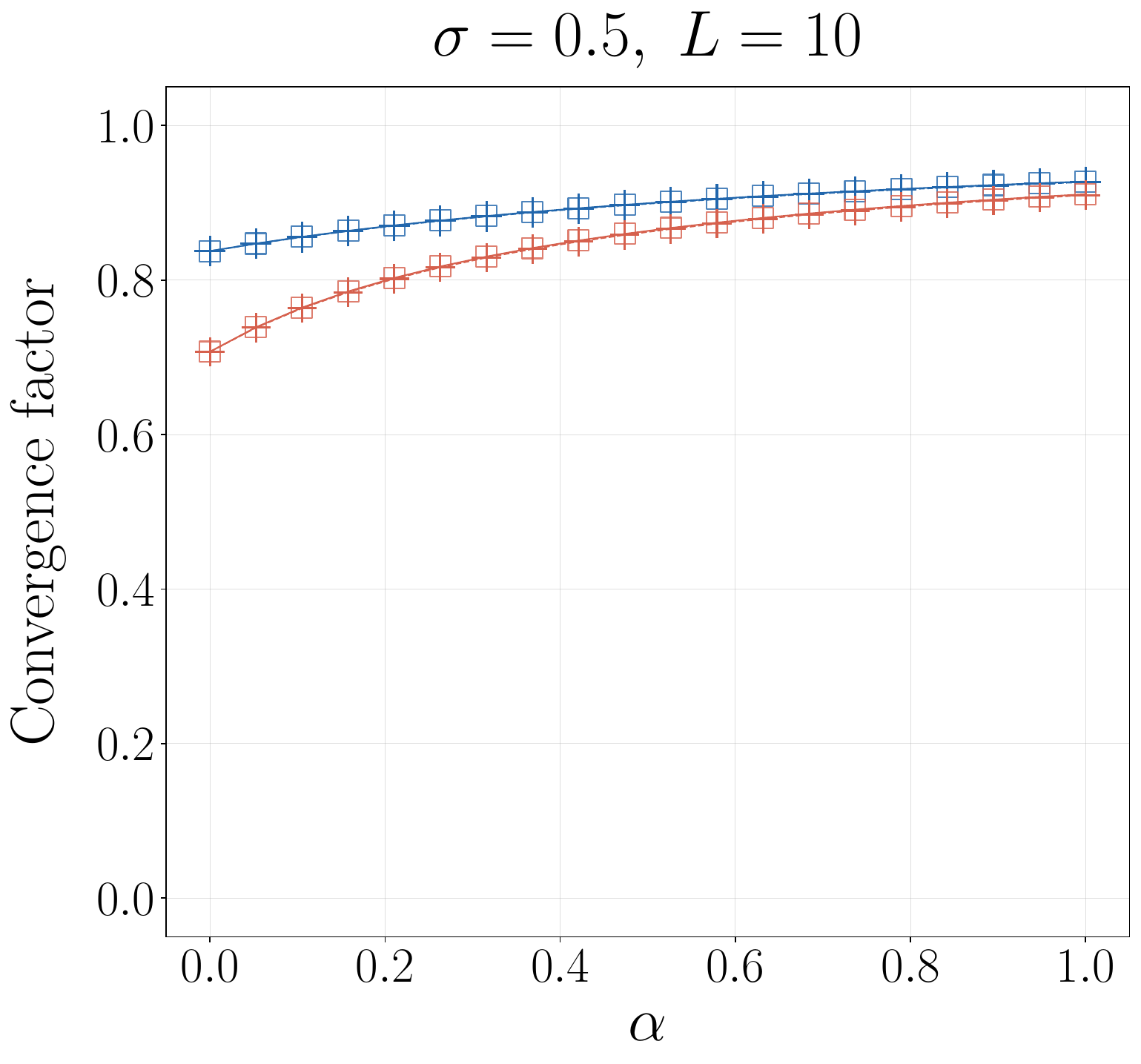}}
\end{subfigure}\hfill
\begin{subfigure}{0.3\textwidth}{\includegraphics[width=\linewidth, trim={0.1cm 0cm 0.2cm 0.2cm},clip]{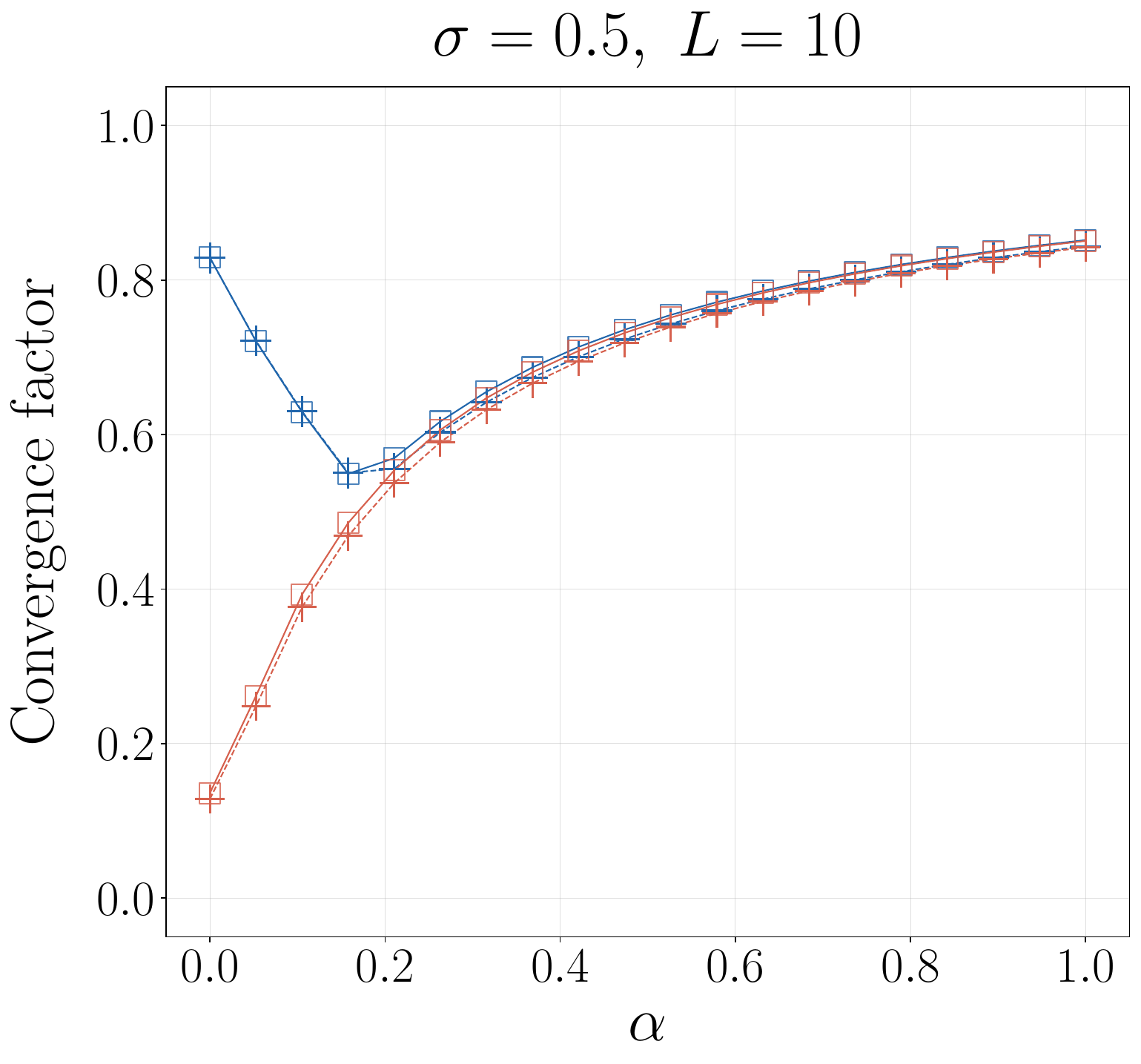}}
\end{subfigure}
\\[6pt]
\begin{subfigure}{0.3\textwidth}{\includegraphics[width=\linewidth, trim={0.1cm 0cm 0.2cm 0.2cm},clip]{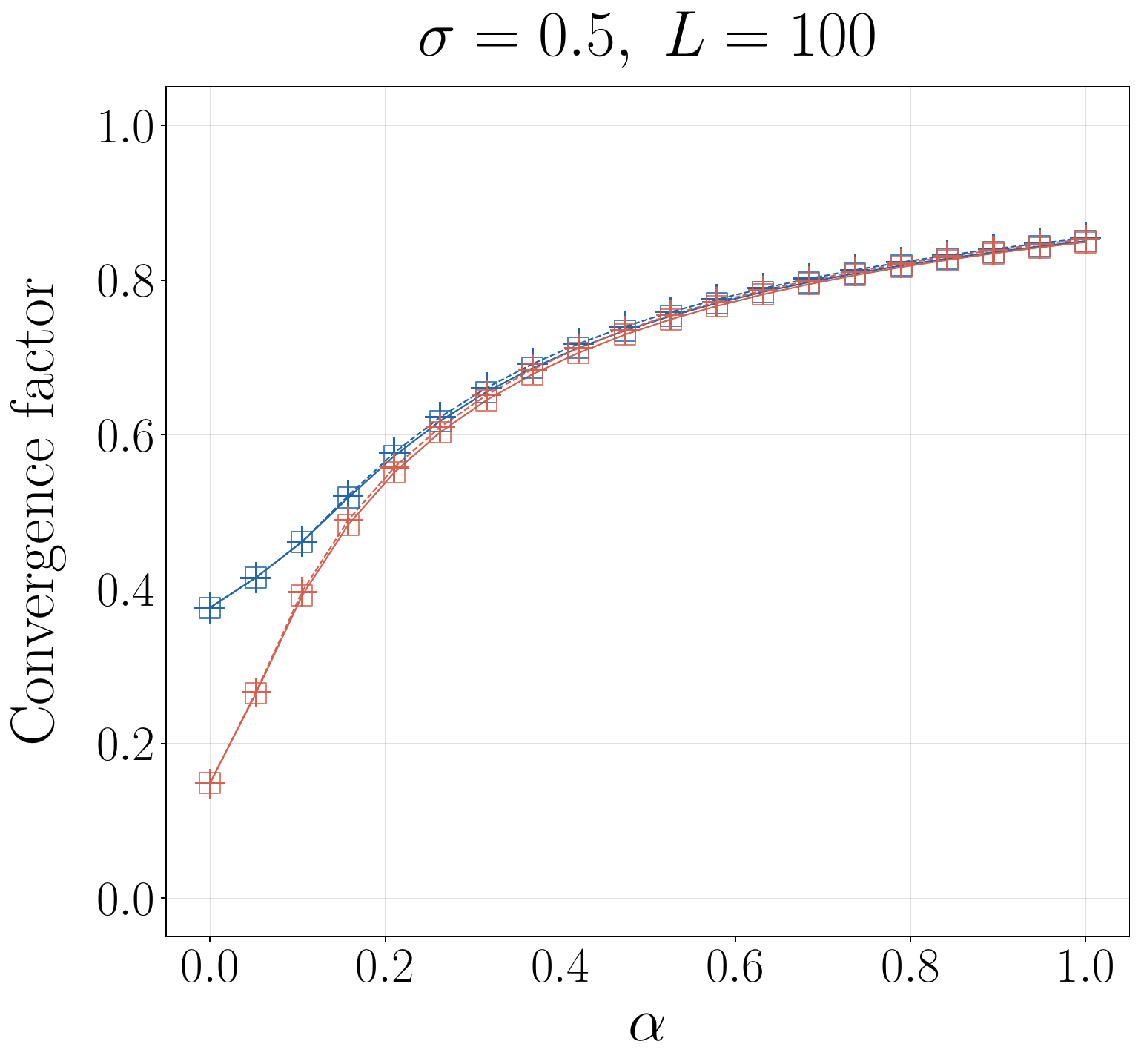}}
\end{subfigure}\hfill
\begin{subfigure}{0.3\textwidth}{\includegraphics[width=\linewidth, trim={0.1cm 0cm 0.2cm 0.2cm},clip]{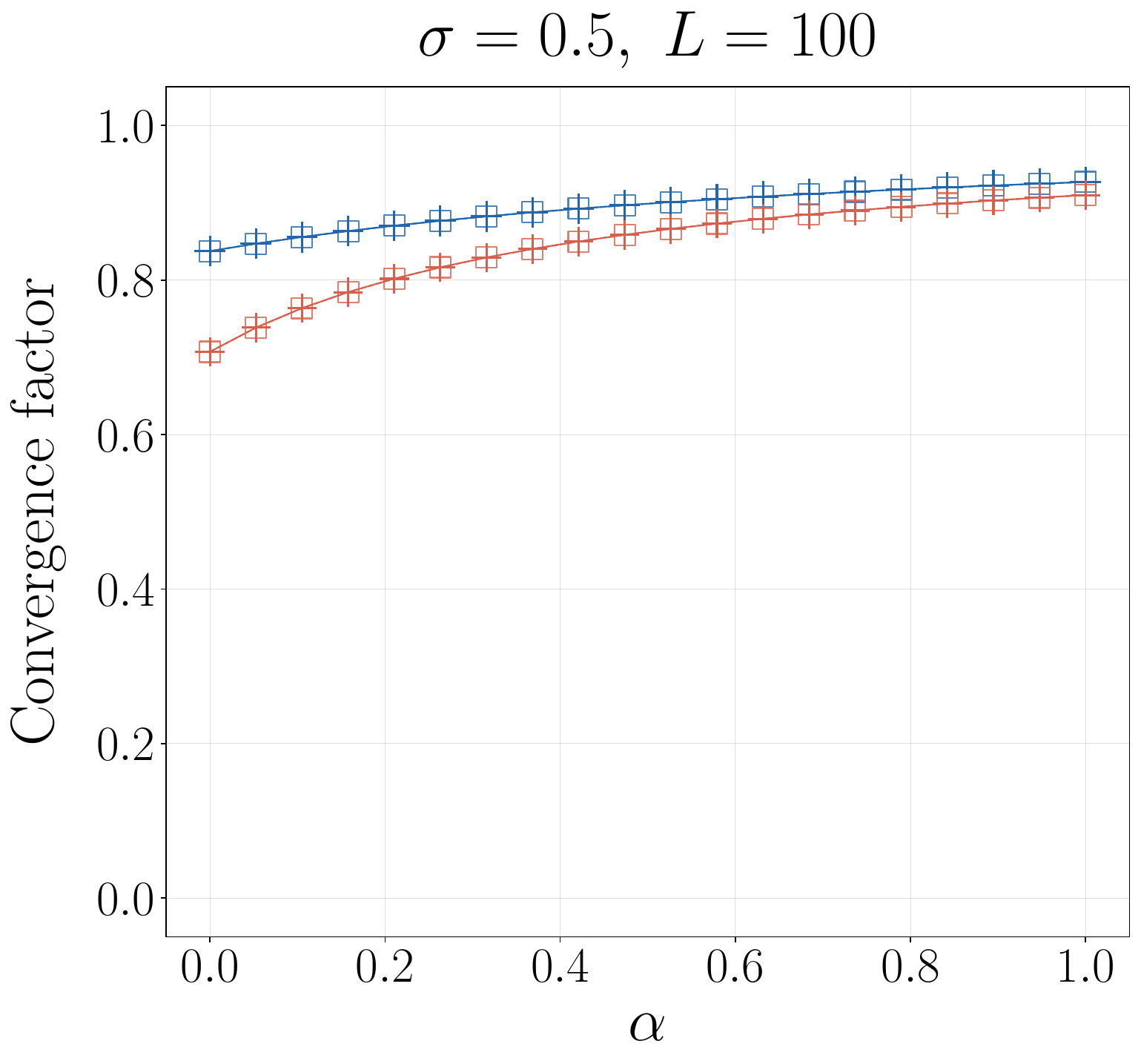}}
\end{subfigure}\hfill
\begin{subfigure}{0.3\textwidth}{\includegraphics[width=\linewidth, trim={0.1cm 0cm 0.2cm 0.2cm},clip]{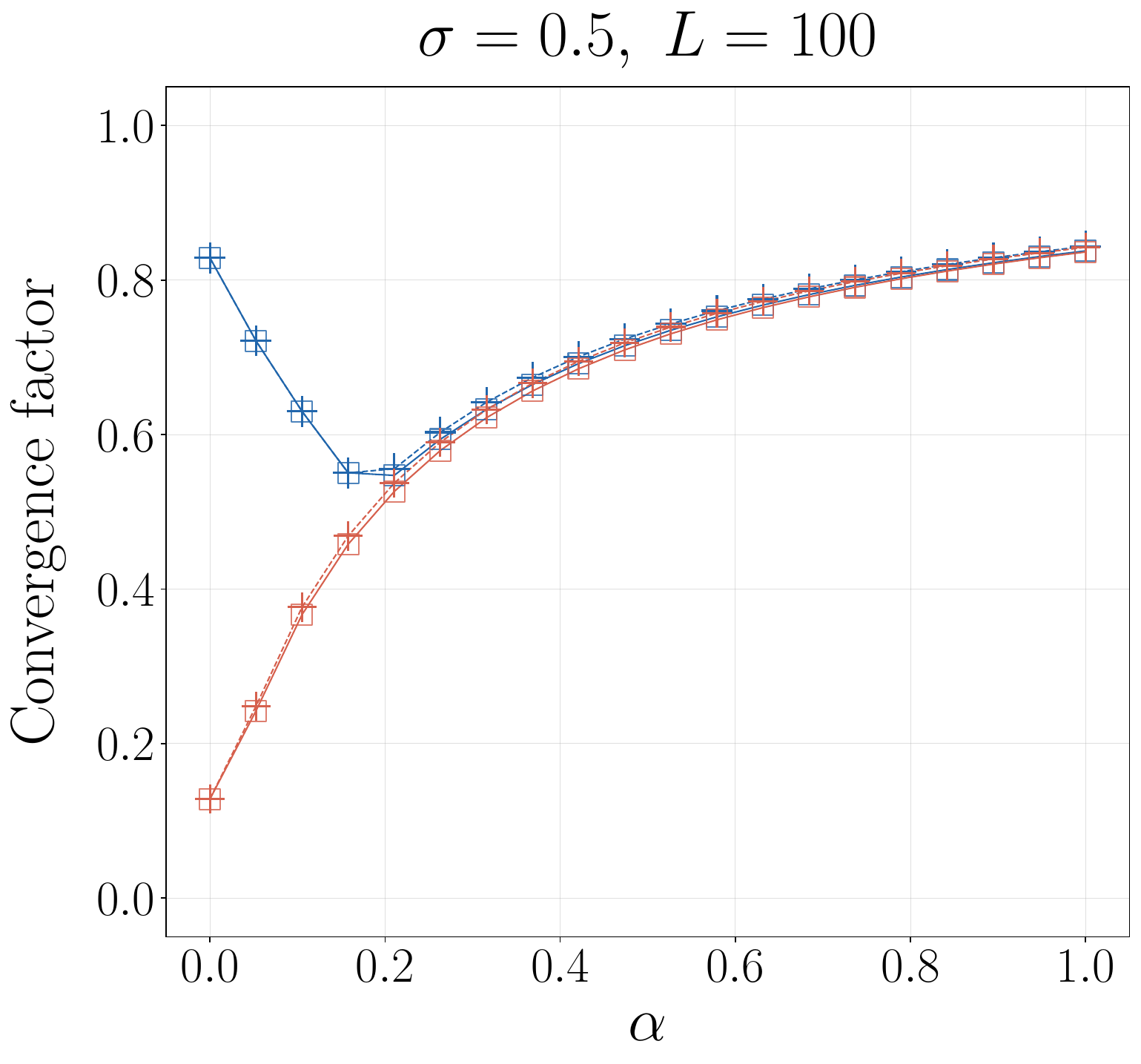}}
\end{subfigure}
\caption{Local convergence factors as functions in $\alpha\in[0,1]$,
for synthetic test problems of Example~\ref{ex:synth}, generated under the three different comparison graph structures:
Homogeneous (left column), Clustered (middle column), and Near-bipartite (right column),
with control parameters $\sigma=0.1,\ 0.5$ (dynamic range) and $L=10,\ 100$ (comparison number).}
\label{fig:1}
\end{figure}

\begin{table}[htbp]
\centering
\caption{
	Local convergence factor $\rho$
	and the corresponding gap $1-\rho$ for the boundary cases $\alpha=0$ and $\alpha=1$
	in~\Cref{fig:1}, with $L=100$.
	All reported values are rounded to three decimal places.
	 }
\label{tab:results}
\renewcommand{\arraystretch}{1.15}
\footnotesize
\begin{tabular}{l l cc cc}
\toprule
& & \multicolumn{2}{c}{$\sigma = 0.1$} & \multicolumn{2}{c}{$\sigma = 0.5$} \\
\cmidrule(lr){3-4} \cmidrule(lr){5-6}
Graph & Algorithm
  & $\rho$ / $1-\rho$ & $\bar{\rho}$ / $1-\bar{\rho}$
  & $\rho$ / $1-\rho$ & $\bar{\rho}$ / $1-\bar{\rho}$\\
\midrule
\multirow{4}{*}{Homogeneous}
   & $\alpha=0$, sync
    & 0.375 / 0.625 & 0.375 / 0.625
    & 0.376 / 0.624 & 0.376 / 0.624 \\
  & $\alpha=1$, sync
    & 0.692 / 0.308 & 0.691 / 0.309
    & 0.850 / 0.150 & 0.854 / 0.146 \\
  & $\alpha=0$, async
    & 0.149 / 0.851 & 0.149 / 0.851
    & 0.149 / 0.851 & 0.148 / 0.852 \\
  & $\alpha=1$, async
    & 0.659 / 0.341 & 0.658 / 0.342
   & 0.849 / 0.151 & 0.853 / 0.147 \\
\midrule
\multirow{4}{*}{Clusters}
  & $\alpha=0$, sync
    & 0.837 / 0.163 & 0.837 / 0.163
    & 0.837 / 0.163 & 0.837 / 0.163 \\
  & $\alpha=1$, sync
    & 0.919 / 0.081 & 0.919 / 0.081
    & 0.927 / 0.073 & 0.927 / 0.073 \\
  & $\alpha=0$, async
    & 0.706 / 0.294 & 0.706 / 0.294
    & 0.707 / 0.293 & 0.707 / 0.293 \\
  & $\alpha=1$, async
    & 0.897 / 0.103 & 0.897 / 0.103
    & 0.910 / 0.090 & 0.910 / 0.090 \\
\midrule
\multirow{4}{*}{Near-bipartite}
  & $\alpha=0$, sync
    & 0.829 / 0.171 & 0.829 / 0.171
    & 0.829 / 0.171 & 0.829 / 0.171 \\
  & $\alpha=1$, sync
    & 0.683 / 0.317 & 0.682 / 0.318
    & 0.838 / 0.162 & 0.844 / 0.156 \\
  & $\alpha=0$, async
    & 0.128 / 0.872 & 0.128 / 0.872
    & 0.128 / 0.872 & 0.128 / 0.872 \\
  & $\alpha=1$, async
    & 0.650 / 0.350 & 0.649 / 0.351
    & 0.836 / 0.164 & 0.842 / 0.158 \\
\bottomrule
\end{tabular}
\par\smallskip
\raggedright\footnotesize
\end{table}

\begin{figure}[htbp]
\begin{subfigure}{0.32\textwidth}\includegraphics[width=\linewidth, trim={0.0cm 0.5cm 4cm 1.5cm},clip]{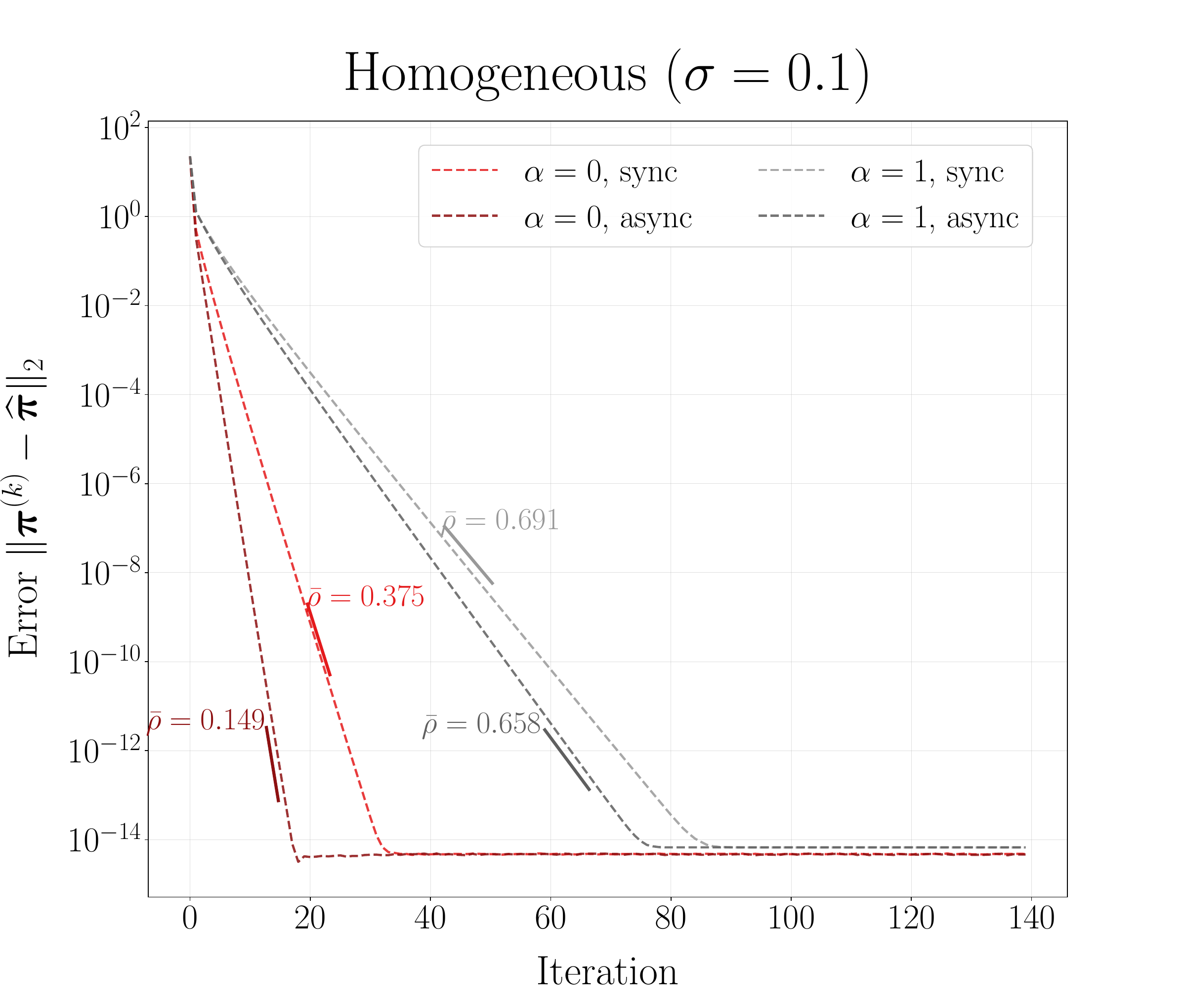}
\end{subfigure}\hfill
  \begin{subfigure}{0.32\textwidth}{\includegraphics[width=\linewidth, trim={0.0cm 0.5cm 4cm 1.5cm},clip]{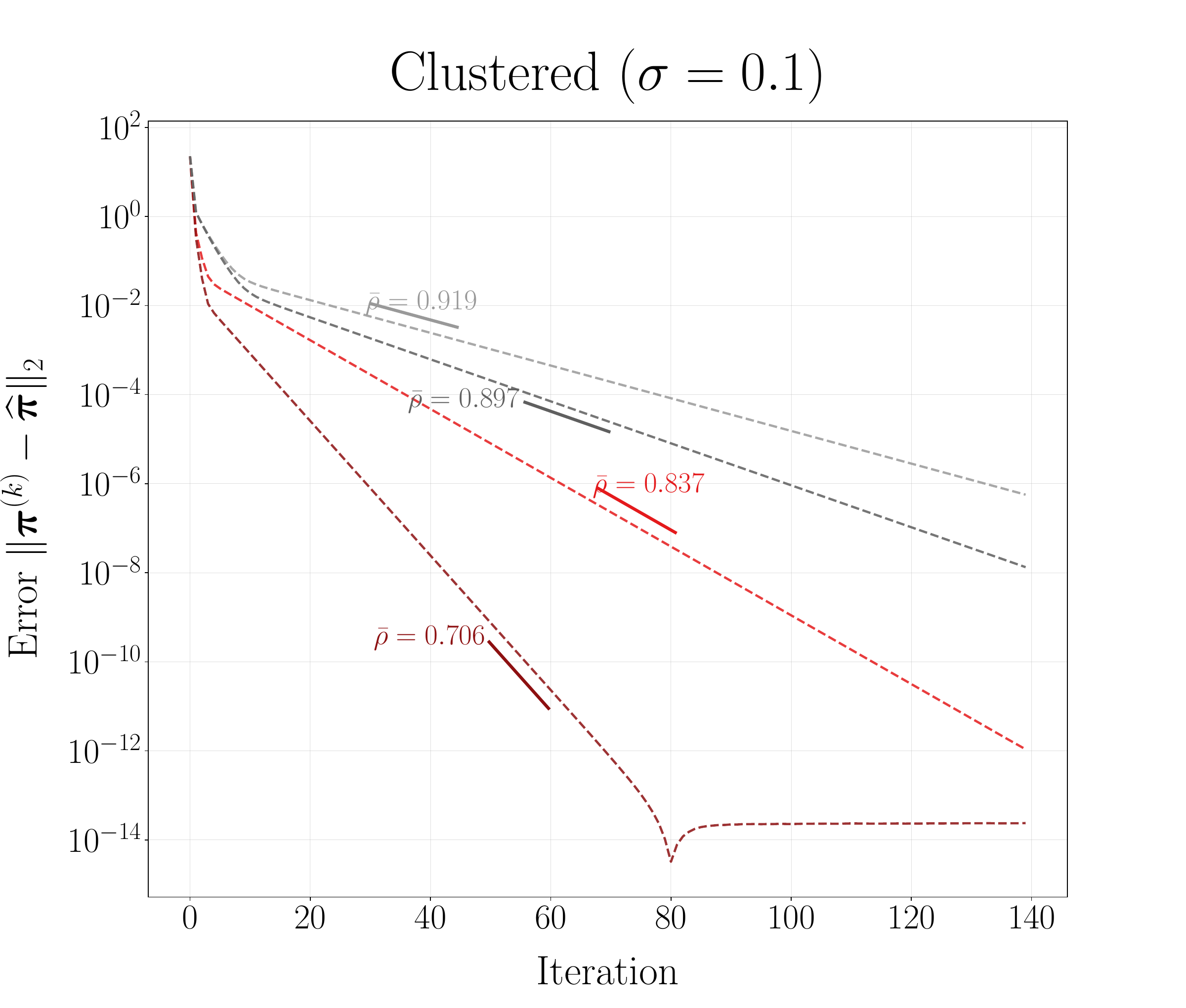}}
\end{subfigure}\hfill
  \begin{subfigure}{0.32\textwidth}{\includegraphics[width=\linewidth, trim={0.0cm 0.5cm 4cm 1.5cm},clip]{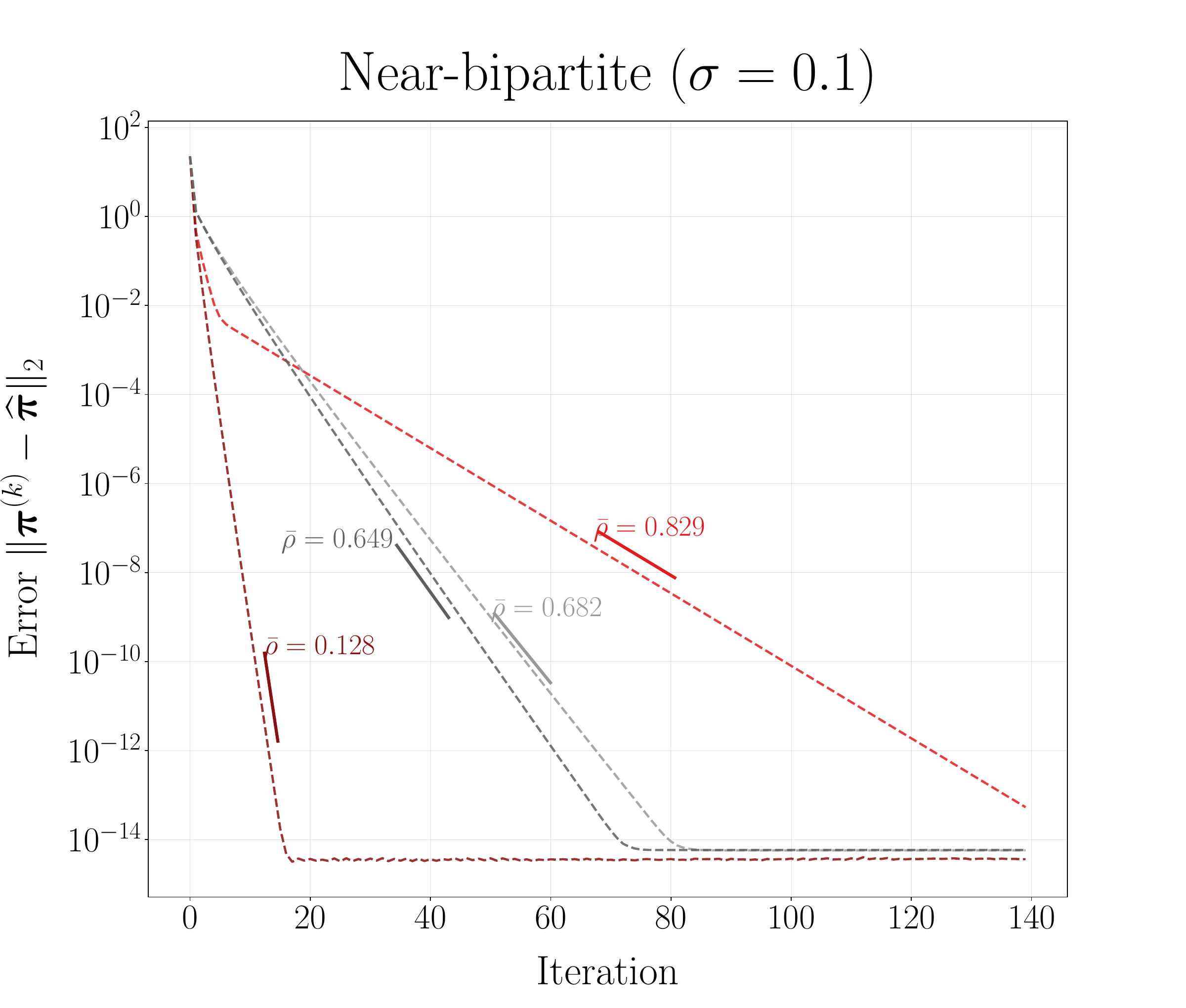}}  \end{subfigure}
 \\[10pt]
  \begin{subfigure}{0.32\textwidth}\includegraphics[width=\linewidth, trim={0.0cm 0.5cm 4cm 1.5cm},clip]{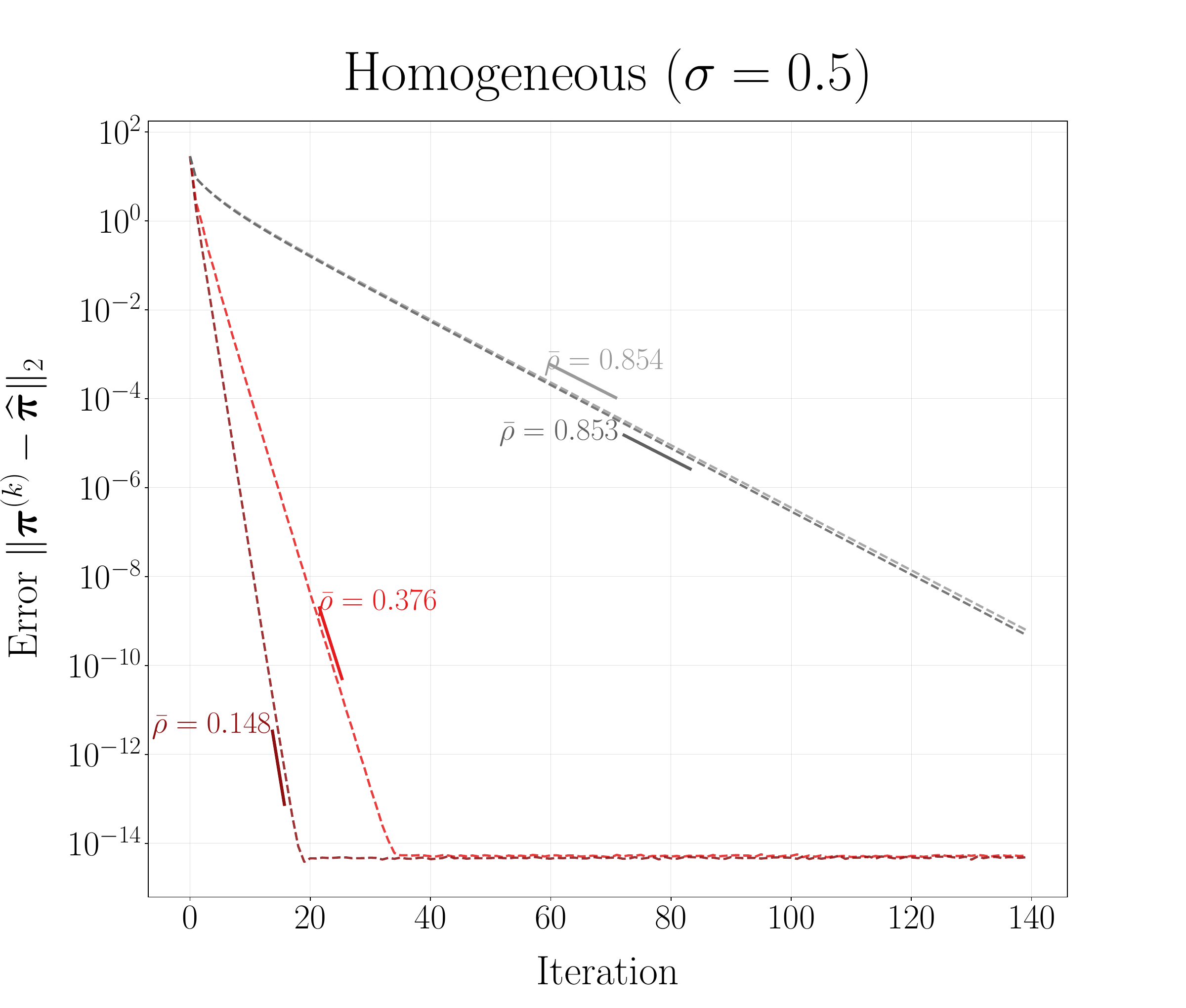}
  \end{subfigure}\hfill
  \begin{subfigure}{0.32\textwidth}{\includegraphics[width=\linewidth, trim={0.0cm 0.5cm 4cm 1.5cm},clip]{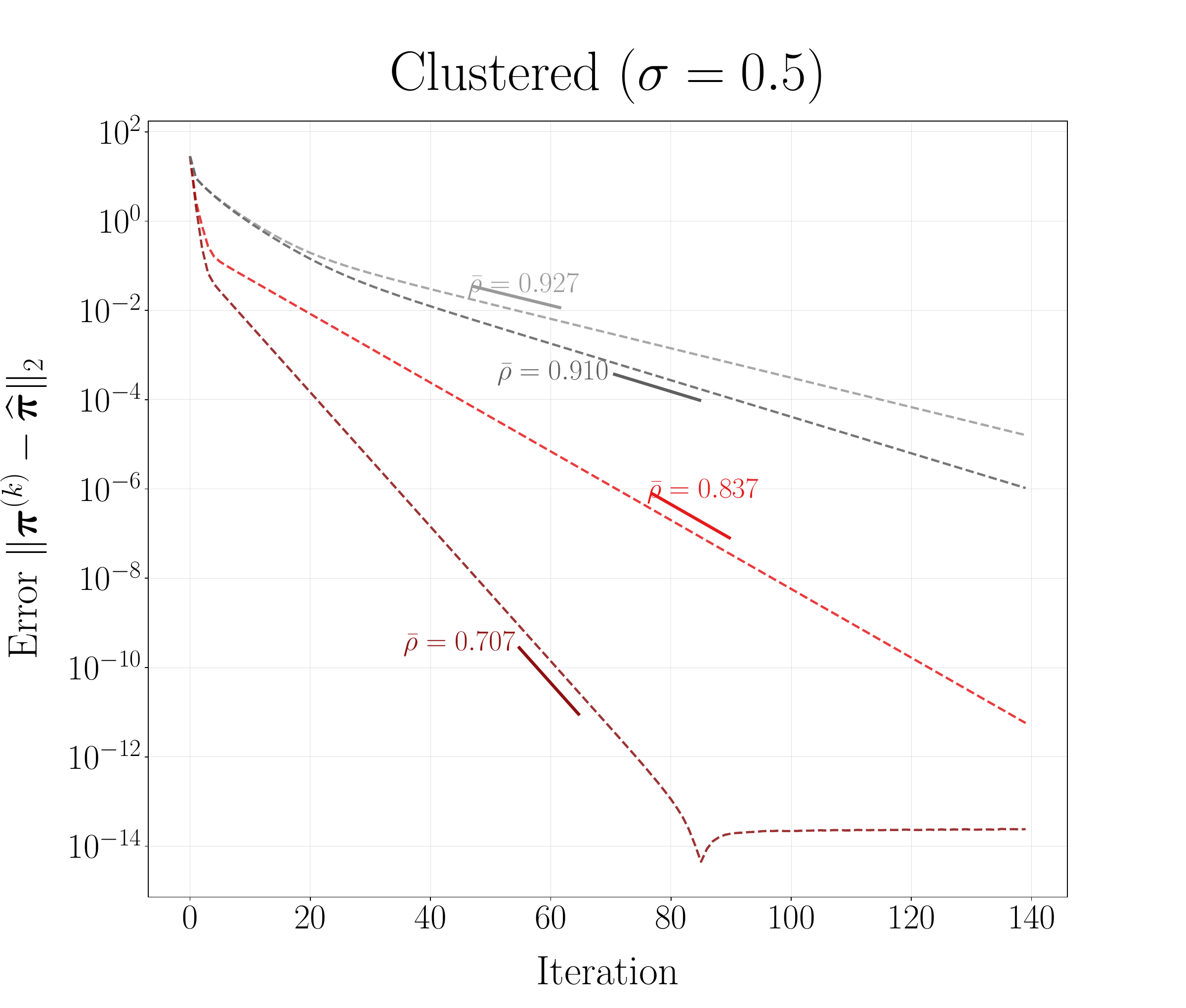}}
  \end{subfigure}\hfill
  \begin{subfigure}{0.32\textwidth}{\includegraphics[width=\linewidth, trim={0.0cm 0.5cm 4cm 1.5cm},clip]{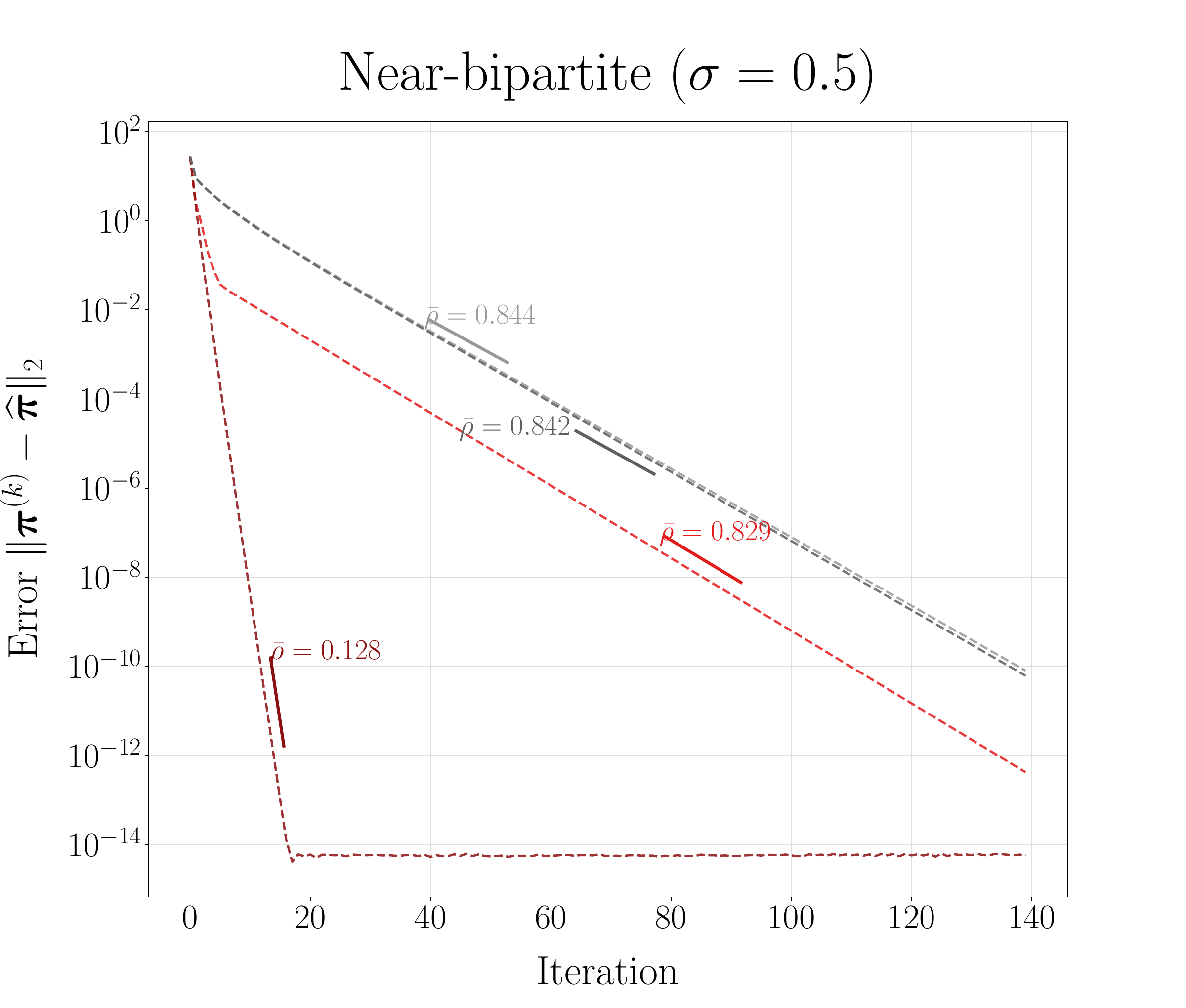}}
    \end{subfigure}
\caption{
	Convergence histories of Newman's $\alpha$-scheme with $\alpha=0$
	and $\alpha=1$ for the synthetic test problems in~\Cref{fig:1}, with $L=100$. 
	All algorithms are initialized from the all-ones vector.
	Solid lines indicate the theoretical asymptotic linear rates
	corresponding to the expected convergence factors $\bar\rho$.
}\label{fig:112} 
\end{figure}
\end{myexample}

\begin{myexample}\rm \label{eg:assumptions}
This example presents two representative cases in which the synchronous update in Newman's algorithm ($\alpha=0$) fails to converge
and therefore does not provide any acceleration over the classical Zermelo's algorithm ($\alpha=1$). In contrast, the asynchronous update remains convergent and may offer significant acceleration. These results further highlight the importance and effectiveness of asynchronous updates 
in the implementation of Newman's algorithm.

In the first experiment, we consider the bipartite comparison graph. In this case, $\brhosync(0)=1$ as noted in Remark~\ref{rmk:ssls}, indicating that Newman's $\alpha$-schemes with synchronous update may fail to converge when $\alpha=0$. In contrast,  with asynchronous updates, the algorithm is guaranteed to converge for all $\alpha\geq 0$ since $\rhoasync(\alpha)<1$ according to~\Cref{thm:rhoasync}.
For comparison, we test asynchronous updates using both the consistent ordering (grouping objects within each part together) and a random ordering (using random permutation).

Our test problem is generated using the same setup as in Example~\ref{ex:synth},
with parameters $p = 0$, $q = 0.5$, $L=10$, and $\sigma = 0.5$.
\Cref{3:(a),3:(b)} report 
the local convergence factors of the algorithms
and the convergence histories for $\alpha=0,1$.
As expected, $\rhosync(0)=1$ (to 3 decimal places) and $\rhoasync(\alpha)<1$.
Indeed, when $\alpha=0$, the synchronous update fails to converge,
while the asynchronous update remains convergent and offers
significant acceleration over Zermelo's algorithm ($\alpha=1$),
with speedup approximately $17.33$ for consistent ordering and $6.95$ for random
ordering, respectively.
From the left panel, the convergence factor $\rhoasync(\alpha)$ corresponding to the consistent ordering 
increases monotonically with $\alpha$, in agreement with the analysis in~\Cref{thm:coo}. This monotonicity is lost under the random ordering, although the two curves remain quite close for larger values of $\alpha$ away from zero.

In the second experiment, we consider the circular comparison graph, where the objects are arranged in a cycle, and each object is compared only with its two neighbors (a specialized bipartite graph). We study a challenging case in which the comparison outcomes in $\W$ are {\em cyclic}, i.e., $w_{ij}\neq 0$ if and only if $j \equiv i + 1 \pmod n$. The cyclic outcomes create a preference loop, with object $i$ preferred over $i+1$,
making the ranking based on the strength vector $\bm\pi$ under the BT model less informative (as the BT model is stochastically transitive), and slowing down the convergence of algorithms. Indeed, it can be shown that such a cyclic $\W$ is unlikely to arise under the BT model (see \Cref{app:bto1} for details), so the convergence analysis of Newman's $\alpha$-schemes based on the expected outcomes of the BT model may not accurately reflect the actual behavior of the algorithms (in particular, $\brhosync$ and $\rhosync$ may differ significantly).

For the numerical test, we set the problem size $n = 20$ 
and sample the nonzero $w_{ij}$ in the cyclic $\W$ 
as independent random integers uniformly distributed over 
$[10^7, 5\times 10^7]$. 
This yields a strongly connected comparison graph $\G([n],\W)$,
for which the MLE $\pih$ uniquely exists and is finite. 
\Cref{3:(c)} reports the local convergence factors of Newman's
$\alpha$-schemes for $\alpha\in[0,1]$.
For the synchronous updates, we observe that Zermelo's algorithm ($\alpha=1$) converges but exhibits a slow convergence rate and thus requires acceleration. However, decreasing $\alpha\to 0$ does not improve performance and in fact leads to divergence, as $\rhosync(\alpha)$ exceeds $1$.~\footnote{
Recall from~\Cref{thm:rhosync}  that the expected  $\brhosync(\alpha)\leq 1$ for all $\alpha\geq 0$, 
indicating a significant deviation from $\rhosync(\alpha)$,
as cyclic outcomes are unlikely under the BT model; see \Cref{app:bto1}.
}
In contrast, for the asynchronous updates, the algorithm remains convergent for all $\alpha$, consistent with~\Cref{thm:rhoasync}. At $\alpha=0$, the algorithm provides noticeable acceleration
compared with $\alpha=1$.
\end{myexample}

\begin{figure}[htbp]
\begin{subfigure}{0.3\textwidth}\includegraphics[width=\linewidth, trim={0.2cm 0cm 0.2cm 0.2cm},clip]{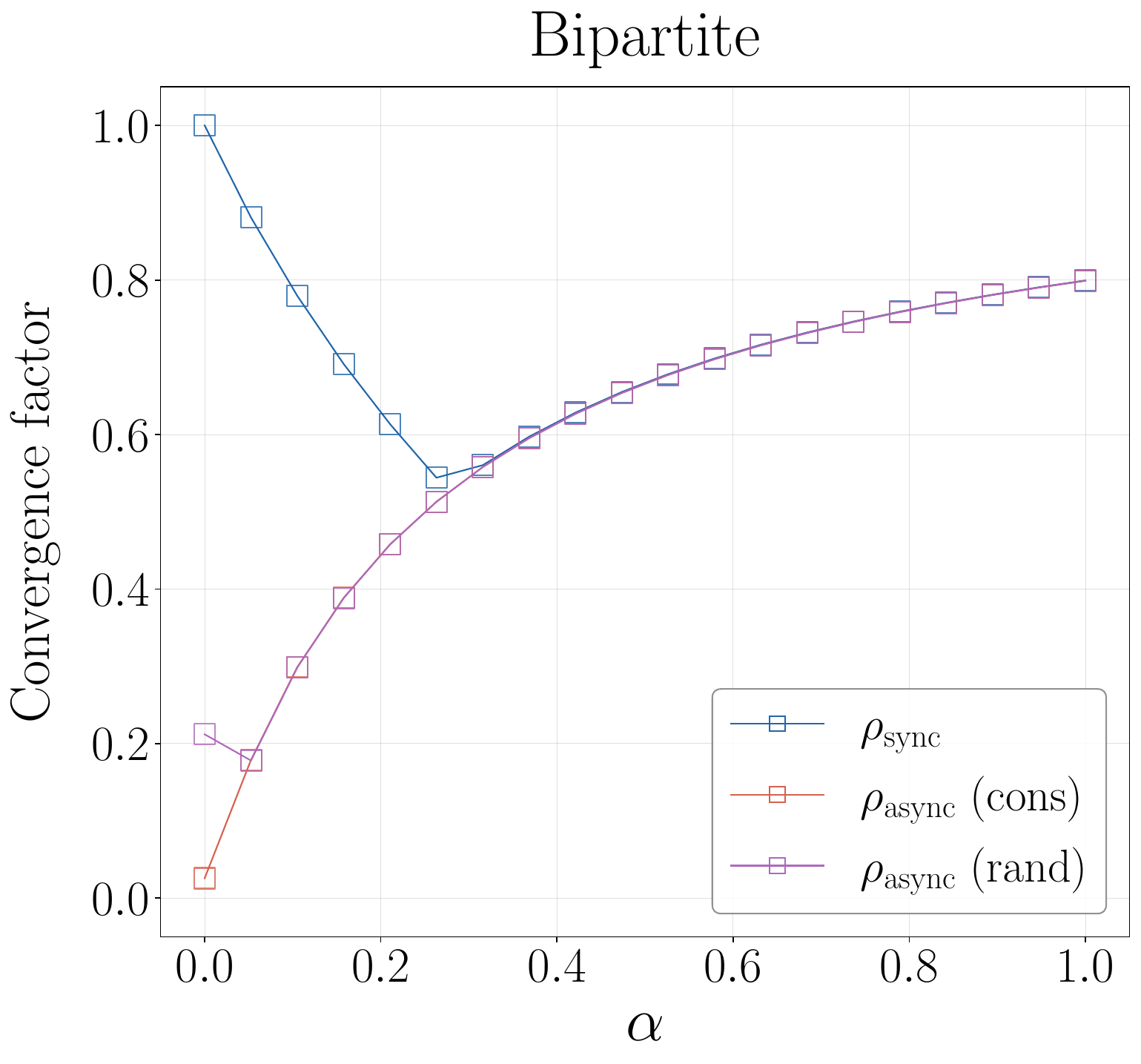}\caption{}\label{3:(a)}
\end{subfigure}\hfill
  \begin{subfigure}{0.32\textwidth}{\includegraphics[width=\linewidth, trim={0.2cm 0.0cm 0.2cm 0.25cm},clip]{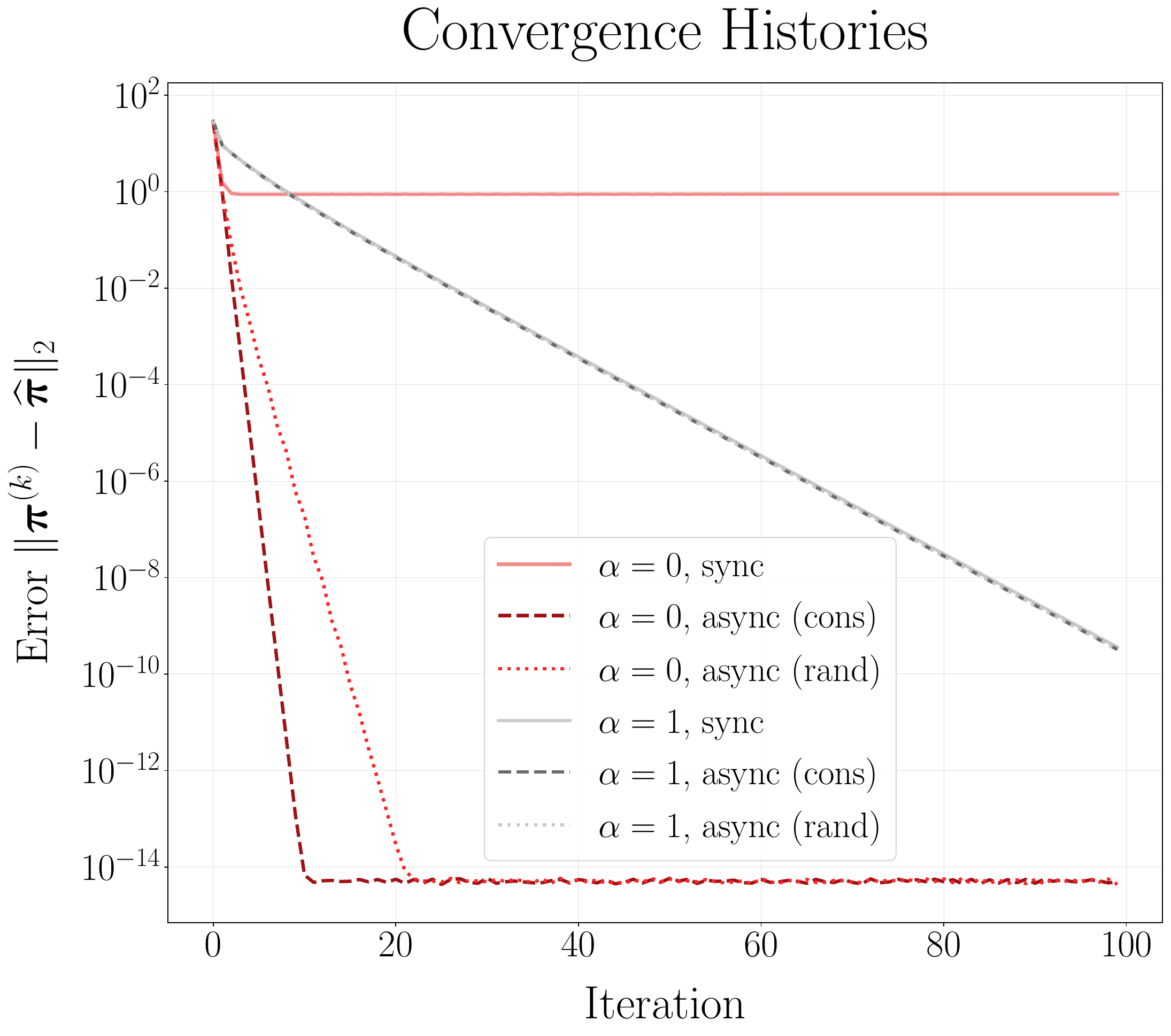}}\caption{}\label{3:(b)}
\end{subfigure}\hfill
  \begin{subfigure}{0.3\textwidth}{\includegraphics[width=\linewidth, trim={0.2cm 0cm 0.2cm 0.2cm},clip]{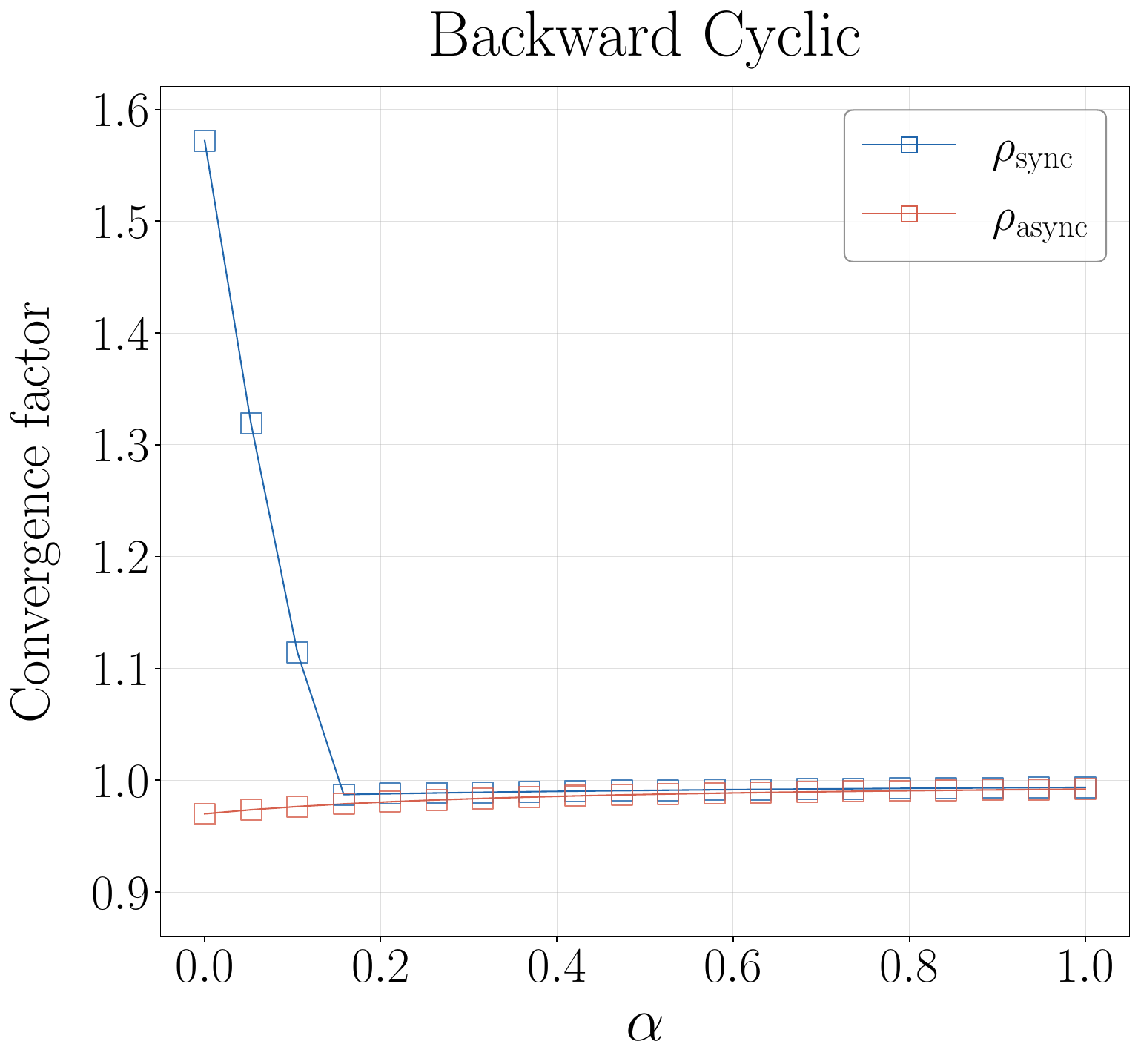}}\caption{}\label{3:(c)}
  \end{subfigure}
\caption{
	{\bf Bipartite comparison graph:}
	(a) 
	Local convergence factors of
	Newman's $\alpha$-scheme under synchronous and asynchronous updates;
	`cons' denotes consistent ordering, and `rand' denotes random
	ordering used in the asynchronous update;
	(b) 
	Convergence histories of the algorithms with
	$\alpha=0$ and $1$.
	{\bf Cyclic comparison graph:}
	(c) 
	Local convergence factors of algorithms.} \label{fig:2}
\end{figure}

\subsection{Real-world data}

This section demonstrates the effectiveness of our convergence analysis on real-world datasets. Comparison outcomes $\W$ arising in real-life applications may not follow the BT model exactly. For many applications, however, the convergence behavior of Newman's $\alpha$-scheme on those outcomes $\W$ can be predicted reasonably accurately by the fitted BT model obtained through the MLE. Specifically, given comparison outcomes $\W=\{w_{ij}\}$, one can fit a BT model~\eqref{bt-model}, with the same total number of comparisons $\M=\{m_{ij}\}$ with $m_{ij} = w_{ij}+w_{ji}$ and a strength vector estimated from the MLE~\eqref{mle0}, i.e., $\bm \pi^*\coloneqq\pih$. We show that the convergence factors $\brhosync$ and $\brhoasync$ based on the expected outcome of this fitted BT model may provide surprisingly accurate estimation of the observed convergence factors $\rhosync$ and $\rhoasync$, respectively; thereby further justifying the convergence analysis developed in~\Cref{sec:ana,sec:gs}.

For experiments, we consider three commonly used datasets for network data analysis and ranking algorithms, including a Vervet monkey dataset, an ATP tennis dataset, and an ASSISTments dataset. 

\begin{enumerate}
\item 
{\em Vervet monkeys.}\footnote{Available at \url{https://github.com/tbonne/rankReliability}}
A dominance hierarchy dataset recording agonistic interactions of
various kinds of wild vervet monkeys in the Samara Private Game Reserve
in South Africa \citep{vilette2020comparing}. After preprocessing to remove the objects with either zero wins or losses and retaining the largest strongly connected component, we are left with $n = 64$ objects and $N = 11,632$ comparisons. 

\item
{\em ATP tennis.}\footnote{Available at \url{https://github.com/PinjunD/Statistical-ranking-with-dynamic-covariate}}
A dataset of historical match outcomes from professional men's tennis tournaments
organized by the Association of Tennis Professionals (ATP).
The original data were collected by Jeff Sackmann\footnote{\url{https://github.com/JeffSackmann/tennis_MatchChartingProject}}, while the cleaned dataset used here is from \citet{dong2025statistical}, which consists of tour-level main-draw matches played between January 1980 and May 2024. 
This dataset contains $n$ = 1,803 players and $N$ = 136,935 matches, with each player participating in at least $10$ matches.

\item
{\em ASSISTments.}\footnote{Available at \url{https://sites.google.com/site/assistmentsdata/home/2009-2010-assistment-data/skill-builder-data-2009-2010}}
A student-problem response dataset collected from ASSISTments, an online math tutoring system that logs whether each student answers each problem correctly, during the 2009--2010 school year \citep{feng2009addressing}. We use the full dataset without restricting to specific problem skills, retaining only students and problems with at least $30$ interactions, and taking the largest strongly connected component of the resulting bipartite comparison graph. This yields $n = 1,545$ objects ($1,020$ problems and $525$ students) and $N = 51,788$ comparisons. 
\end{enumerate}

Following the previous experiments, we apply Newman's $\alpha$-scheme with both synchronous and asynchronous updates to the three test datasets and compare the observed local convergence factors $\rho$ with the predicted $\bar \rho$ computed from the expected outcome $\overline \W$ of the fitted BT model. As shown in the left panels of~\Cref{fig:3}, the predicted convergence factors closely match the observed ones, with the two curves visually overlapping for all $\alpha\in[0,1]$. This proximity is partly explained by the fact that the expected outcome matrix $\bar{\bm W}$ remains close to the original $\bm W$; particularly, the relative error $\|\bm W-\bar{\bm W}\|_2/\|\bm W\|_2$ is about $0.18$, $0.09$, and $0.19$ for Vervet Monkeys, ATP, and ASSISTments, respectively. These results indicate that the convergence properties predicted from the fitted BT model (for the ASSISTments, this is also called the Rasch model \citep{rasch1960studies}) can accurately capture the actual behavior of the algorithms on the real-world dataset, further supporting the practical relevance of our convergence analysis based on the expected outcomes of BT models.

\begin{figure}[htbp]
\centering
\includegraphics[width=.32\linewidth, trim={0.1cm 0cm 4cm 1.6cm},clip]{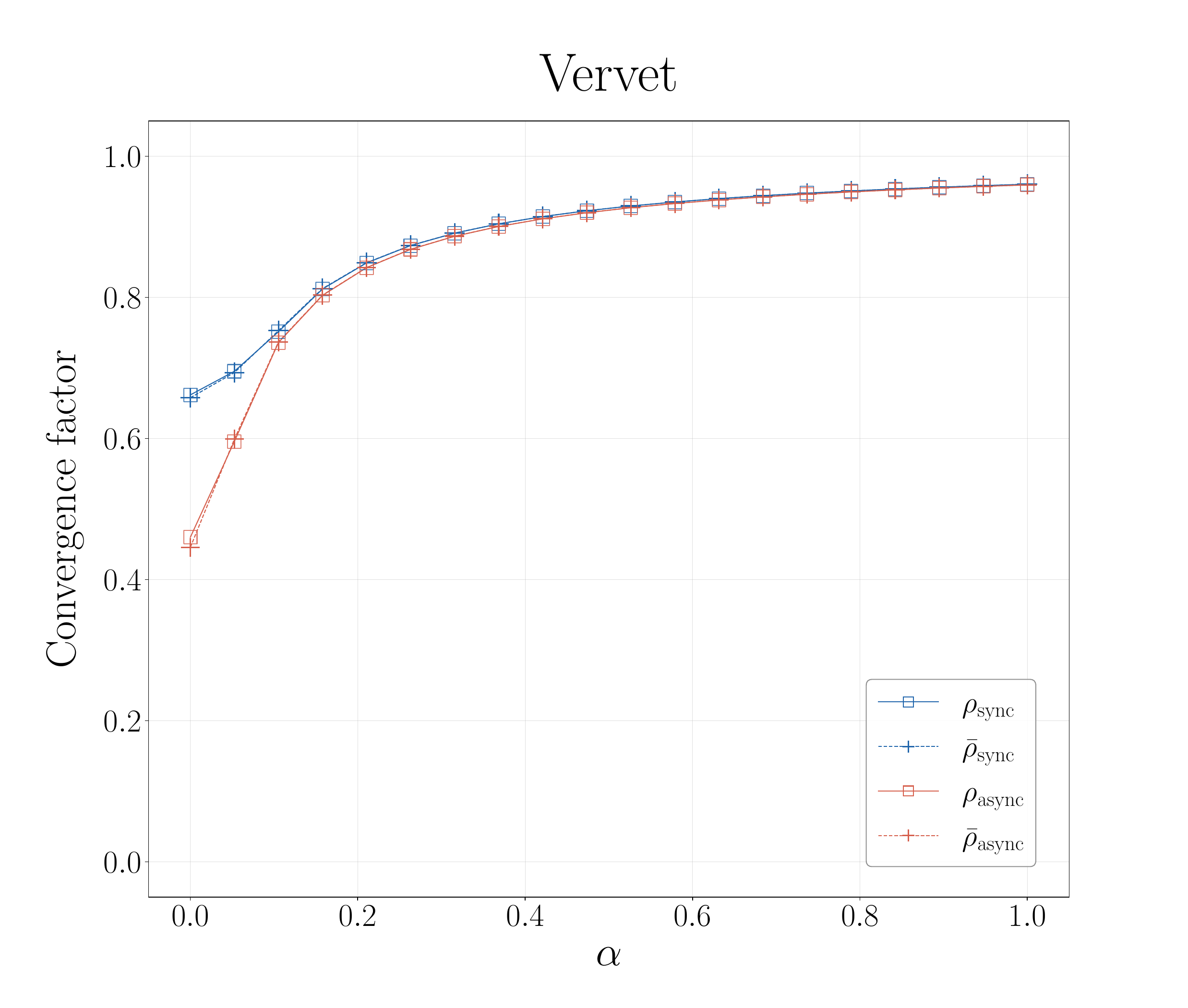}
\includegraphics[width=.32\linewidth, trim={0.1cm 0cm 4cm 1.6cm},clip]{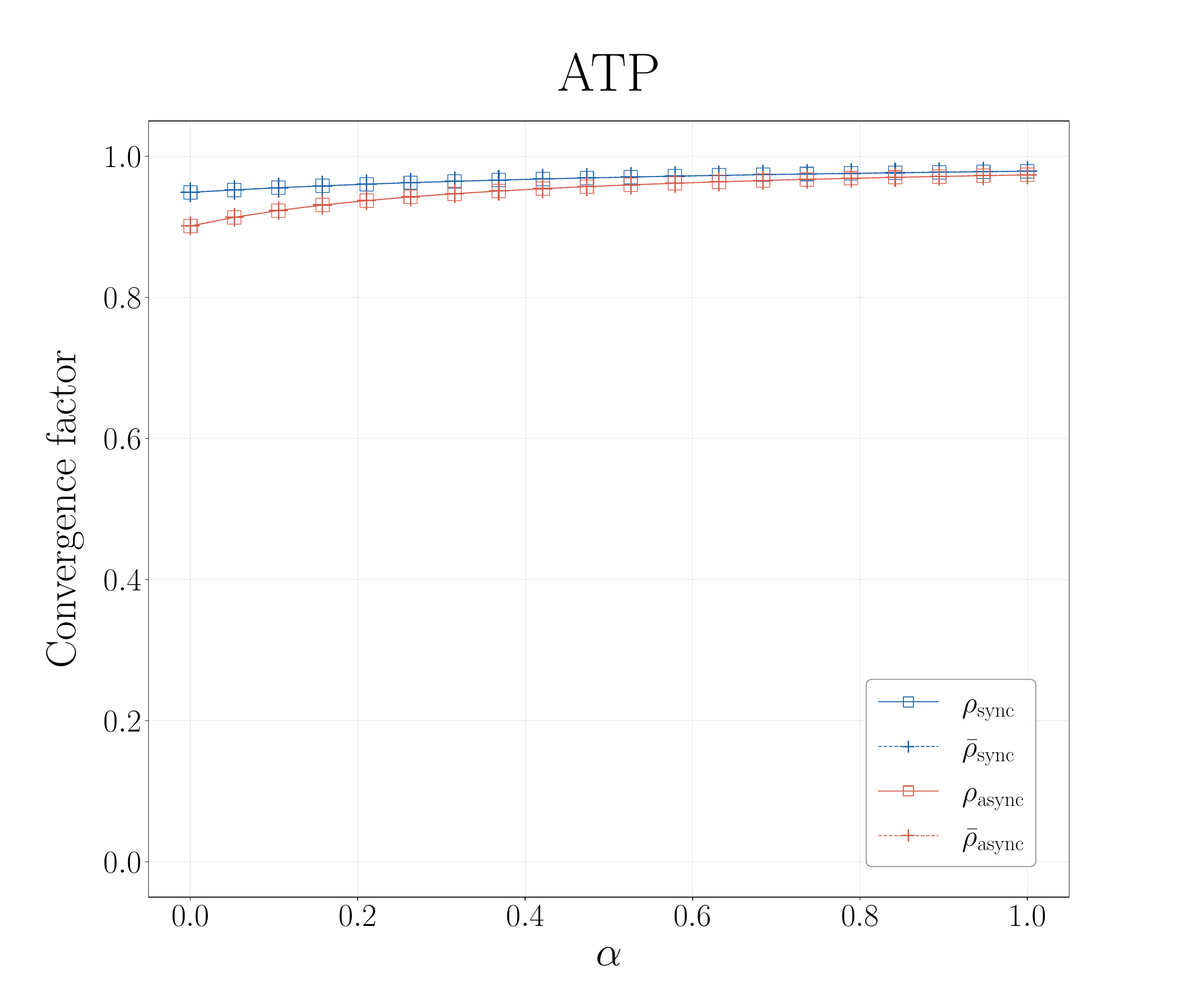}
\includegraphics[width=.32\linewidth, trim={0.1cm 0cm 4cm
1.6cm},clip]{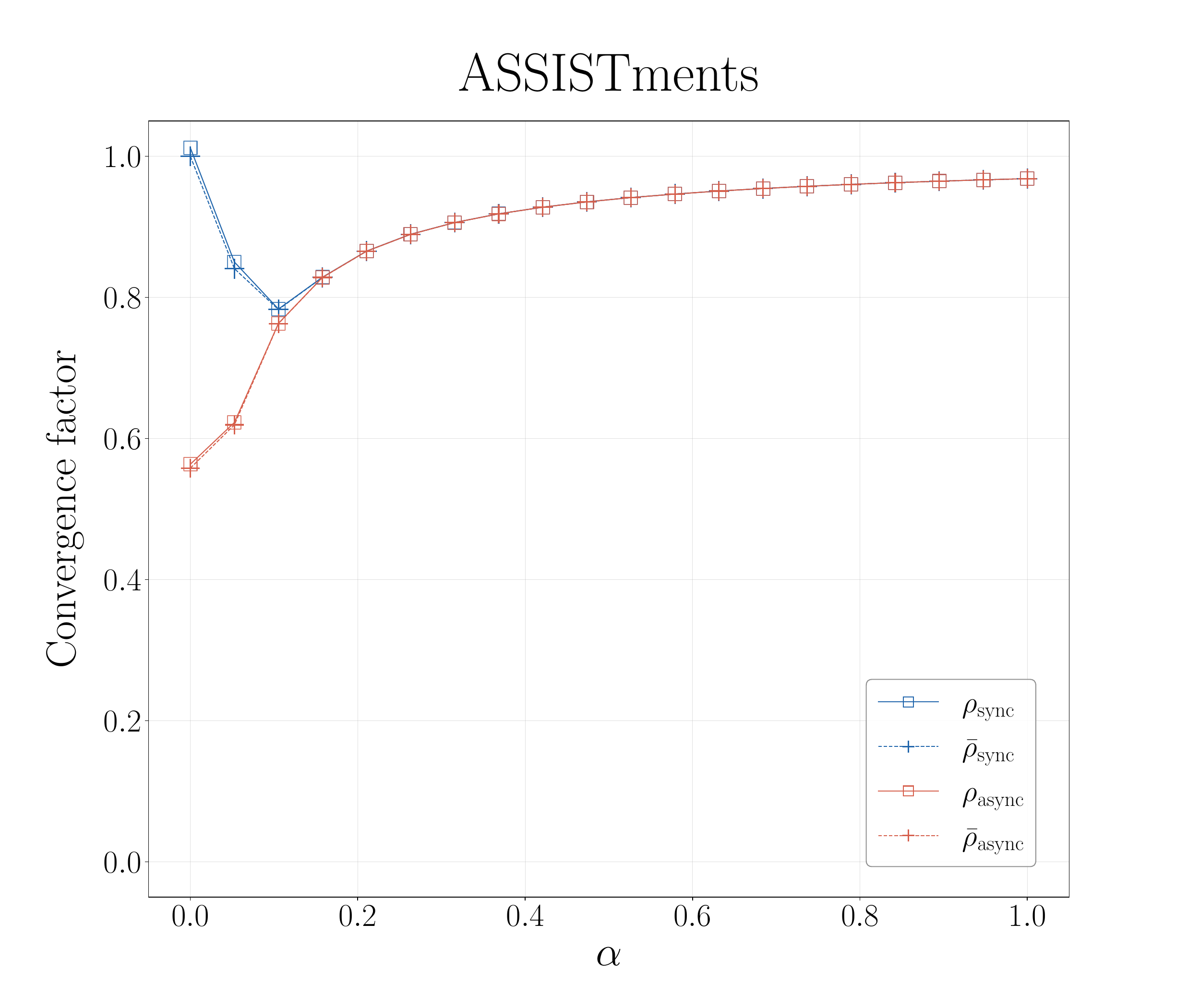}\\
\includegraphics[width=.32\linewidth, trim={0.1cm 0cm 4cm 1.6cm},clip]{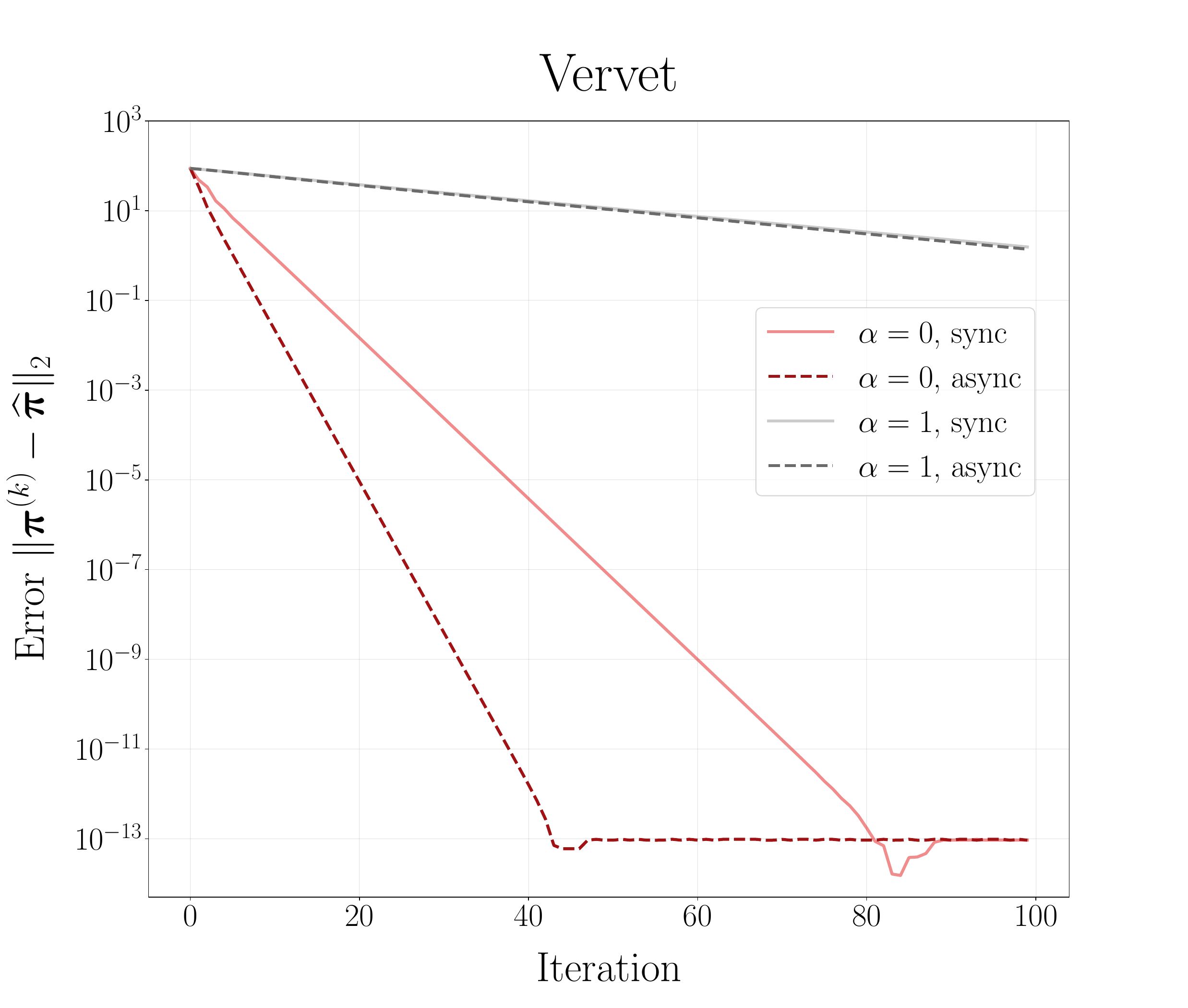}
\includegraphics[width=.32\linewidth, trim={0.1cm 0cm 4cm 1.6cm},clip]{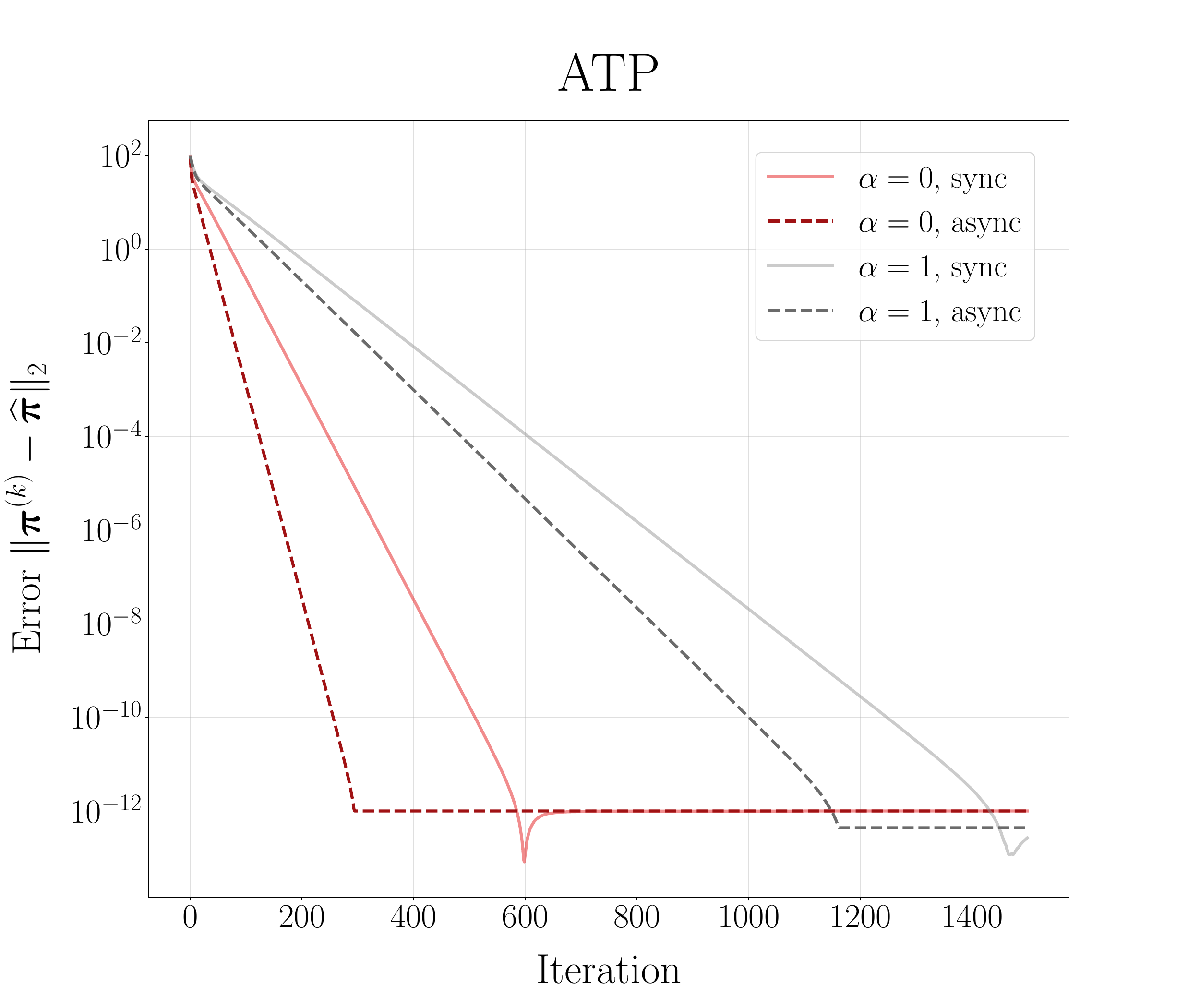}
\includegraphics[width=.32\linewidth, trim={0.1cm 0cm 4cm 1.6cm},clip]{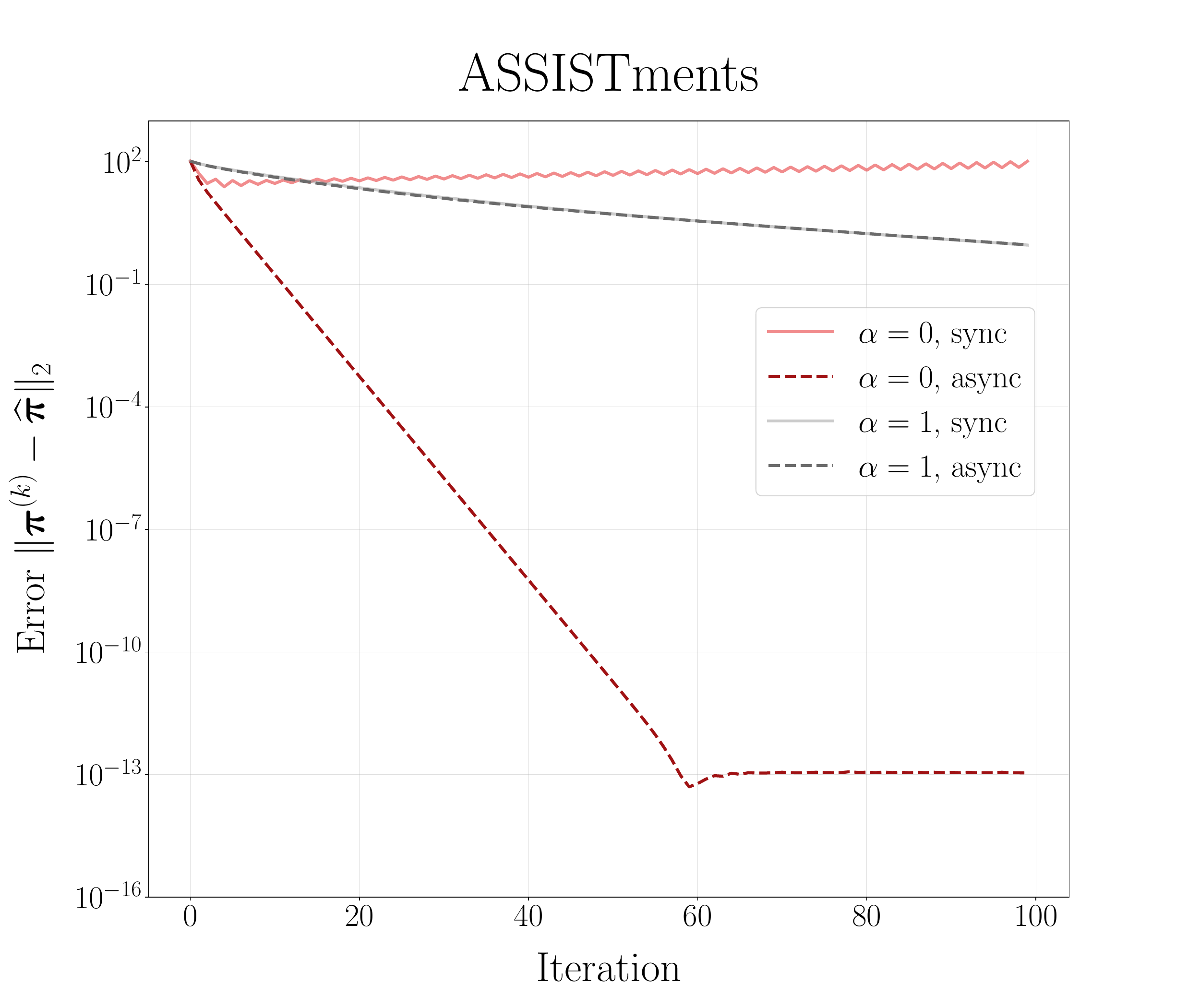}
\caption{
{\bf Top row}:
Observed local convergence factors $\rho$ of Newman's $\alpha$-scheme 
with synchronous and asynchronous updates,
together with the corresponding factors $\bar{\rho}$ predicted by the estimated BT model, for the Vervet Monkeys dataset (left column), the ATP dataset (middle column), and the ASSISTments dataset (right column). 
{\bf Bottom row}:
Convergence histories of the algorithms with $\alpha=0$ and $\alpha=1$.
} \label{fig:3}
\end{figure}

We also observe that, consistent with the synthetic experiments, asynchronous updates become increasingly effective and worthwhile as $\alpha \to 0$. For the Vervet Monkeys dataset, the reduction in convergence factors of asynchronous over synchronous update, $\rhoasync(\alpha)/\rhosync(\alpha)$, is about $0.69$ (with the predicted $\brhoasync(\alpha)/\brhosync(\alpha)\approx 0.67$) for Newman's algorithm ($\alpha = 0$), substantially smaller than the mere $0.99$ (and $0.99$) for Zermelo's algorithm ($\alpha=1$). Similarly, for the ATP dataset, the corresponding reductions are $0.94$ (and $0.95$) for $\alpha=0$ and $0.99$ (and $0.99$) for $\alpha=1$; and for the ASSISTments dataset, they are $0.56$ (and $0.56$) for $\alpha=0$ and $1.00$ (and $1.00$) for $\alpha=1$. These results further highlight the practical benefits of asynchronous updates for the implementation of Newman's algorithm.


Another noteworthy observation from~\Cref{fig:3} is that the dependence
of the convergence factors on the parameter $\alpha$ for the Vervet
monkeys, ATP, and ASSISTments datasets closely resembles that of the
homogeneous, clustered, and near-bipartite cases in~\Cref{fig:1}, respectively. Indeed, the Vervet dataset consists of a small,
well-connected comparison graph that is qualitatively homogeneous. The
ATP dataset, by contrast, exhibits clustered structures that may be
partially explained by generational separation among players. For instance, the gap between the \textit{Agassi--Sampras rivalry} of the early 1990s and the rise of the \textit{Big Three} (Djokovic, Federer, Nadal) in the 2000s results in two relatively well-connected sub-populations with sparser cross-era comparisons. The ASSISTments dataset is bipartite due to the item-response mechanism generating the pairwise comparisons.

\section{Conclusion}\label{sec:conc}
This paper studied the local convergence of Newman's $\alpha$-scheme
under both synchronous and asynchronous updates. We derived closed-form expressions for the local convergence factors and analyzed their dependence on the parameter $\alpha$ through the population version of the BT model. For the synchronous update, we showed that the local convergence factor may exceed $1$ when $\alpha<1$, leading to local divergence, and that it is a quasi-convex function in $\alpha$ under the population BT model. In contrast, for the asynchronous update, we proved the local convergence factor is always strictly less than $1$, and that it is provably monotonically increasing in $\alpha$ under the population BT model of bipartite comparison graphs in consistent ordering. These results complement existing convergence analyses of Zermelo and related algorithms for the BT model and, in particular, highlight the importance of the asynchronous update for implementing Newman's algorithm with $\alpha=0$ in order to achieve accelerated convergence over the classical Zermelo's algorithm (with $\alpha=1$). We further established asymptotic results justifying the approximation of the local convergence factors by their population counterparts under the BT model and demonstrated the effectiveness of our theoretical
analysis through numerical experiments.

There are several interesting directions for future research. First, our analysis of the variation of the convergence factors is established for the population version of the BT model. To what extent these results extend to general comparison outcomes (e.g., under specific forms of model misspecification) remains an open question. Second, although extensive numerical experiments suggest that the convergence factor of asynchronous updates is often monotonic in $\alpha$, our current theory establishes this only for population BT models on consistently ordered bipartite graphs. Extending this analysis to more general comparison graphs remains a significant theoretical challenge. A related direction is to understand the impact of the update ordering in the asynchronous implementation for a fixed comparison graph and to design efficient update-ordering strategies. We hope that this work provides a foundation for addressing these questions in future research.

\paragraph{Acknowledgment.} We would like to thank Hengrui Luo for carefully reading the manuscript and for pointing out a few gaps and typos. R. Han was supported by the Hong Kong Research Grants Council (No.~15302225) and the Hong Kong Polytechnic University (P0044617, P0045351). D. Lu was supported in part by NSF Grant DMS-2110731. Y. Xu was supported by start-up funding from the University of Kentucky and the AMS-Simons Travel Grant (No.~3048116562).

\appendix

\section{Proofs in \Cref{sec:app5}}

\subsection{Proof of \Cref{thm:asymp1}}\label{app:001}
We use $i\sim j$ to denote that $i, j$ are compared, i.e., $m_{ij}>0$. 
For the existence of the MLE, a sufficient condition is to require $\min_{i\sim j}\{w_{ij}, w_{ji}\}>0$, which, combined with the connectivity of $\G([n], \M)$, implies the strong connectivity property. It follows from a union bound that the probability of this event is at least 
\begin{align}\label{mlewq1}
\P\left(\min_{i\sim j}\{w_{ij}, w_{ji}\}>0\right)&\geq 1- \sum_{i\sim j}\left(\P(w_{ij}=0) + \P(w_{ji} = 0)\right)\\
& = 1-\sum_{i\sim j}\left[\left(\frac{\pi^*_i}{\pi^*_i+\pi^*_j}\right)^{m_{ij}} + \left(\frac{\pi^*_j}{\pi^*_i+\pi^*_j}\right)^{m_{ij}}\right]\nonumber\\
& \geq 1 - n(n-1)\left(\frac{\kappa}{1+\kappa}\right)^L\xrightarrow{L\to\infty} 1, \nonumber
\end{align}
where $\kappa$ is the dynamic range of $\bm\pi^*$. 

For the consistency of the MLE, we apply a similar chaining estimate as in \cite{han2023general}. In this case, it is more convenient to use $u_i=\log\pi_i$ and write the log-likelihood as
\begin{align*}
l(\bm u) = \sum_{i\sim j}\left[w_{ij}\log\left(\frac{e^{u_i}}{e^{u_i}+e^{u_j}}\right) + w_{ji}\log\left(\frac{e^{u_j}}{e^{u_i}+e^{u_j}}\right)\right], 
\end{align*} 
with constraint $\sum_{i}u_i = 0$. For $S, T\subseteq [n]$, let $m_{S, T} = \sum_{i\in S}\sum_{j\in T}m_{ij}$ and let $\partial S$ denote the boundary of $S$ in $\G([n], \M)$. It follows from the direct computation that   
\begin{align*}
\sum_{i\in S}\partial_i l(\bm u^*) = \sum_{i\in S}\sum_{j\sim i} \left(w_{ij} - \frac{m_{ij} e^{u^*_i}}{e^{u^*_i}+e^{u^*_j}}\right) = \sum_{i\in S}\sum_{j\in\partial S}\left(w_{ij} - \frac{m_{ij} e^{u^*_i}}{e^{u^*_i}+e^{u^*_j}}\right),
\end{align*}
where we used the that if both $i, j\in S$ and $i\sim j$ then the associated summands do not contribute. 
The (distribution of the) sum is $m_{S, \partial S}$ independent centered Bernoulli random variables. Applying Hoeffding's inequality and a union bound over $S$ yields that, with probability at least $1-(2/L)^n$, 
\begin{align}
\sum_{i\in S}\partial_i l(\bm u^*)\leq c_1\sqrt{n\cdot m_{S, \partial S}\log L}\quad\quad\text{for all $S\subset [n]$}, \label{mlewq2}
\end{align}
where $c_1>0$ is some absolute constant. 

We now condition on the intersection of the probabilistic events in \eqref{mlewq1}-\eqref{mlewq2}. Under such circumstances, the MLE $\widehat{\bm u}$ exists and is finite, and satisfies $\partial_i l(\widehat{\bm u}) = 0$ for all $i\in [n]$. Consequently, 
\begin{align}
\sum_{i\in S}\partial_i l(\bm u^*)-\sum_{i\in S}\partial_i l(\widehat{\bm u})\leq c_1\sqrt{n\cdot m_{S, \partial S}\log L}\quad\quad\text{for all $S\subset [n]$}.\label{mlewq3}
\end{align}
We now construct a particular sequence of sets that allows us to estimate $\|\widehat{\bm u}-\bm u^*\|_\infty$. Let 
\begin{align}
\beta^+\in\argmax_i (\widehat{u}_i - u_i^*), \quad\quad \beta^-\in\argmin_i (\widehat{u}_i - u_i^*), 
\end{align}
and define $B_k  = \{j\in [n]: (\widehat{u}_{\beta^+} - u_{\beta^+}^*)-(\widehat{u}_j - u_j^*)\leq (k-1)\tau\}$ (i.e., $\beta^+\in B_1$), where $c_2 = (1+2\kappa^2)^{-1}$ and $\tau = (2c_1/c_2)\sqrt{n\log L/L}$. We may assume $L$ is large enough so that $\tau<1/2$. 

We claim that $|B_k|$ is strictly increasing in $k$ if $|B_k|<n$, which implies that $\beta^-\in B_n$. The desired result follows by noting 
\begin{align*}
\|\widehat{\bm u}-\bm u^*\|_\infty\leq (\widehat{u}_{\beta^+} - u_{\beta^+}^*) - (\widehat{u}_{\beta^-} - u_{\beta^-}^*)\leq (n-1)\tau\xrightarrow{L\to\infty} 0. 
\end{align*}
It remains to verify the claim. We argue by contradiction. Suppose $|B_k| = |B_{k+1}|<n$ for some $k$. Then for $i\in B_k$ and $j\in \partial B_k\subseteq B^\complement_{k+1}$, $\widehat{u}_i-\widehat{u}_j\geq u_i^*-u_j^* + \tau$. This allows us to lower bound \eqref{mlewq3} as 
\begin{align*}
\sum_{i\in B_k}\partial_i l(\bm u^*)-\sum_{i\in B_k}\partial_i l(\widehat{\bm u}) &= \sum_{i\in B_k}\sum_{j\in\partial B_k}m_{ij}\left(\frac{1}{1+e^{-(\widehat{u}_i-\widehat{u}_j)}}- \frac{1}{1+e^{-(u^*_i-u^*_j)}}\right)\\
&\geq \sum_{i\in B_k}\sum_{j\in\partial B_{k}}m_{ij}\left(\frac{1}{1+e^{-(u^*_i-u^*_j + \tau)}}- \frac{1}{1+e^{-(u^*_i-u^*_j)}}\right)\\
&\geq c_2 m_{B_k, \partial B_{k}}\tau, 
\end{align*} 
where the last step used the mean-value theorem for $\tau<1/2$. Combining this with \eqref{mlewq3} yields 
\begin{align*}
c_2 m_{B_k, \partial B_{k}}\tau\leq c_1\sqrt{n\cdot m_{B_k, \partial B_k}\log L}\implies \frac{2c_1}{c_2}\sqrt{\frac{n\log L}{L}}\leq\tau\leq \frac{c_1}{c_2}\sqrt{\frac{n\log L}{m_{B_k, \partial B_k}}}\leq \frac{c_1}{c_2}\sqrt{\frac{n\log L}{L}},
\end{align*}
where the last step needs $|B_k|<n$. But this is a contradiction. Hence, the proof is complete. 

\subsection{Proof of \Cref{thm:asymp}}\label{app:002}

\paragraph{Proof of part (a)}
Since $n$ is fixed, it suffices to show $[\bm{J}_{F_\alpha}(\pih)]_{ij}-[\JJ_{F_\alpha}(\bm\pi^*)]_{ij}\xrightarrow{\P} 0$ as $L\to\infty$ for each $i, j\in [n]$. 
Fixing $i\neq j$ and dividing the numerator and denominator in each term of $[\bm J_{F_\alpha}(\bm\pi)]_{ij}$ by $m_i$ and $m_i^2$, respectively, we obtain
\begin{align}
[\bm J_{F_\alpha}(\bm\pi)]_{ij}=
\frac{\frac{1}{m_i}w_{ij} \frac{\pi_i(1 - \alpha)}{(\pi_i + \pi_j)^2} }{\frac{1}{m_i}\sum_{\ell} \frac{\alpha w_{i\ell} + w_{\ell i}}{\pi_i + \pi_\ell} } + \frac{\left(\frac{1}{m_i}\sum_{\ell} w_{i\ell} \frac{\alpha \pi_i + \pi_\ell}{\pi_i + \pi_\ell}\right)\frac{1}{m_i}\frac{\alpha w_{ij} + w_{ji}}{(\pi_i + \pi_j)^2}}{\left(\frac{1}{m_i}\sum_{\ell} \frac{\alpha w_{i\ell} + w_{\ell i}}{\pi_i + \pi_\ell}\right)^2}.\label{tempn}
\end{align}
The regularity of $[\bm J_{F_\alpha}(\bm\pi)]_{ij}$ in general may have an arbitrary dependence on the comparison results $\W=\{w_{k\ell}\}$. However, as $L\to\infty$, the weak law of large numbers implies that the ratio $w_{k\ell}/\bar{w}_{k\ell}$ converges in probability to $1$ (with the convention that $0/0=1$), where $\bar{w}_{k\ell}$ is the expected value given in \eqref{eq:barw0}. Because $n$ is fixed, a union bound over all $1\leq k, \ell\leq n$ ensures that as $L\to\infty$, with probability tending to one, the following bounds hold simultaneously:
\begin{align}
\frac{1}{2}\bar{w}_{k\ell}\leq w_{k\ell}\leq 2\bar{w}_{k\ell},\quad\quad \bar{w}_{k\ell} = \frac{m_{k\ell}\pi^*_k}{\pi^*_k+\pi^*_\ell}. \label{myW7}
\end{align} 

We now condition on the high-probability event in \eqref{myW7}. Let $\mathcal B\coloneqq\{\bm\pi\in\R^n_+: (2\kappa)^{-1}\leq\min_{i}\pi_i\leq\max_{i}\pi_i\leq 2\kappa\}$ be a closed neighborhood around the true parameter $\bm\pi^*$, where $\kappa$ is the dynamic range of $\bm\pi^*$. Within $\mathcal B$, the components of $\bm\pi$ are strictly bounded away from zero and infinity. Since $m_i = \sum_j (w_{ij}+w_{ij})$, the event \eqref{myW7} guarantees that the denominators of both terms on the right-hand side of \eqref{tempn} are strictly positive and uniformly bounded away from zero, with the bound depending only on $\kappa$, $\alpha$, and $n$. Since $[\bm J_{F_\alpha}(\bm\pi)]_{ij}$ is the sum of two continuously differentiable rational functions of $\bm\pi$ with non-vanishing denominators on the compact set $\mathcal{B}$, its gradient is uniformly bounded. By the mean-value theorem, it is locally Lipschitz continuous with respect to the $\|\cdot\|_\infty$ norm, with a Lipschitz constant depending only on $\kappa$, $\alpha$, and $n$. Consequently, by the triangle inequality, 
\begin{align}
|[\bm{J}_{F_\alpha}(\pih)]_{ij}-[\JJ_{F_\alpha}(\bm\pi^*)]_{ij}|&\leq |[\bm{J}_{F_\alpha}(\pih)]_{ij}-[\bm{J}_{F_\alpha}(\bm\pi^*)]_{ij}| + |[\bm{J}_{F_\alpha}(\bm\pi^*)]_{ij}-[\JJ_{F_\alpha}(\bm\pi^*)]_{ij}|\nonumber\\
&\lesssim \|\pih-\bm\pi^*\|_\infty +  |[\bm{J}_{F_\alpha}(\bm\pi^*)]_{ij}-[\JJ_{F_\alpha}(\bm\pi^*)]_{ij}|, \label{ij-bdd}
\end{align}
where the implicit constant in $\lesssim$ depends on $\kappa$, $\alpha$, and $n$.  

For the first term, \Cref{thm:asymp1} implies that $\|\pih-\bm\pi^*\|_\infty\xrightarrow{\P} 0$ as $L\to\infty$. 
For the second term, note that  
\begin{align*}
[\JJ_{F_\alpha}(\bm\pi^*)]_{ij} = \frac{\E\left[\frac{1}{m_i}w_{ij} \frac{\pi_i(1 - \alpha)}{(\pi^*_i + \pi^*_j)^2}\right]}{\E\left[\frac{1}{m_i}\sum_{\ell} \frac{\alpha w_{i\ell} + w_{\ell i}}{\pi^*_i + \pi^*_\ell}\right]} + \frac{\E\left[\frac{1}{m_i}\sum_{\ell} w_{i\ell} \frac{\alpha \pi^*_i + \pi^*_\ell}{\pi^*_i + \pi^*_\ell}\right]\E\left[\frac{1}{m_i}\frac{\alpha w_{ij} + w_{ji}}{(\pi^*_i + \pi^*_j)^2}\right]}{\E\left[\frac{1}{m_i}\sum_{\ell} \frac{\alpha w_{i\ell} + w_{\ell i}}{\pi^*_i + \pi^*_\ell}\right]^2}.
\end{align*}
It follows from the law of large numbers and the continuous mapping theorem \cite[Theorem 2.7]{billingsley2013convergence} that $|[\bm{J}_{F_\alpha}(\bm\pi^*)]_{ij}-[\JJ_{F_\alpha}(\bm\pi^*)]_{ij}|\xrightarrow{\P} 0$ as $L\to\infty$. Substituting these estimates into \eqref{ij-bdd} yields the desired result. The case $i = j$ can be proved similarly. 

\paragraph{Proof of part (b)} 

Part~\ref{0022} follows from Part~\ref{0021} and the continuity of eigenvalues of matrices. However, since  $\JJ_{F_\alpha}(\bm\pi^*) $ depends on $\M$ and varies with $L$, the eigenvalue continuity requires careful treatment.
We address this using a classical matrix eigenvalue perturbation 
theory due to Ostrowski and Elsner (see, e.g.,~\cite[Section IV. 1]{Stewart:1990}):
For $\bm A, \bm B\in\mathbb C^{n\times n}$, it holds 
\begin{equation}\label{eq:elsner}
\mbox{md}(\bm A, \bm B)\leq (2n-1) (\|\bm A\|_2+\|\bm B\|_2)^{1-\frac{1}{n}}\|\bm A-\bm B\|_2^{\frac{1}{n}},
\end{equation}
where $\mbox{md}(\bm A, \bm B)$ denotes  the \emph{matching distance} (a metric) between the eigenvalues of $\bm A$ and $\bm B$:
\begin{equation}\label{eq:md}
\mbox{md}(\bm A, \bm B):= \min_{\sigma\in S_n} \left\{\max_{i=1,\dots,n} |\lambda_{\sigma(i)} - \mu_i|\right\},
\end{equation}
with $S_n$ the set of all permutations of $[n]$, 
and  $\{\lambda_i\}$,  $\{\mu_i\}$ the eigenvalues of $\bm A$ and $\bm B$, respectively.
Observe that  $\mbox{md}(\bm A, \bm B)\to 0$ implies a one-to-one convergence (up to permutation) of 
the eigenvalues of $\bm A$ to those of $\bm B$.

Recalling from~\eqref{Jbar} that all entries of $\JJ_{F_\alpha}(\bm\pi^*)$ are bounded by the dynamic range $\kappa\geq 1$ in magnitude, 
hence, 
\begin{equation}\label{eq:boundjjf}
\|\JJ_{F_\alpha}(\bm\pi^*)\|_2\leq \|\JJ_{F_\alpha}(\bm\pi^*)\|_F \leq  \kappa n.
\end{equation}
Combining this with the elementary inequality 
$\|\bm A\|_2+\|\bm B\|_2\leq 2\|\bm B\|_2+\|\bm A-\bm B\|_2$,
the perturbation bound~\eqref{eq:elsner} yields 
\[
\mbox{md}\left(\bm J_{F_\alpha}(\pih),\JJ_{F_\alpha}(\bm\pi^*)\right)
\leq (2n-1)\left (2\kappa n +\| \bm J_{F_\alpha}(\pih)-\JJ_{F_\alpha}(\bm\pi^*)\|_2\right)^{1-\frac{1}{n}}
\cdot 
\| \bm J_{F_\alpha}(\pih)-\JJ_{F_\alpha}(\bm\pi^*)\|_2^{\frac{1}{n}}.
\]
\Cref{0021} therefore implies
\[
\mbox{md}\left(\bm J_{F_\alpha}(\pih),\JJ_{F_\alpha}(\bm\pi^*)\right) \xrightarrow{\P} 0 \quad\text{as $L\to\infty$}.
\]
That is, the eigenvalues of 
$\bm J_{F_\alpha}(\pih)$ are one-to-one matched to those of  $\JJ_{F_\alpha}(\bm\pi^*)$ in probability tends to $1$.
We hence establish~\Cref{0022} by the definitions of convergence factors in~\eqref{eq:rhosync}~and~\eqref{eq:brhosync}.

\subsection{Proof of \Cref{lm:gs}}\label{app:0066}
The proof of the first part follows immediately from \Cref{thm:asymp}. 
The proof of the second part proceeds in complete analogy to that of~\Cref{thm:asymp}.
The only essential difference is that the uniform bound on the Jacobian in~\eqref{eq:boundjjf} now becomes:
\begin{align*}
\|\JJ_{A_\alpha}(\bm\pi^*)\|_2\leq 
\left(\sum_{k=0}^{n-1}\|\JJ_{F_\alpha, \l}(\bm\pi^*)\|_2^k\right)\cdot \|\JJ_{F_\alpha, \d}(\bm\pi^*) + \JJ_{F_\alpha, \u}(\bm\pi^*)\|_2
\leq 
\left(\sum_{k=0}^{n-1}n^k\right)\cdot n
\leq n^{n+1},
\end{align*}
where we used the identity 
\[
\JJ_{A_\alpha}(\bm\pi^*) \equiv (\bm I - \JJ_{F_\alpha, \l}(\bm\pi^*))^{-1}(\JJ_{F_\alpha, \d}(\bm\pi^*) + \JJ_{F_\alpha, \u}(\bm\pi^*)) = 
\left(\sum_{k=0}^{n-1}(\JJ_{F_\alpha, \l}(\bm\pi^*))^k\right)\cdot (\JJ_{F_\alpha, \d}(\bm\pi^*) + \JJ_{F_\alpha, \u}(\bm\pi^*)),
\]
and the fact that
$\|\JJ_{F_\alpha, \l}(\bm\pi^*)\|_2\leq \kappa n$ and $\|\JJ_{F_\alpha, \d}(\bm\pi^*) + \JJ_{F_\alpha, \u}(\bm\pi^*)\|_2\leq \kappa n$,
since all entries of  $\JJ_{F_\alpha}(\bm\pi^*)$ have magnitudes bounded by $\kappa$
(recall the derivation~\eqref{eq:boundjjf}).

\subsection{Proof of \Cref{thm:long}}\label{sec:proof}

For convenience, we use $i\sim j$ to represent that $i$ and $j$ are compared, i.e., $m_{ij} = L$.  
For $i\sim j$ and $s\in [L]$, we denote by $z_{ijs}\in\{0, 1\}$ the $s$th comparison outcome between $i$ and $j$ ($z_{ijs}=1$ if $i$ defeats $j$), i.e., $z_{ijs} = 1-z_{jis}$ and $w_{ij} = \sum_{s}z_{ijs}$. Let $d_i$ denote the number of distinct objects $i$ is compared to. We let $d_{\min} = \min_i d_i$, and $d_{\max} = \max_i d_i$. Since $\sup_n\kappa_n\leq \kappa$ and $\prod_i\pi_i^*=1$, $\kappa^{-1}\leq\min_{i}\pi_i^*\leq\max_{i}\pi_i^*\leq \kappa$. We also let $\alpha_+=\max\{\alpha, 1\}$. 

In the following analysis, we use $C_i$ ($i\geq 1$) to denote absolute constants that depend only on $\kappa$. 
For $i\in [n]$, we define 
\begin{align}
\vartheta_{\max} \coloneqq \max_i\vartheta_i, \quad\quad \text{where}\ \vartheta_i\coloneqq \sum_{j\sim i}d_j^{-1} + d_i\sum_{j\sim i}d_j^{-2}. \label{mytheta}
\end{align}

We begin by rewriting the Jacobian matrix $\bm J_{F_\alpha}(\bm\pi)$ as $\bm J_{F_\alpha}(\bm\pi) = \bm J_{F_\alpha}^{(1)}(\bm\pi) + \bm J_{F_\alpha}^{(2)}(\bm\pi)$, where 
{\small\begin{equation}
\begin{aligned}
&[\bm J^{(1)}_{F_\alpha}(\bm\pi)]_{ij}= \begin{cases}
\frac{\frac{1}{\alpha_+d_iL}w_{ij} \frac{\pi_i(1 - \alpha)}{(\pi_i + \pi_j)^2}}{\frac{1}{\alpha_+d_iL}\sum_{\ell} \frac{\alpha w_{i\ell} + w_{\ell i}}{\pi_i + \pi_\ell} }  & j\neq i\\
\frac{\frac{1}{\alpha_+d_iL}\sum_{\ell} \frac{w_{i\ell} \pi_\ell(\alpha - 1)}{(\pi_i + \pi_\ell)^2} }{\frac{1}{\alpha_+d_iL}\sum_{\ell} \frac{\alpha w_{i\ell} + w_{\ell i}}{\pi_i + \pi_\ell} }& j = i
\end{cases}, & [\bm J^{(2)}_{F_\alpha}(\bm\pi)]_{ij}= \begin{cases}
\frac{\frac{1}{\alpha_+d_iL}\left(\sum_{\ell} w_{i\ell} \frac{\alpha \pi_i + \pi_\ell}{\pi_i + \pi_\ell}\right)\frac{1}{\alpha_+d_iL}\frac{\alpha w_{ij} + w_{ji}}{(\pi_i + \pi_j)^2}}{\left(\frac{1}{\alpha_+d_iL}\sum_{\ell} \frac{\alpha w_{i\ell} + w_{\ell i}}{\pi_i + \pi_\ell}\right)^2} & j\neq i\\
\frac{\frac{1}{\alpha_+d_iL}\left( \sum_{\ell} w_{i\ell} \frac{\alpha \pi_i + \pi_\ell}{\pi_i + \pi_\ell} \right) \frac{1}{\alpha_+d_iL}\left(\sum_{\ell} \frac{\alpha w_{i\ell} + w_{\ell i}}{(\pi_i + \pi_\ell)^2} \right)}{\left(\frac{1}{\alpha_+d_iL}\sum_{\ell} \frac{\alpha w_{i\ell} + w_{\ell i}}{\pi_i + \pi_\ell} \right)^2} & j = i
\end{cases}.
\end{aligned}
\end{equation}}
Here we normalize each factor in both numerators and denominators by $\alpha_+d_iL$. We leverage the joint scaling of $n$ and $L$ to obtain a quantitative approximation using concentration inequalities. 

Since the randomness in $\bm J_{F_\alpha}(\bm\pi)$ is nonlinearly coupled, a direct concentration argument is difficult. The first step is to replace the well-concentrated terms by their limits and introduce intermediate objects that are more amenable to analysis. Let 
\begin{equation}\label{myS12}
\begin{aligned}
S^{(1)}_{i}(\bm\pi) &= \frac{1}{\alpha_+d_iL}\sum_{\ell} \frac{\alpha w_{i\ell} + w_{\ell i}}{\pi_i + \pi_\ell} = \frac{1}{\alpha_+d_iL}\sum_{\ell}\sum_{s} \frac{\alpha z_{i\ell s} + z_{\ell i s}}{\pi_i + \pi_\ell},\\
S^{(2)}_i(\bm\pi) &= \frac{1}{\alpha_+d_iL} \sum_{\ell} w_{i\ell} \frac{\alpha \pi_i + \pi_\ell}{\pi_i + \pi_\ell} = \frac{1}{\alpha_+d_iL} \sum_{\ell}\sum_{s} z_{i\ell s} \frac{\alpha \pi_i + \pi_\ell}{\pi_i + \pi_\ell}.
\end{aligned}
\end{equation}
Both $S^{(1)}_i(\bm\pi^*)$ and $S^{(2)}_i(\bm\pi^*)$ are sample averages of $d_iL$ independent random variables bounded by $\kappa\geq 1$ (as ensured by $\alpha_+$ in the normalization factor) and are thus concentrated around their means with high probability. Fixing $\e>0$ and applying the Chernoff bound together with a union bound over $i$ yields that, with probability at least $1-2n\exp\{-d_{\min}L\e^2/(3\kappa)\}$, 
\begin{align}
\left|S^{(1)}_i(\bm\pi^*)-\eta_i(\bm\pi^*)\right|\leq \e\eta_i(\bm\pi^*), \quad\quad \left|S^{(2)}_i(\bm\pi^*)-\zeta_i(\bm\pi^*)\right|\leq \e\zeta_i(\bm\pi^*)\quad\quad\text{for}\ i = 1, \ldots, n,  \label{sumc}
\end{align} 
where
\begin{equation}\label{myeta12} 
\begin{aligned}
\eta_i(\bm\pi)&\coloneqq \E\left[S^{(1)}_i(\bm\pi) \right] = \frac{1}{\alpha_+d_i}\sum_{\ell\sim i} \frac{\alpha \pi^*_{i} + \pi^*_{\ell}}{(\pi_i + \pi_\ell)(\pi^*_i + \pi^*_\ell)}, \\
\zeta_i(\bm\pi)&\coloneqq \E\left[S^{(2)}_i(\bm\pi)\right] = \frac{1}{\alpha_+d_i}\sum_{\ell\sim i} \frac{\pi^*_{i}(\alpha \pi_{i} + \pi_{\ell})}{(\pi_i + \pi_\ell)(\pi^*_i + \pi^*_\ell)}.
\end{aligned}
\end{equation}
Setting $\bm\pi = \bm\pi^*$, both $\eta_i(\bm\pi^*)$ and $\zeta_i(\bm\pi^*)$ are uniformly bounded below and above by 
an absolute constant depending only on $\kappa$. Moreover, $\zeta_i(\bm\pi^*) = \pi_i^* \eta_i(\bm\pi^*)$.  

Now we approximate $\bm J_{F_\alpha}(\bm\pi)$ by $\widetilde{\bm J}_{F_\alpha}(\bm\pi)\coloneqq\widetilde{\bm J}_{F_\alpha}^{(1)}(\bm\pi)+\widetilde{\bm J}_{F_\alpha}^{(2)}(\bm\pi)$, where
{\small\begin{align}\label{tildeform}
&[\widetilde{\bm J}^{(1)}_{F_\alpha}(\bm\pi)]_{ij}= \begin{cases}
\frac{\frac{1}{\alpha_+d_iL}w_{ij} \frac{\pi_i(1 - \alpha)}{(\pi_i + \pi_j)^2}}{\eta_i(\bm\pi)}  & j\neq i\\
\frac{\frac{1}{\alpha_+d_iL}\sum_{\ell} \frac{w_{i\ell} \pi_\ell(\alpha - 1)}{(\pi_i + \pi_\ell)^2} }{\eta_i(\bm\pi)}& j = i
\end{cases}, & [\widetilde{\bm J}^{(2)}_{F_\alpha}(\bm\pi)]_{ij}= \begin{cases}
\frac{\zeta_i(\bm\pi)\frac{1}{\alpha_+d_iL}\frac{\alpha w_{ij} + w_{ji}}{(\pi_i + \pi_j)^2}}{\eta^2_i(\bm\pi)} & j\neq i\\
\frac{\zeta_i(\bm\pi) \frac{1}{\alpha_+d_iL}\left(\sum_{\ell} \frac{\alpha w_{i\ell} + w_{\ell i}}{(\pi_i + \pi_\ell)^2} \right)}{\eta^2_i(\bm\pi)} & j = i
\end{cases}.
\end{align}}

Compared to $\bm J_{F_\alpha}^{(i)}(\bm\pi^*)$, $\widetilde{\bm J}_{F_\alpha}^{(i)}(\bm\pi^*)$ are sums of independent bounded random matrices, which can be shown to concentrate around their means under the spectral norm using the matrix Bernstein inequality.  
\begin{lemma}\label{lm:t1}
For $0<\e<1$, with probability at least $1-4n\exp\left\{-d_{\min}L\e^2/(C_1 \vt_{\max})\right\}$, 
\begin{align}
\left\|\widetilde{\bm J}_{F_\alpha}^{(i)}(\bm\pi^*)-\E[\widetilde{\bm J}_{F_\alpha}^{(i)}(\bm\pi^*)]\right\|_2\leq \e\quad\quad \mathrm{for}\ i = 1, 2. \label{ding}
\end{align}  
\end{lemma}
\begin{proof}
See \Cref{lpf1}. 
\end{proof}
On the other hand, a direct computation shows that $\E[\widetilde{\bm J}_{F_\alpha}(\bm\pi^*)] = \JJ_{F_\alpha}(\bm\pi^*)$. 
Combining this with \eqref{sumc} and \Cref{lm:t1}, we quantify the distance between $\bm J_{F_\alpha}(\bm\pi^*)$ and $\JJ_{F_\alpha}(\bm\pi^*)$. 
\begin{lemma}\label{lm:trig}
For $0<\e<1$, with probability at least $1-6n\exp\left\{-d_{\min}L\e^2/(C_1 \vt_{\max})\right\}$, 
\begin{align}
\left\|\bm J_{F_\alpha}(\bm\pi^*)-\JJ_{F_\alpha}(\bm\pi^*)\right\|_2\lesssim\sqrt{\vt_{\max}}\e,   
\end{align} 
where the implicit constant depends only on $\kappa$. 
\end{lemma}
\begin{proof}
See \Cref{lpf2}. 
\end{proof}
To quantify the distance between $\bm J_{F_\alpha}(\pih)$ and $\JJ_{F_\alpha}(\bm\pi^*)$, we need a regularity lemma to connect $\bm J_{F_\alpha}(\bm\pi^*)$ to $\bm J_{F_\alpha}(\pih)$. The following result generalizes the Lipschitz property used in the proof for \Cref{thm:asymp} to the large-$n$ setting. 
\begin{lemma}\label{lm:reg}
Let $\mathcal B\coloneqq\{\bm\pi\in\R^n_+: (2\kappa)^{-1}\leq\min_{i}\pi_i\leq\max_{i}\pi_i\leq 2\kappa\}$. With probability at least $1-4n\exp\left\{-d_{\min}L/(12\kappa)\right\}$, for $\bm\pi, \bm\pi'\in\mathcal B$, 
\begin{align*}
\|\bm J_{F_\alpha}(\bm\pi)-\bm J_{F_\alpha}(\bm\pi')\|_2\lesssim\sqrt{\vt_{\max}}\|\bm\pi-\bm\pi'\|_\infty, 
\end{align*}
where the implicit constant depends only on $\kappa$.  
\end{lemma}

\begin{proof}
See \Cref{lpf3}. 
\end{proof}

Combining \Cref{lm:trig}, \Cref{lm:reg}, and the uniform consistency of the MLE in \Cref{thm:asymp2} yields an upper bound on $\|\bm J_{F_\alpha}(\pih)-\JJ_{F_\alpha}(\bm\pi^*)\|_2$.  
\begin{lemma}\label{lm:2est}
If $np_{\min}/(\log n)^3\to\infty$, then for all sufficiently large $n$, with probability at least $1-2n^{-2}$, 
\begin{align*}
\|\bm J_{F_\alpha}(\pih)-\JJ_{F_\alpha}(\bm\pi^*)\|_2\lesssim\sqrt{\frac{(\log n)^3}{np_{\min}L}},  
\end{align*}
where the implicit constant depends only on $\kappa$. 
\end{lemma}

\begin{proof}
See \Cref{lpf4}. 
\end{proof}

To transfer the bound in the spectral norm to the convergence radius, we do not directly apply the classical eigenvalue perturbation bound in \Cref{app:002} because $n$ is diverging. Instead, we relate both $\bm J_{F_\alpha}(\pih)$ and $\JJ_{F_\alpha}(\bm\pi^*)$ to their symmetrized versions via rebalancing and apply eigenvalue perturbation results for symmetric matrices.

Recall the diagonal matrix $\bm D_\alpha(\pih)$ defined in \eqref{myD} and its population counterpart $\bar{\bm D}_\alpha(\bm\pi^*)=\diag(\bm\pi^*)^{1/2}\cdot 
\left(\bar{\bm C}(\bm \pi^*)+\alpha \bar{\bm R}(\bm\pi^*)\right)^{-1/2}$. In particular, 
\begin{align*}
[\bar{\bm D}_\alpha(\bm\pi^*)]_{ii} = \sqrt{\frac{\pi^*_i}{\alpha_+d_iL\eta_i(\bm\pi^*)}}, \quad\quad [\bm D_\alpha(\pih)]_{ii} = \sqrt{\frac{\widehat{\pi}_i}{\alpha_+d_iLS^{(1)}_{i}(\pih)}}, \end{align*}
where $S^{(1)}_{i}$ and $\eta_i$ are defined in \eqref{myS12} and \eqref{myeta12}. 
Let 
\begin{align*}
\bm J^{\sym}_{F_\alpha}(\pih)\coloneqq\bm D_\alpha(\pih)^{-1}\bm J_{F_\alpha}(\pih)\bm D_\alpha(\pih),\quad\quad \bar{\bm J}_{F_\alpha}^{\sym}(\bm\pi^*)\coloneqq\bar{\bm D}_\alpha(\bm\pi^*)^{-1} \JJ_{F_\alpha}(\bm\pi^*)\bar{\bm D}_\alpha(\bm\pi^*)
\end{align*}
be the symmetrized versions of $\bm J_{F_\alpha}(\pih)$ and $\JJ_{F_\alpha}(\bm\pi^*)$, respectively.  The following lemma compares the difference between $\bm D_\alpha(\pih)$ and $\bar{\bm D}_\alpha(\bm\pi^*)$ and bounds their condition numbers.    

\begin{lemma}\label{lm:bfl}
Under the same assumptions in \Cref{lm:2est}, for all sufficiently large $n$, with probability at least $1-2n^{-2}$, 
\begin{align*}
\max\{\mathrm{cond}(\bar{\bm D}_\alpha(\bm\pi^*)), \mathrm{cond}(\bm D_\alpha(\pih))\}\leq 3\kappa^{\frac{3}{2}}, \quad\quad \max_{i}\left|\frac{[\bar{\bm D}_\alpha(\bm\pi^*)]_{ii}}{[\bm D_\alpha(\pih)]_{ii}}-1\right|\lesssim\sqrt{\frac{(\log n)^3}{np_{\min}L}},
\end{align*}
where the implicit constant in the second bound depends only on $\kappa$.    
\end{lemma}

\begin{proof}
See \Cref{lpf5}. 
\end{proof}

We now have all the ingredients to prove \Cref{thm:long}. 

\begin{proof}[Proof of \Cref{thm:long}]
Condition on the events in Lemmas~\ref{lm:2est}-\ref{lm:bfl}, which hold with probability at least $1-4n^{-2}$. Let $\Lambda(\cdot )$ denote the ordered eigenvalues of a matrix (with real eigenvalues). By Weyl's inequality and the triangle inequality, 
\begin{align*}
&\|\Lambda\left(\bm{J}_{F_\alpha}(\pih)\right)- \Lambda\left(\JJ_{F_\alpha}(\bm\pi^*)\right)\|_\infty= \|\Lambda\left(\bm J^{\sym}_{F_\alpha}(\pih)\right)- \Lambda(\bar{\bm J}_{F_\alpha}^{\sym}(\bm\pi^*))\|_\infty\\
\leq&\ \|\bm D_\alpha(\pih)^{-1}\bm J_{F_\alpha}(\pih)\bm D_\alpha(\pih)-\bar{\bm D}_\alpha(\bm\pi^*)^{-1} \JJ_{F_\alpha}(\bm\pi^*)\bar{\bm D}_\alpha(\bm\pi^*)\|_2\\
\leq&\ \|\bm D_\alpha(\pih)^{-1}\bm J_{F_\alpha}(\pih)\bm D_\alpha(\pih)-\bm D_\alpha(\pih)^{-1} \JJ_{F_\alpha}(\bm\pi^*)\bar{\bm D}_\alpha(\bm\pi^*)\|_2\\
&\ + \|\bm D_\alpha(\pih)^{-1} \JJ_{F_\alpha}(\bm\pi^*)\bar{\bm D}_\alpha(\bm\pi^*)-\bar{\bm D}_\alpha(\bm\pi^*)^{-1} \JJ_{F_\alpha}(\bm\pi^*)\bar{\bm D}_\alpha(\bm\pi^*)\|_2\\
\leq&\ \|\bm D_\alpha(\pih)^{-1}\|_2\|\bm J_{F_\alpha}(\pih)-\JJ_{F_\alpha}(\bm\pi^*)\|_2\|\bar{\bm D}_\alpha(\bm\pi^*)\|_2\\
&\ + \|\bm D_\alpha(\pih)^{-1}\|_2\|\JJ_{F_\alpha}(\bm\pi^*)\|_2\|\bm D_\alpha(\pih)-\bar{\bm D}_\alpha(\bm\pi^*)\|_2\\
&\ + \|\bm D_\alpha(\pih)^{-1}\|_2\|\bm J_{F_\alpha}(\pih)-\JJ_{F_\alpha}(\bm\pi^*)\|_2\|\bm D_\alpha(\pih)-\bar{\bm D}_\alpha(\bm\pi^*)\|_2\\
&\  + \|\bm I-\bm D_\alpha(\pih)^{-1}\bar{\bm D}_\alpha(\bm\pi^*)\|_2\|\bar{\bm D}_\alpha(\bm\pi^*)^{-1} \JJ_{F_\alpha}(\bm\pi^*)\bar{\bm D}_\alpha(\bm\pi^*)\|_2\\
\lesssim&\ \kappa^{\frac{3}{2}}\sqrt{\frac{(\log n)^3}{np_{\min}L}},
\end{align*}
where last step follows from Lemmas~\ref{lm:2est}-\ref{lm:bfl} and $\|\bar{\bm J}_{F_\alpha}^{\sym}(\bm\pi^*)\|_2=\lambda_1(\JJ_{F_\alpha}(\bm\pi^*))=1$. 
\end{proof}

\subsection{Proof of \Cref{thm:long1}}

The proof of \Cref{thm:long1} is similar to that of \Cref{thm:long}, so we only highlight the difference and sketch the proof. We will need the following lemma. 
\begin{lemma}\label{lm:lowerbdd}
If $np_{\min}/(\log n)^3\to\infty$, then for all sufficiently large $n$, with probability at least $1-2n^{-2}$, 
\begin{align*}
\|(\bm I - \JJ_{F_\alpha, \l}(\bm\pi^*))^{-1}\|_2\leq 6\kappa^{\frac{3}{2}}.  
\end{align*}
\end{lemma}
\begin{proof}
See \Cref{lpf6}. 
\end{proof}

We condition on the probabilistic event in \Cref{lm:lowerbdd} and recall the definitions of $\bm J_{A_\alpha}(\pih)$ (c.f.~\eqref{J_A_new}) and $\JJ_{A_\alpha}(\bm\pi^*)$ (c.f.~\eqref{eq:jjabar}):
\begin{align}
\bm J_{A_\alpha}(\pih) &= (\bm I-\bm J_{F_\alpha, \l}(\pih))^{-1}(\bm J_{F_\alpha}(\pih)-\bm J_{F_\alpha, \l}(\pih)),\label{K11}\\
\JJ_{A_\alpha}(\bm\pi^*) &= (\bm I-\JJ_{F_\alpha, \l}(\bm\pi^*))^{-1}(\JJ_{F_\alpha}(\bm\pi^*) - \JJ_{F_\alpha, \l}(\bm\pi^*)).\label{K12} 
\end{align}
The difference between $\bm J_{A_\alpha}(\pih)$ and $\JJ_{A_\alpha}(\bm\pi^*)$ can be bounded using matrix perturbation, which is formulated as the following lemma.
\begin{lemma}\label{lm:perturb}
Let $\bm A, \bm B\in\R^{n\times n}$ be two matrices and suppose $\bm A$ is invertible. For any $\delta\bm A, \delta\bm B\in\R^{n\times n}$ with $\|\delta\bm A\|_2\leq 1/(2\|\bm A^{-1}\|_2)$, 
\begin{align*}
\|(\bm A + \delta\bm A)^{-1}(\bm B+\delta\bm B) - \bm A^{-1}\bm B\|_2\leq 2\|\bm{A}^{-1}\|_2 \left( \|\delta\bm B\|_2 + \|\delta\bm A\|_2 \|\bm A^{-1}\|_2\|\bm B\|_2\right).
\end{align*}
\end{lemma}
The proof of this lemma is standard and thus omitted. We now apply \Cref{lm:perturb} with 
\begin{align*}
\bm A &\coloneqq \bm I-\JJ_{F_\alpha, \l}(\bm\pi^*),
&\delta\bm A &\coloneqq \JJ_{F_\alpha, \l}(\bm\pi^*)-\bm J_{F_\alpha, \l}(\pih),\\
\bm B &\coloneqq \JJ_{F_\alpha}(\bm\pi^*) - \JJ_{F_\alpha, \l}(\bm\pi^*),
&\delta\bm B &\coloneqq \JJ_{F_\alpha, \l}(\bm\pi^*)-\bm J_{F_\alpha, \l}(\pih)-(\JJ_{F_\alpha}(\bm\pi^*) - \bm J_{F_\alpha}(\pih)).
\end{align*}
The probabilistic event in \Cref{lm:lowerbdd} ensures
\begin{align}
\|\bm A^{-1}\|_2\leq 6\kappa^{\frac{3}{2}}. \label{aa1}
\end{align}
On the other hand, 
\begin{align}
\|\bm B\|_2\leq\sqrt{\|\bm B\|_1\|\bm B\|_\infty}\leq\sqrt{\|\JJ_{F_\alpha}(\pi^*)\|_1\|\JJ_{F_\alpha}(\pi^*)\|_\infty}\lesssim\sqrt{\vt_{\max}}.  \label{aa2}
\end{align}
To bound the perturbation terms, note that a similar Lipschitz property holds for the spectral norm of $\bm J_{F_\alpha, \l}(\bm\pi)$ for $\bm\pi\in\mathcal B$ in the same high-probability event in \Cref{lm:reg}. Specifically, for all sufficiently large $n$, with probability at least $1-4n\exp\left\{-d_{\min}L/(12\kappa)\right\}$, for all $\bm\pi\in\mathcal B$,  
\begin{align*}
\|\bm J_{F_\alpha, \l}(\bm\pi)- \bm J_{F_\alpha, \l}(\bm\pi^*)\|_2\lesssim\sqrt{\vt_{\max}}\|\bm\pi - \bm\pi^*\|_\infty. 
\end{align*}
By a similar concentration argument as in \Cref{lm:trig}, with probability at least $1-6n\exp\left\{-d_{\min}L\e^2/(C_1\vt_{\max})\right\}$,  
\begin{align*}
\|\bm J_{F_\alpha, \l}(\bm\pi^*)-\JJ_{F_\alpha, \l}(\bm\pi^*)\|_2\lesssim\sqrt{\vt_{\max}}\e. 
\end{align*}
If $np_{\min}/(\log n)^3\to\infty$, taking $\e \gtrsim \sqrt{(\log n)^3/(np_{\min}L)}$ and noting $\vt_{\max}\leq 20$ yields that, with probability at least $1-2n^{-2}$, $\pih\in\mathcal B$ so that 
\begin{align}
\|\delta\bm A\|_2\lesssim\sqrt{\frac{(\log n)^3}{np_{\min}L}}.  \label{aa3}
\end{align} 
This combined with \Cref{lm:2est} shows that, with probability at least $1-4n^{-2}$, 
\begin{align}
\|\delta\bm B\|_2\lesssim\sqrt{\frac{(\log n)^3}{np_{\min}L}}.  \label{aa4}
\end{align} 
Putting the estimates \eqref{aa1}-\eqref{aa4} together and appealing to \Cref{lm:perturb} finishes the proof. 

\subsection{Proofs of auxiliary lemmas}

\subsubsection{Proof of \Cref{lm:t1}}\label{lpf1}
We prove the result for $\widetilde{\bm J}_{F_\alpha}^{(1)}(\bm\pi^*)$. Write
\begin{align*}
\widetilde{\bm J}_{F_\alpha}^{(1)}(\bm\pi^*)-\E[\widetilde{\bm J}_{F_\alpha}^{(1)}(\bm\pi^*)] = \sum_{i\sim j}\sum_{s}(\bm Z^{(ijs)}-\E[\bm Z^{(ijs)}]) \eqqcolon \sum_{i\sim j}\sum_{s}\bar{\bm Z}^{(ijs)}, 
\end{align*}
where $\bm Z^{(ijs)}\in\R^{n\times n}$ are random matrices defined as 
\begin{align*}
[\bm Z^{(ijs)}]_{k\ell} = \begin{cases}
\frac{\frac{1}{\alpha_+d_iL}z_{ijs} \frac{\pi^*_i(1 - \alpha)}{(\pi^*_i + \pi^*_j)^2}}{\eta_i(\bm\pi^*)} & (k, \ell) = (i, j)\\
-\frac{\frac{1}{\alpha_+d_iL}z_{ijs} \frac{\pi^*_j(1 - \alpha)}{(\pi^*_i + \pi^*_j)^2}}{\eta_i(\bm\pi^*)} & (k, \ell) = (i, i)\\
\frac{\frac{1}{\alpha_+d_jL}z_{jis} \frac{\pi^*_j(1 - \alpha)}{(\pi^*_i + \pi^*_j)^2}}{\eta_j(\bm\pi^*)} & (k, \ell) = (j, i)\\
-\frac{\frac{1}{\alpha_+d_jL}z_{jis} \frac{\pi^*_i(1 - \alpha)}{(\pi^*_i + \pi^*_j)^2}}{\eta_j(\bm\pi^*)} & (k, \ell) = (j, j)
\end{cases}.
\end{align*}
Define 
\begin{align*}
R\coloneqq\max_{i\sim j, s}\|\bar{\bm Z}^{(ijs)}\|_2, \quad\sigma^2\coloneqq\max\Bigg\{\left\|\sum_{i\sim j}\sum_{s}\E[\bar{\bm Z}^{(ijs)}(\bar{\bm Z}^{(ijs)})^\top]\right\|_2, \ \left\|\sum_{i\sim j}\sum_{s}\E[(\bar{\bm Z}^{(ijs)})^\top\bar{\bm Z}^{(ijs)}]\right\|_2\Bigg\}. 
\end{align*}
Using the fact that the non-zero entries $|[\bar{\bm Z}^{(ijs)}]_{ij}|=|[\bar{\bm Z}^{(ijs)}]_{ii}|\asymp 1/(d_iL)$ and $|[\bar{\bm Z}^{(ijs)}]_{ji}|=|[\bar{\bm Z}^{(ijs)}]_{jj}|\asymp 1/(d_jL)$, we have $R\lesssim 1/(d_{\min}L)$ and 
\begin{align*}
\sigma^2&\leq\max\Bigg\{\left\|\sum_{i\sim j}\sum_{s}\E[\bar{\bm Z}^{(ijs)}(\bar{\bm Z}^{(ijs)})^\top]\right\|_1, \ \left\|\sum_{i\sim j}\sum_{s}\E[(\bar{\bm Z}^{(ijs)})^\top\bar{\bm Z}^{(ijs)}]\right\|_1\Bigg\}\\
&\lesssim \max\left\{\max_i (d_iL)^{-1}\sum_{j\sim i}(d^{-1}_i + d^{-1}_j), \ \max_i\sum_{j\sim i}[(d^2_iL)^{-1} + (d^2_jL)^{-1}]\right\}\\
&\lesssim \vartheta_{\max}/(d_{\min}L), 
\end{align*} 
where the last step follows from the definition of $\vt_{\max}$ in \eqref{mytheta} and the implicit constant depends only on $\kappa$. 
Now we apply the matrix Bernstein inequality \cite[Theorem 1.6]{tropp2012user} to obtain that, with probability at least $1-2n\exp\left\{-d_{\min}L\e^2/(C_1\vt_{\max})\right\}$, 
\begin{align*}
\left\|\widetilde{\bm J}_{F_\alpha}^{(1)}(\bm\pi^*)-\E[\widetilde{\bm J}_{F_\alpha}^{(1)}(\bm\pi^*)]\right\|_2\leq \e.
\end{align*}
A similar high-probability bound for $\widetilde{\bm J}_{F_\alpha}^{(2)}(\bm\pi^*)$ can be obtained by the same reasoning. The proof is finished by combining the estimates for $\widetilde{\bm J}_{F_\alpha}^{(1)}(\bm\pi^*)$ and $\widetilde{\bm J}_{F_\alpha}^{(2)}(\bm\pi^*)$. 

\subsubsection{Proof of \Cref{lm:trig}}\label{lpf2}
We condition on the events in \eqref{sumc} and \eqref{ding}.
Let 
\begin{align*}
\bm\Lambda = \diag\left(\frac{\eta_i(\bm\pi^*)}{S_i^{(1)}(\bm\pi^*)}\right)\in\R^{n\times n}\quad\quad\bm\Gamma = \diag\left(\frac{\eta^2_i(\bm\pi^*)S_i^{(2)}(\bm\pi^*)}{(S_i^{(1)}(\bm\pi^*))^2\zeta_i(\bm\pi^*)}\right)\in\R^{n\times n}. 
\end{align*}
By the triangle inequality, 
\begin{align}
&\left\|\bm J_{F_\alpha}(\bm\pi^*)-\E[\widetilde{\bm J}_{F_\alpha}(\bm\pi^*)]\right\|_2\nonumber\\
=&\  \left\|\bm\Lambda\widetilde{\bm J}_{F_\alpha}^{(1)}(\bm\pi^*)+\bm\Gamma\widetilde{\bm J}_{F_\alpha}^{(2)}(\bm\pi^*)-\E[\widetilde{\bm J}_{F_\alpha}(\bm\pi^*)]\right\|_2\nonumber\\
\leq&\ \left\|(\bm\Lambda-\bm I)\widetilde{\bm J}_{F_\alpha}^{(1)}(\bm\pi^*)\right\|_2+\left\|(\bm\Gamma-\bm I)\widetilde{\bm J}_{F_\alpha}^{(2)}(\bm\pi^*)\right\|_2+ \left\|\widetilde{\bm J}_{F_\alpha}^{(1)}(\bm\pi^*)-\E[\widetilde{\bm J}_{F_\alpha}^{(1)}(\bm\pi^*)]\right\|_2 + \left\|\widetilde{\bm J}_{F_\alpha}^{(2)}(\bm\pi^*)-\E[\widetilde{\bm J}_{F_\alpha}^{(2)}(\bm\pi^*)]\right\|_2\nonumber\\
\leq&\ \|\bm\Lambda-\bm I\|_2\|\widetilde{\bm J}_{F_\alpha}^{(1)}(\bm\pi^*)\|_2 + \|\bm\Gamma-\bm I\|_2\|\widetilde{\bm J}_{F_\alpha}^{(2)}(\bm\pi^*)\|_2 + 2\e\nonumber\\
\lesssim&\ \e\left(\|\widetilde{\bm J}_{F_\alpha}^{(1)}(\bm\pi^*)\|_2 + \|\widetilde{\bm J}_{F_\alpha}^{(2)}(\bm\pi^*)\|_2 + 1\right),  \label{l20}
\end{align}
where the implicit constant depends only on $\kappa$. Here we used \eqref{sumc} and the fact that $\eta_i(\bm\pi^*)$ and $\zeta_i(\bm\pi^*)$ are uniformly bounded below by an absolute constant depending only on $\kappa$ to control $\|\bm\Lambda-\bm I\|_2$ and $\|\bm\Gamma-\bm I\|_2$, and \eqref{ding} to control $\|\widetilde{\bm J}_{F_\alpha}^{(1)}(\bm\pi^*)-\E[\widetilde{\bm J}_{F_\alpha}^{(1)}(\bm\pi^*)]\|_2$ and $\|\widetilde{\bm J}_{F_\alpha}^{(2)}(\bm\pi^*)-\E[\widetilde{\bm J}_{F_\alpha}^{(2)}(\bm\pi^*)]\|_2$. 

To further bound $\|\widetilde{\bm J}_{F_\alpha}^{(1)}(\bm\pi^*)\|_2$ and $\|\widetilde{\bm J}_{F_\alpha}^{(2)}(\bm\pi^*)\|_2$, we will use the Schur bound 
\begin{align}
\|\widetilde{\bm J}_{F_\alpha}^{(i)}(\bm\pi^*)\|_2\leq\sqrt{\|\widetilde{\bm J}_{F_\alpha}^{(i)}(\bm\pi^*)\|_1\|\widetilde{\bm J}_{F_\alpha}^{(i)}(\bm\pi^*)\|_\infty}. \label{l21}
\end{align}
By a direct computation based on the formulas in \eqref{tildeform}, the nonzero off-diagonal and diagonal entries of $\widetilde{\bm J}_{F_\alpha}^{(1)}(\bm\pi^*)$ and $\widetilde{\bm J}_{F_\alpha}^{(2)}(\bm\pi^*)$ are bounded as follows: 
\begin{align*}
\max_{i\neq j}|[\widetilde{\bm J}_{F_\alpha}^{(1)}(\bm\pi^*)]_{ij}| &\leq \frac{|1-\alpha|\kappa^6}{\alpha_+d_i} \leq \frac{\kappa^6}{d_i}, & \max_{i}|[\widetilde{\bm J}_{F_\alpha}^{(1)}(\bm\pi^*)]_{ii}| &\leq \frac{|1-\alpha|\kappa^6}{\alpha_+} \leq \kappa^6,\\
\max_{i\neq j}|[\widetilde{\bm J}_{F_\alpha}^{(2)}(\bm\pi^*)]_{ij}| &\leq \frac{\kappa^6}{d_i}, & \max_{i}|[\widetilde{\bm J}_{F_\alpha}^{(2)}(\bm\pi^*)]_{ii}| &\leq \kappa^6.
\end{align*}
Meanwhile, the $j$th row or column has at most $d_j$ nonzero off-diagonal entries. Consequently, 
\begin{align}
\|\widetilde{\bm J}_{F_\alpha}^{(i)}(\bm\pi^*)\|_1\leq \kappa^6\max_i\sum_{j\sim i}\frac{1}{d_j}\leq\kappa^6\vt_{\max},\quad\quad\|\widetilde{\bm J}_{F_\alpha}^{(i)}(\bm\pi^*)\|_\infty\leq \kappa^6, \label{l22}
\end{align}
where $\vt_{\max}$ is defined in \eqref{mytheta}. 
Substituting \eqref{l21} and \eqref{l22} into \eqref{l20} completes the proof.

\subsubsection{Proof of \Cref{lm:reg}}\label{lpf3}

We first use concentration of measure to narrow the configuration of comparison results, $\W=\{w_{k\ell}\}$, down to a regime in which the desired Lipschitz condition can be explicitly verified. By the Chernoff bound and a union bound, with probability at least $1-4n\exp\left\{-d_{\min}L/(12\kappa)\right\}$, for all $i\in [n]$, the total wins and losses of $i$ are concentrated around their expected values:  
\begin{equation}\label{newW1}
\begin{aligned}
&\frac{d_iL}{4\kappa} \leq \frac{1}{2}\min\left\{\sum_{\ell\sim i}\bar{w}_{i\ell}, \sum_{\ell\sim i}\bar{w}_{\ell i}\right\} \leq\min\left\{\sum_{\ell}w_{i\ell}, \sum_{\ell}w_{\ell i}\right\}, \\
&\max\left\{\sum_{\ell}w_{i\ell}, \sum_{\ell}w_{\ell i}\right\} \leq 2\max\left\{\sum_{\ell\sim i}\bar{w}_{i\ell}, \sum_{\ell\sim i}\bar{w}_{\ell i}\right\} \leq \kappa d_iL. 
\end{aligned}
\end{equation}

We now condition $\W$ on the event in \eqref{newW1}. For $\bm\pi\in\mathcal B$, we claim that the off-diagonal entries of $\bm J_{F_\alpha}^{(1)}(\bm\pi)$ and $\bm J_{F_\alpha}^{(2)}(\bm\pi)$ in the $i$th row are $(C_2/d_i)$-Lipschitz functions of $\bm\pi$ with respect to the $\ell_\infty$-norm, and all diagonal entries are $C_2$-Lipschitz functions of $\bm\pi$, where $C_2>0$ is an absolute constant depending only on $\kappa$ (i.e., $C_2$ is uniform for all entries and is independent of $n$). These Lipschitz properties carry over to those of $\bm J_{F_\alpha}(\bm\pi)$ additively. Since both $\bm J_{F_\alpha}(\bm\pi)$ and $\bm J_{F_\alpha}(\bm\pi')$ have the same nonzero entries, for $\bm\pi, \bm\pi'\in\mathcal B$, one can apply a similar estimate as in \eqref{l21}-\eqref{l22} to $\bm J_{F_\alpha}(\bm\pi)-\bm J_{F_\alpha}(\bm\pi')$ combined with the entrywise Lipschitz property to obtain that 
\begin{align*}
\|\bm J_{F_\alpha}(\bm\pi)-\bm J_{F_\alpha}(\bm\pi')\|_2\leq 2C_2\sqrt{\vt_{\max}}\|\bm\pi-\bm\pi'\|_\infty. 
\end{align*}

It remains to verify the claim. To this end, we explicitly compute the gradient of entry with respect to $\bm\pi$ and apply the mean-value theorem to control the Lipschitz constants. We consider $\bm J_{F_\alpha}^{(1)}(\bm\pi)$ first, and the proof for $\bm J_{F_\alpha}^{(2)}(\bm\pi)$ is similar. 
Recall that for $i\neq j$,  
\begin{align*}
[\bm J^{(1)}_{F_\alpha}(\bm\pi)]_{ij} = \frac{\frac{1}{\alpha_+d_iL}w_{ij} \frac{\pi_i(1 - \alpha)}{(\pi_i + \pi_j)^2}}{\frac{1}{\alpha_+d_iL}\sum_{\ell} \frac{\alpha w_{i\ell} + w_{\ell i}}{\pi_i + \pi_\ell}}\eqqcolon \frac{N_{ij}(\bm\pi)}{S^{(1)}_{i}(\bm\pi)}, 
\end{align*}
where $N_{ij}(\bm\pi)$ and $S^{(1)}_{i}(\bm\pi)$ are the numerator and denominator, respectively, i.e., $S^{(1)}_{i}(\bm\pi)$ is the same as the definition in \eqref{myS12}. Since $w_{ij}\leq L$, 
\begin{align}
|N_{ij}(\bm\pi)|
  &\leq \frac{\kappa|1-\alpha|}{\alpha_+d_i}\leq\frac{\kappa}{d_i}, \nonumber\\
|\partial_i N_{ij}(\bm\pi)|
  &= \frac{w_{ij}|1-\alpha|}{\alpha_+d_iL}
     \cdot\frac{|\pi_j-\pi_i|}{(\pi_i+\pi_j)^3}
   \leq \frac{|1-\alpha|}{\alpha_+d_i(\pi_i+\pi_j)^2}
   \leq \frac{\kappa^2}{d_i}, \label{eq:dN}\\
|\partial_j N_{ij}(\bm\pi)|
  &= \frac{w_{ij}|1-\alpha|}{\alpha_+d_iL}
     \cdot\frac{2\pi_i}{(\pi_i+\pi_j)^3}
   \leq \frac{2|1-\alpha|}{\alpha_+d_i(\pi_i+\pi_j)^2}
   \leq \frac{2\kappa^2}{d_i}, \nonumber
\end{align}
and $\partial_\ell N_{ij}(\bm\pi)=0$ for $\ell\notin\{i, j\}$. Meanwhile, $S^{(1)}_{i}(\bm\pi)$ is lower bounded by
\begin{align*}
S^{(1)}_{i}(\bm\pi)=\frac{1}{\alpha_+d_iL}\sum_{\ell} \frac{\alpha w_{i\ell} + w_{\ell i}}{\pi_i + \pi_\ell}\geq\frac{1}{\alpha_+d_iL}\max\left\{\sum_{\ell} \frac{\alpha w_{i\ell} }{\pi_i + \pi_\ell}, \sum_{\ell} \frac{w_{\ell i}}{\pi_i + \pi_\ell}\right\}\stackrel{\eqref{newW1}}{\geq}\frac{1}{16\kappa^2}. 
\end{align*}
Moreover, the absolute value of its gradient with respect to $\bm\pi$ can be computed as 
\begin{align*}
|\partial_{\ell}S^{(1)}_{i}(\bm\pi)| = 
\begin{cases} 
\frac{1}{\alpha_+d_iL} \sum_{k} \frac{\alpha w_{ik} + w_{ki}}{(\pi_i + \pi_k)^2}\leq\kappa^2 & \text{if } \ell = i \\
\frac{1}{\alpha_+d_iL} \frac{\alpha w_{i\ell} + w_{\ell i}}{(\pi_i + \pi_\ell)^2}\leq\frac{\kappa^2}{d_i} & \text{if } \ell\sim i \\
0 & \text{if } \ell\not\sim i \text{ and } \ell \neq i
\end{cases},
\end{align*}
where we used that $w_{i\ell}+\alpha w_{\ell i} \leq\alpha_+(w_{i\ell}+w_{\ell i})=\alpha_+ L$. 

By the quotient rule, for any $k\in[n]$, we have
\begin{align*}
|\partial_k[\bm J^{(1)}_{F_\alpha}(\bm\pi)]_{ij}| \leq \frac{|\partial_k N_{ij}(\bm\pi)|}{S^{(1)}_i(\bm\pi)} + \frac{|N_{ij}(\bm\pi)| |\partial_k S^{(1)}_i(\bm\pi)|}{(S^{(1)}_i(\bm\pi))^2}.
\end{align*}
We substitute the established bounds into this expression, considering three distinct cases for $k$ such that $\partial_k[\bm J^{(1)}_{F_\alpha}(\bm\pi)]_{ij}$ is not identical to zero:

\noindent\textbf{Case 1 ($k = i$):} 
\begin{align*}
|\partial_i[\bm J^{(1)}_{F_\alpha}(\bm\pi)]_{ij}| &\leq \frac{\kappa^2 d_i^{-1}}{1 / 16\kappa^2} + \frac{(\kappa d_i^{-1})(\kappa^2)}{(1 / 16\kappa^2)^2} =16\kappa^4 d_i^{-1} + 256\kappa^7 d_i^{-1} = \mathcal{O}(d_i^{-1}).
\end{align*}

\noindent\textbf{Case 2 ($k = j$):} 
\begin{align*}
|\partial_j[\bm J^{(1)}_{F_\alpha}(\bm\pi)]_{ij}| &\leq \frac{2\kappa^2 d_i^{-1}}{1 / 16\kappa^2} + \frac{(\kappa d_i^{-1})(\kappa^2 d_i^{-1})}{(1 / 16\kappa^2)^2} =32\kappa^4 d_i^{-1} + 256\kappa^7 d_i^{-2} \leq \mathcal{O}(d_i^{-1}).
\end{align*}

\noindent\textbf{Case 3 ($k \neq i, j$ and $k \sim i$):} Since $\partial_k N_{ij}(\bm\pi) = 0$, the first term vanishes, yielding
\begin{align*}
|\partial_k[\bm J^{(1)}_{F_\alpha}(\bm\pi)]_{ij}| &\leq 0 + \frac{(\kappa d_i^{-1})(\kappa^2 d_i^{-1})}{(1 / 16\kappa^2)^2} = 256\kappa^7 d_i^{-2} = \mathcal{O}(d_i^{-2}). 
\end{align*}
The implicit constants in $\mathcal O(\cdot)$ depend only on $\kappa$. For any $\bm\pi, \bm\pi'\in\mathcal B$, let $\bar{\bm\pi}^{(\ell)} = (\pi'_1, \ldots, \pi'_\ell, \pi_{\ell+1}, \ldots, \pi_n)^\top$, i.e., $\bar{\bm\pi}^{(0)} = \bm\pi$ and $\bar{\bm\pi}^{(n)} = \bm\pi'$. Applying the triangle inequality and the mean-value theorem, 
\begin{align*}
|[\bm J^{(1)}_{F_\alpha}(\bm\pi)]_{ij}-[\bm J^{(1)}_{F_\alpha}(\bm\pi')]_{ij}|&\leq\sum_{\ell=1}^n|[\bm J^{(1)}_{F_\alpha}(\bar{\bm\pi}^{(\ell)})]_{ij}-[\bm J^{(1)}_{F_\alpha}(\bar{\bm\pi}^{(\ell-1)})]_{ij}|\\
& =  \sum_{\ell=1}^n|\partial_\ell[\bm J^{(1)}_{F_\alpha}(\bm z^{(\ell)})]_{ij}||\pi_\ell-\pi_\ell'|\\
& \lesssim \left(\frac{1}{d_i} + \frac{1}{d_i} + \frac{d_i-1}{d_i^2}\right)\|\bm\pi-\bm\pi'\|_\infty\\
& \lesssim \frac{1}{d_i}\|\bm\pi-\bm\pi'\|_\infty, 
\end{align*}
where $\bm z^{(\ell)}$ lies in the segment between $\bar{\bm\pi}^{(\ell-1)}$ and $\bar{\bm\pi}^{(\ell)}$, and the second inequality follows from the derivative estimates above. This proves the claim for the nonzero off-diagonals of $\bm J_{F_\alpha}^{(1)}(\bm\pi)$. 

For the diagonal term, it is a sum of $d_i$ terms similar to the off-diagonal terms except that the $\pi_i$ in the numerator is replaced by $\pi_\ell$ for $\ell\sim i$. Consequently, the same computation can be applied to show that each of the summands is $(C_2/d_i)$-Lipschitz, which carries over to the whole sum additively and results in a Lipschitz constant $C_2$. This establishes the claim for $\bm J^{(1)}_{F_\alpha}(\bm\pi)$. The proof for $\bm J^{(2)}_{F_\alpha}(\bm\pi)$ is similar and thus omitted.

\subsubsection{Proof of \Cref{lm:2est}}\label{lpf4}
Recall the uniform consistency of the MLE in \Cref{thm:asymp2}: If $np_{\min}/(\log n)^3\to\infty$, then for all sufficiently large $n$, with probability at least $1-n^{-2}$, the MLE $\pih$ exists and satisfies
\begin{align}
\|\pih - \bm\pi^*\|_\infty\lesssim\sqrt{\frac{(\log n)^3}{np_{\min}L}},  \label{ucmle}
\end{align}
where the implicit constants depend on $\kappa$.  

On the other hand, a standard concentration argument (over the randomness in the SBM) shows that, with probability at least $1-n^{-2}$, $n(p+q)/4\leq d_i\leq n(p+q)$ for all $i\in [n]$, i.e., $d_{\min}\geq np_{\min}/2$. This implies that $\vt_i\leq 20$. 

Taking $\e = \sqrt{\max\{40C_1, 12\kappa\}(\log n)^3/(np_{\min}L)}$, the events in \eqref{ucmle}, \Cref{lm:trig}, and \Cref{lm:reg}, hold simultaneously with probability at least 
\begin{align*}
&1-n^{-2} - 6n\exp\left\{-d_{\min}L\e^2/(C_1 \vt_{\max})\right\} - 4n\exp\left\{-d_{\min}L/(12\kappa)\right\}\\
\geq&\ 1-n^{-2} - 10n\exp\left\{-(\log n)^3\right\}\geq 1-2n^{-2}
\end{align*}
for all sufficiently large $n$. On the intersection of these events, applying \Cref{lm:trig} and \Cref{lm:reg}, for sufficiently large $n$, $\pih\in\mathcal B$, and consequently, 
\begin{align*}
\|\bm J_{F_\alpha}(\pih)-\JJ_{F_\alpha}(\bm\pi^*)\|_2 &\leq \|\bm J_{F_\alpha}(\pih)-\bm J_{F_\alpha}(\bm\pi^*)\|_2 + \|\bm J_{F_\alpha}(\bm\pi^*)-\JJ_{F_\alpha}(\bm\pi^*)\|_2\lesssim\sqrt{\frac{(\log n)^3}{np_{\min}L}}, 
\end{align*}
where the implicit constant depends on $\kappa$. 

\subsubsection{Proof of \Cref{lm:bfl}}\label{lpf5}
For the condition number estimates, note that for all $i\in [n]$, we can establish strict bounds on the expected sums $\eta_i(\bm\pi^*)$. For the upper bound, we use the fact that $\alpha \pi_i^* + \pi_\ell^* \leq \alpha_+(\pi_i^* + \pi_\ell^*)$, which yields:
\begin{align*}
\frac{\min_j\pi_j^*}{4\max_j(\pi_j^*)^2} \leq \frac{(\alpha+1)\min_j\pi_j^*}{4\alpha_+\max_j(\pi_j^*)^2} \leq \eta_i(\bm\pi^*) = \frac{1}{\alpha_+d_i}\sum_{\ell \sim i} \frac{\alpha \pi_i^* + \pi_\ell^*}{(\pi_i^* + \pi_\ell^*)^2} \leq \frac{1}{\alpha_+d_i}\sum_{\ell \sim i} \frac{\alpha_+}{\pi_i^* + \pi_\ell^*} \leq \frac{1}{2\min_j\pi_j^*}.
\end{align*}
Meanwhile, recall the degree concentration as used in the proof of \Cref{lm:2est} (which holds under the high-probability event considered in \Cref{lm:2est}): for all $i\in [n]$,
\begin{align*}
\frac{n(p+q)}{4}\leq d_i\leq n(p+q).
\end{align*}
Combining the above estimates, the condition number of $\bar{\bm D}_\alpha(\bm\pi^*)$ satisfies
\begin{align*}
\text{cond}(\bar{\bm D}_\alpha(\bm\pi^*))^2 &= \frac{\max_i[\pi_i^*/(d_iL\eta_i(\bm\pi^*))]}{\min_i[\pi_i^*/(d_iL\eta_i(\bm\pi^*))]} \leq \frac{\max_i\pi_i^*}{\min_i\pi_i^*}\cdot\frac{\max_i d_i}{\min_i d_i}\cdot\frac{\max_i\eta_i(\bm\pi^*)}{\min_i\eta_i(\bm\pi^*)}\\
&\leq \kappa \cdot 4 \cdot \frac{1/(2\min_j\pi_j^*)}{1\min_j\pi_j^*/(4\max_j(\pi_j^*)^2)} = 4\kappa \cdot 2\left(\frac{\max_j\pi_j^*}{\min_j\pi_j^*}\right)^2 \leq 8\kappa^3,
\end{align*}
where we used $\max_i\pi_i^*/\min_i\pi_i^*\leq \kappa$ and $\max_i d_i/\min_i d_i\leq 4$. 
Taking square roots yields $\text{cond}(\bar{\bm D}_\alpha(\bm\pi^*))\leq \sqrt{8}\kappa^{3/2}$. 
Once the second statement is established, a similar bound for $\text{cond}(\bm D_\alpha(\pih))$ can be established by slightly enlarging the constant for all sufficiently large $n$.  

To prove the second statement, taking the ratio between $[\bar{\bm D}_\alpha(\bm\pi^*)]_{ii}$ and $[\bm D_\alpha(\pih)]_{ii}$, 
\begin{align}
\frac{[\bar{\bm D}_\alpha(\bm\pi^*)]_{ii}}{[\bm D_\alpha(\pih)]_{ii}} = \sqrt{\frac{\pi^*_i S^{(1)}_i(\pih)}{\widehat{\pi}_i \eta_i(\bm\pi^*)}} = \sqrt{\frac{\pi^*_i}{\widehat{\pi}_i}}\cdot \sqrt{\frac{S^{(1)}_i(\pih)}{S^{(1)}_i(\bm\pi^*)}}\cdot\sqrt{\frac{S^{(1)}_i(\bm\pi^*)}{\eta_i(\bm\pi^*)}}. \label{letusbdd}
\end{align}
Since the assumptions here are the same as \Cref{lm:2est}, one can choose $c$ in \Cref{lm:2est} sufficiently large so that the upper bound in \eqref{ucmle} is less than $1/(2\kappa)\leq\min_j\pi_j^*/2$. Since $\min_j\pi_j^*\geq 1/\kappa$, $\widehat{\pi}_i\geq 1/(2\kappa)$. For the first term in \eqref{letusbdd},  
\begin{align*}
1- \mathcal O(\|\pih-\bm\pi^*\|_\infty)\leq\sqrt{\frac{\pi^*_i}{\widehat{\pi}_i}}\leq 1+\mathcal O(\|\pih-\bm\pi^*\|_\infty),  
\end{align*}
where the implicit constant depends on $\kappa$. 
For the second term, a similar perturbation argument shows that 
\begin{align*}
1-\mathcal O(\|\pih-\bm\pi^*\|_\infty)\leq\min_{\ell\sim i}\frac{\pi_i^*+\pi_\ell^*}{\widehat{\pi}_i+\widehat{\pi}_\ell}\leq\sqrt{\frac{S^{(1)}_i(\pih)}{S^{(1)}_i(\bm\pi^*)}}\leq\max_{\ell\sim i}\frac{\pi_i^*+\pi_\ell^*}{\widehat{\pi}_i+\widehat{\pi}_\ell}\leq 1+\mathcal O(\|\pih-\bm\pi^*\|_\infty). 
\end{align*}
The last term is within order $\mathcal O(\e)$ at $1$ by the uniform consistency of the MLE (which holds under the event in \Cref{lm:2est}), with $\e=\mathcal O(\sqrt{(\log n)^3/(np_{\min}L)})=o(1)$. Combining these estimates and using \eqref{ucmle} yields  
\begin{align*}
\left|\frac{[\bar{\bm D}_\alpha(\bm\pi^*)]_{ii}}{[\bm D_\alpha(\pih)]_{ii}}-1\right|\lesssim\sqrt{\frac{(\log n)^3}{np_{\min}L}}. 
\end{align*}
Since this holds simultaneously for all $i$ and the bound is independent of $i$, taking the maximum over $i$ finishes the proof of the second statement.

\subsubsection{Proof of \Cref{lm:lowerbdd}}\label{lpf6}
Recall the $\bar{\bm J}^{\sym}_{F_\alpha}(\bm\pi^*)$ used in the proof of \Cref{lm:bfl}: $\bar{\bm J}^{\sym}_{F_\alpha}(\bm\pi^*)\coloneqq\bar{\bm D}_\alpha(\bm\pi^*)^{-1}\JJ_{F_\alpha}(\bm\pi^*)\bar{\bm D}_\alpha(\bm\pi^*)$. Since $\bar{\bm D}_\alpha(\bm\pi^*)$ is diagonal, the lower triangular part of $\bar{\bm J}^{\sym}_{F_\alpha}(\bm\pi^*)$, denoted by $\bm J^{\sym}_{F_\alpha, \l}(\bm\pi^*)$, satisfies $\bm J^{\sym}_{F_\alpha, \l}(\bm\pi^*)\coloneqq\bar{\bm D}_\alpha(\bm\pi^*)^{-1}\JJ_{F_\alpha, \l}(\bm\pi^*)\bar{\bm D}_\alpha(\bm\pi^*)$. Rearranging terms gives
\begin{equation}\label{heqisq} 
\begin{aligned}
\|(\bm I - \JJ_{F_\alpha, \l}(\bm\pi^*))^{-1}\|_2 &= \|\bar{\bm D}_\alpha(\bm\pi^*)(\bm I-\bm J^{\sym}_{F_\alpha, \l}(\bm\pi^*))^{-1}\bar{\bm D}_\alpha(\bm\pi^*)^{-1}\|_2\\
&\leq \text{cond}(\bar{\bm D}_\alpha(\bm\pi^*))\|(\bm I-\bm J^{\sym}_{F_\alpha, \l}(\bm\pi^*))^{-1}\|_2\\
&\leq 3\kappa^{\frac{3}{2}}\|(\bm I-\bm J^{\sym}_{F_\alpha, \l}(\bm\pi^*))^{-1}\|_2,
\end{aligned} 
\end{equation}
where the last step follows from the same degree concentration result in Section~\ref{lpf4} and holds with probability at least $1-n^{-2}$. For every $\bm x\in\R^n$, 
\begin{align*}
\left|\<\bm x, (\bm I-\bm J^{\sym}_{F_\alpha, \l}(\bm\pi^*))\bm x\>\right|&\geq \left|\Re\left(\<\bm x, (\bm I-\bm J^{\sym}_{F_\alpha, \l}(\bm\pi^*))\bm x\>\right)\right|\\
&= \left|\<\bm x, \left(\bm I-\frac{1}{2}(\bar{\bm J}^{\sym}_{F_\alpha}(\bm\pi^*) - \bar{\bm J}^{\sym}_{F_\alpha, \d}(\bm\pi^*))\right)\bm x\>\right|\geq\frac{1}{2}\|\bm x\|_2^2, 
\end{align*} 
where $\bar{\bm J}^{\sym}_{F_\alpha, \d}(\bm\pi^*)$ is the diagonal part of $\bar{\bm J}^{\sym}_{F_\alpha}(\bm\pi^*)$. For the last step, note that $\bar{\bm J}^{\sym}_{F_\alpha}(\bm\pi^*)$ is symmetric, with eigenvalues identical to those of $\JJ_{F_\alpha}(\bm\pi^*)$ and therefore lying in $[-1, 1]$. Since $\bar{\bm J}^{\sym}_{F_\alpha, \d}(\bm\pi^*)$ is positive-semidefinite, $\left(\bm I-\frac{1}{2}(\bar{\bm J}^{\sym}_{F_\alpha}(\bm\pi^*) - \bar{\bm J}^{\sym}_{F_\alpha, \d}(\bm\pi^*))\right)$ is positive-definite with all eigenvalues bounded below by $1/2$. Consequently, $\|(\bm I-\bm J^{\sym}_{F_\alpha, \l})(\bm\pi^*)^{-1}\|_2\leq 2$. Substituting this into \eqref{heqisq} yields the desired result.

\section{Violation of BT assumptions under cyclic outcomes}\label{app:bto1}

This section provides a rigorous justification of the asymptotic violation of the BT model assumptions for the purely cyclic comparison outcomes considered in \Cref{sec:num1} using a large-deviation argument. The empirical distribution associated with $\W$ in the cyclic Markov chain is $\mu\sim\prod_{i}\mathrm{Bin}(m_{i, i+1}, 1)$, where $m_{n ,n+1} = m_{n, 1}$. For $\bm\pi\in\R_+^n$, its induced distribution on $\W$ is $\nu(\bm\pi) = \prod_{i}\mathrm{Bin}(m_{i, i+1}, \pi_{i}/(\pi_i+\pi_{i+1}))$, where $\pi_{n+1} = \pi_1$. Assuming $L\leq m_{i, i+1}\leq c L$, the average Kullback--Leibler divergence from $\nu(\bm\pi)$ to $\mu$ is lower bounded by some universal constant:  
\begin{align*}
\frac{1}{\sum_{i}m_{i, i+1}}\KL(\mu\,\|\,\nu(\bm\pi)) &= \frac{1}{\sum_{i}m_{i, i+1}}\sum_{i}m_{i, i+1}\log\left(\frac{\pi_i + \pi_{i+1}}{\pi_i}\right)\\
&\geq \frac{L}{\sum_{i}m_{i, i+1}}\sum_{i}\log\left(\frac{\pi_i + \pi_{i+1}}{\pi_i}\right)\\
&\geq \frac{L}{\sum_{i}m_{i, i+1}}\sum_{i}\log\left(2\sqrt{\frac{\pi_{i+1}}{\pi_i}}\right)\\
&= \frac{Ln}{\sum_{i}m_{i, i+1}}\log 2 \geq\frac{\log 2}{c}>0.  
\end{align*}
This suggests that the probability of observing these particular cyclic outcomes is exponentially small uniformly for all valid parameters in a BT model. Consequently, approximating $\rhosync$ by $\brhosync$ (with $\bm\pi^*$ by $\pih$ when computing $\brhosync$) becomes invalid.

\printbibliography

@article{10.1111/rssa.12124,
	author = {Varin, Cristiano and Cattelan, Manuela and Firth, David},
	title = {Statistical Modelling of Citation Exchange Between Statistics Journals},
	journal={J. R. Stat. Soc. Ser. A},
	volume = {179},
	number = {1},
	pages = {1-63},
	year = {2015},
	month = {11},
	issn = {0964-1998}
}

@article{abbe2018community,
  title={Community detection and stochastic block models: recent developments},
  author={Abbe, Emmanuel},
	journal = {J. Mach. Learn. Res.},
    volume={18},
  number={177},
  pages={1--86},
  year={2018}
}

@inproceedings{agarwal2018accelerated,
  title={Accelerated spectral ranking},
  author={Agarwal, Arpit and Patil, Prathamesh and Agarwal, Shivani},
  booktitle={ICML},
  pages={70--79},
  year={2018},
  organization={PMLR}
}

@article{baker2014dynamic,
  title={A dynamic paired comparisons model: Who is the greatest tennis player?},
  author={Baker, Rose D and McHale, Ian G},
	journal={Eur. J. Oper. Res.},
  volume={236},
  number={2},
  pages={677--684},
  year={2014},
  publisher={Elsevier}
}

@book{berman1994nonnegative,
  title={Nonnegative matrices in the mathematical sciences},
  author={Berman, Abraham and Plemmons, Robert J},
  year={1994},
  publisher={SIAM}
}

@book{billingsley2013convergence,
  title={Convergence of probability measures},
  author={Billingsley, Patrick},
  year={2013},
  publisher={John Wiley \& Sons}
}

@article{bozoki2016application,
	title={An application of incomplete pairwise comparison matrices for ranking top tennis players},
	author={Boz{\'o}ki, S{\'a}ndor and Csat{\'o}, L{\'a}szl{\'o} and Temesi, J{\'o}zsef},
	journal={Eur. J. Oper. Res.},
	volume={248},
	number={1},
	pages={211--218},
	year={2016},
	publisher={Elsevier}
}

@article{MR0070925,
  title={Rank analysis of incomplete block designs: I. the method of paired comparisons},
  author={Bradley, Ralph Allan and Terry, Milton E},
  journal={Biometrika},
  volume={39},
  number={3/4},
  pages={324--345},
  year={1952},
  publisher={JSTOR}
}

@article{chen2019spectral,
  title={Spectral method and regularized MLE are both optimal for top-K ranking},
  author={Chen, Yuxin and Fan, Jianqing and Ma, Cong and Wang, Kaizheng},
  journal={Ann. Statist.},
  volume={47},
  number={4},
  pages={2204},
  year={2019}
}

@article{chen2022optimal,
  title={Optimal full ranking from pairwise comparisons},
  author={Chen, Pinhan and Gao, Chao and Zhang, Anderson Y},
  	journal = {Ann. Statist.},
  volume={50},
  number={3},
  pages={1775--1805},
  year={2022},
  publisher={Institute of Mathematical Statistics}
}

@article{chen2022partial,
	title={Partial recovery for top-k ranking: optimality of MLE and suboptimality of the spectral method},
	author={Chen, Pinhan and Gao, Chao and Zhang, Anderson Y},
	journal={Ann. Statist.},
	volume={50},
	number={3},
	pages={1618--1652},
	year={2022},
	publisher={Institute of Mathematical Statistics}
}

@article{chen2025item,
    AUTHOR = {Chen, Yunxiao and Li, Xiaoou and Liu, Jingchen and Ying,
              Zhiliang},
  title={Item response theory—A statistical framework for educational and psychological measurement},
   journal = {Statist. Sci.},
    VOLUME = {40},
      YEAR = {2025},
    NUMBER = {2},
     PAGES = {167--194}
}

@inproceedings{Christiano2017,
  title={Deep reinforcement learning from human preferences},
  author={Christiano, Paul F. and Leike, Jan and Brown, Tom and Martic, Miljan and Legg, Shane and Amodei, Dario},
  booktitle={NeurIPS},
  pages={4299--4307},
  year={2017}
}

@article{dong2025statistical,
    AUTHOR = {Dong, Pinjun and Han, Ruijian and Jiang, Binyan and Xu,
              Yiming},
     TITLE = {Statistical ranking with dynamic covariates},
   journal = {J. R. Stat. Soc. Ser. B Stat. Methodol.},
    VOLUME = {88},
      YEAR = {2026},
    NUMBER = {1},
     PAGES = {221--238}
}

@article{Dykstra:1956,
  title = {A Note on the Rank Analysis of Incomplete Block Designs -- Applications beyond the Scope of Existing Tables},
  author = {Dykstra, O.},
  year = {1956},
  journal = {Biometrics},
  volume = {12},
  number = {3},
  pages = {301}
}

@article{fahrmeir1985consistency,
  title={Consistency and asymptotic normality of the maximum likelihood estimator in generalized linear models},
  author={Fahrmeir, Ludwig and Kaufmann, Heinz},
  	journal = {Ann. Statist.},
  volume={13},
  number={1},
  pages={342--368},
  year={1985},
  publisher={Institute of Mathematical Statistics}
}

@article{fang2026recent,
  title={Recent advances in the Bradley--Terry Model: theory, algorithms, and applications},
  author={Fang, Shuxing and Han, Ruijian and Luo, Yuanhang and Xu, Yiming},
  journal={arXiv preprint arXiv:2601.14727},
  year={2026}
}

@article{feng2009addressing,
  title={Addressing the assessment challenge with an online system that tutors as it assesses},
  author={Feng, Mingyu and Heffernan, Neil and Koedinger, Kenneth},
  journal={User Model. User-Adapt. Interact.},
  volume={19},
  number={3},
  pages={243--266},
  year={2009},
  publisher={Springer}
}

@article{gao2023uncertainty,
  title={Uncertainty quantification in the Bradley--Terry--Luce model},
  author={Gao, Chao and Shen, Yandi and Zhang, Anderson Y},
  journal={Inf. Inference},
  volume={12},
  number={2},
  pages={1073--1140},
  year={2023},
  publisher={Oxford University Press}
}

@article{goldenberg2010survey,
  title={A survey of statistical network models},
  author={Goldenberg, Anna and Zheng, Alice X and Fienberg, Stephen E and Airoldi, Edoardo M},
  journal={Found. Trends Mach. Learn.},
  volume={2},
  number={2},
  pages={129--233},
  year={2010},
  publisher={Emerald Publishing Limited}
}

@article{hamilton2023many,
  title={The many routes to the ubiquitous Bradley-Terry model},
  author={Hamilton, Ian and Tawn, Nick and Firth, David},
  journal={arXiv preprint arXiv:2312.13619},
  year={2023}
}

@article{han2020asymptotic,
  title={Asymptotic theory of sparse Bradley--Terry model},
  author={Han, Ruijian and Ye, Rougang and Tan, Chunxi and Chen, Kani},
  journal={Ann. Appl. Probab.},
  volume={30},
  number={5},
  pages={2491--2515},
  year={2020},
  publisher={JSTOR}
}

@article{han2023general,
  title={A general pairwise comparison model for extremely sparse networks},
  author={Han, Ruijian and Xu, Yiming and Chen, Kani},
  journal={J. Amer. Statist. Assoc.},
  volume={118},
  number={544},
  pages={2422--2432},
  year={2023},
  publisher={Taylor \& Francis}
}

@article{han2025unified,
  title={A unified analysis of likelihood-based estimators in the Plackett--Luce model},
  author={Han, Ruijian and Xu, Yiming},
  journal={Ann. Statist.},
  volume={53},
  number={5},
  pages={2077--2102},
  year={2025},
  publisher={Institute of Mathematical Statistics}
}

@article{Higham:2002a,
  title = {Detecting a Definite Hermitian Pair and a Hyperbolic or Elliptic Quadratic Eigenvalue Problem, and Associated Nearness Problems},
  author = {Higham, Nicholas J. and Tisseur, Fran{\c c}oise and Van Dooren, Paul M.},
  year = {2002},
   journal = {Linear Algebra Appl.},
  volume = {351--352},
  pages = {455--474}
}

@article{holland1983stochastic,
  title={Stochastic blockmodels: First steps},
  author={Holland, Paul W and Laskey, Kathryn Blackmond and Leinhardt, Samuel},
  journal={Soc. Networks},
  volume={5},
  number={2},
  pages={109--137},
  year={1983},
  publisher={Elsevier}
}

@article{hunter2004mm,
  title={MM algorithms for generalized Bradley-Terry models},
  author={Hunter, David R},
  journal={Ann. Statist.},
  volume={32},
  number={1},
  pages={384--406},
  year={2004},
  publisher={Institute of Mathematical Statistics}
}

@book{Kato:2013,
  title={Perturbation theory for linear operators},
  author={Kato, Tosio},
  volume={132},
  year={2013},
  publisher={Springer Science \& Business Media}
}

@article{loewen2012testing,
  title={Testing the power of arguments in referendums: A {B}radley--{T}erry approach},
  author={Loewen, Peter John and Rubenson, Daniel and Spirling, Arthur},
  journal={Elect. Stud.},
  volume={31},
  number={1},
  pages={212--221},
  year={2012},
  publisher={Elsevier}
}

@article{mcfadden1973conditional,
  title={Conditional logit analysis of qualitative choice behavior},
  author={McFadden, D},
  journal={Frontiers in Economics},
  pages={105--142},
  year={1973},
  publisher={Academic Press}
}

@article {MR0097876,
	AUTHOR = {Ford, Jr., L. R.},
	TITLE = {Solution of a ranking problem from binary comparisons},
	journal = {Amer. Math. Monthly},
	VOLUME = {64},
	YEAR = {1957},
	NUMBER = {8, part II},
	PAGES = {28--33},
	MRCLASS = {62.00},
	MRNUMBER = {0097876},
	MRREVIEWER = {I. R. Savage},
}

@article{Negahban2012RankCR,
  title={Rank centrality: Ranking from pairwise comparisons},
  author={Negahban, Sahand and Oh, Sewoong and Shah, Devavrat},
  journal={Oper. Res.},
  volume={65},
  number={1},
  pages={266--287},
  year={2017},
  publisher={INFORMS Inst. for Operations Res. and the Management Sciences}
}

@article{newman2023efficient,
  title={Efficient computation of rankings from pairwise comparisons},
  author={Newman, MEJ},
  journal={J. Mach. Learn. Res.},
  volume={24},
  number={238},
  pages={1--25},
  year={2023}
}

@book{ortega2000iterative,
  title={Iterative solution of nonlinear equations in several variables},
  author={Ortega, James M and Rheinboldt, Werner C},
  year={2000},
  publisher={SIAM}
}

@article{qu2025sinkhorn,
  title={On sinkhorn’s algorithm and choice modeling},
  author={Qu, Zhaonan and Galichon, Alfred and Gao, Wenzhi and Ugander, Johan},
journal = {Oper. Res.},
year={2025},
  publisher={INFORMS}
}

@inproceedings{Rafailov2023,
  title={Direct preference optimization: Your language model is secretly a reward model},
  author={Rafailov, Rafael and Sharma, Archit and Mitchell, Eric and Manning, Christopher D and Ermon, Stefano and Finn, Chelsea},
  booktitle={NeurIPS},
  volume={36},
  pages={53728--53741},
  year={2023}
}

@article{rasch1960studies,
  title={Studies in mathematical psychology: I. Probabilistic models for some intelligence and attainment tests.},
  author={Rasch, Georg},
  year={1960},
  publisher={Nielsen \& Lydiche}
}

@book{Rellich:1969,
  title={Perturbation theory of eigenvalue problems},
  author={Rellich, Franz},
  year={1969},
  publisher={CRC Press}
}

@book{saad2003iterative,
  title={Iterative methods for sparse linear systems},
  author={Saad, Yousef},
  year={2003},
  publisher={SIAM}
}

@book{Stewart:1990,
  title     = {Matrix Perturbation Theory},
  author    = {Stewart, G. W. and Sun, J.-G.},
  year      = {1990},
  publisher = {Academic Press},
  address   = {Boston, MA}
}

@inproceedings{Sun2025,
  title={Rethinking Reward Modeling in Preference-based Large Language Model Alignment},
  author={Sun, Hao and Shen, Yunyi and Ton, Jean-Francois},
  booktitle={ICLR},
  year={2025},
  url={https://openreview.net/forum?id=rfdblE10qm}
}

@article{thurstone1927method,
  title={The method of paired comparisons for social values.},
  author={Thurstone, Louis L},
  journal = {J. Abnorm. Psychol.},
  volume={21},
  number={4},
  pages={384},
  year={1927},
  publisher={American Psychological Association}
}

@article{Tisseur:2001,
  title = {The {{Quadratic Eigenvalue Problem}}},
  author = {Tisseur, Fran{\c c}oise and Meerbergen, Karl},
  year = 2001,
  journal = {SIAM Rev.},
  volume = {43},
  number = {2},
  pages = {235--286}
}

@article{tropp2012user,
	    AUTHOR = {Tropp, Joel A.},
  title={User-friendly tail bounds for sums of random matrices},
  journal = {Found. Comput. Math.},
VOLUME = {12},
YEAR = {2012},
NUMBER = {4},
PAGES = {389--434},
MRCLASS = {60B20 (60F10 60G42 60G50)},
MRNUMBER = {2946459}
}

@article{vilette2020comparing,
  title={Comparing dominance hierarchy methods using a data-splitting approach with real-world data},
  author={Vilette, Chlo{\'e} and Bonnell, Tyler and Henzi, Peter and Barrett, Louise},
  journal={Behav. Ecol.},
  volume={31},
  number={6},
  pages={1379--1390},
  year={2020},
  publisher={Oxford University Press UK}
}

@article{vojnovic2023accelerated,
  title={Accelerated mm algorithms for inference of ranking scores from comparison data},
  author={Vojnovi{\'c}, Milan and Yun, Se-Young and Zhou, Kaifang},
journal = {Oper. Res.},
volume={71},
  number={4},
  pages={1318--1342},
  year={2023},
  publisher={INFORMS}
}

@article{wapman2022quantifying,
	title={Quantifying hierarchy and dynamics in {US} faculty hiring and retention},
	author={Wapman, K Hunter and Zhang, Sam and Clauset, Aaron and Larremore, Daniel B},
	journal={Nature},
	volume={610},
	number={7930},
	pages={120--127},
	year={2022},
	publisher={Nature Publishing Group UK London}
}

@article{yeung2025efficient,
  title={Efficient inference of rankings from multibody comparisons},
  author={Yeung, Jack and Kaiser, Daniel and Radicchi, Filippo},
   journal = {Phys. Rev. E},
  volume={112},
  number={1},
  pages={014305},
  year={2025},
  publisher={APS}
}

@book{young2014iterative,
  title={Iterative solution of large linear systems},
  author={Young, David M},
  year={2014},
  publisher={Elsevier}
}

@article{zermelo1929berechnung,
  title={Die berechnung der turnier-ergebnisse als ein maximumproblem der wahrscheinlichkeitsrechnung},
  author={Zermelo, Ernst},
  journal={Math. Z.},
  volume={29},
  number={1},
  pages={436--460},
  year={1929},
  publisher={Springer}
}

\end{document}